\pdfoutput=1 
\documentclass[10pt, reqno]{amsart}

\usepackage[T1]{fontenc}
\usepackage[utf8]{inputenc}
\usepackage{lmodern}
\usepackage[USenglish]{babel}
\usepackage{amsmath,amssymb,amsthm,upgreek}
\usepackage{mathrsfs}
\usepackage{aligned-overset}
\usepackage{mathtools}
\usepackage{colortbl,color} %
\usepackage[dvipsnames]{xcolor}
\usepackage{graphicx,tikz,pgfplots,pgfplotstable} %
\usepackage{enumitem}
\pgfplotsset{compat=1.6}
\usepackage{subcaption,float,longtable}
\usepackage[bookmarks=true,
            bookmarksnumbered,
            plainpages=false,
            linktocpage,
            colorlinks=true,
            citecolor=green!80!black,
            linkcolor=red!70!black,
            filecolor=magenta,
            urlcolor=magenta,
            hidelinks,
            breaklinks,
            pdfauthor={Paul Geuchen, Thomas Telaar, Dominik Stöger, Felix Voigtlaender},
            pdftitle={Upper and lower bounds for the Lipschitz constant of random neural networks},
            unicode=true]{hyperref}
\usepackage[nameinlink,capitalize,noabbrev]{cleveref}

\usepackage{adforn} %
\usepackage{bbm}
\usepackage{csquotes}
\MakeOuterQuote{"}
\usepackage{comment}
\usepackage[margin=3cm]{geometry}
\usepackage[style = numeric, url = true, giveninits=true, eprint = false, sortcites=true,backend = biber, maxbibnames=99]{biblatex}
\AtBeginBibliography{\small}
\AtEveryBibitem{\clearlist{language}}
\AtEveryBibitem{\clearfield{issn}}
\AtEveryCitekey{\clearfield{issn}}

\DeclareBibliographyCategory{needsurl}
\renewbibmacro*{url+urldate}{%
  \ifcategory{needsurl}
    {\printfield{url}%
     \iffieldundef{urlyear}
       {}
       {\setunit*{\addspace}%
        \printurldate}}
    {}}
\newcommand{\entryneedsurl}[1]{\addtocategory{needsurl}{#1}}
\entryneedsurl{van2014probability}

\usepackage{ifthen}

\usetikzlibrary{decorations.fractals,
                calc,
                intersections,
                through,
                backgrounds,
                arrows,
                patterns,
                shapes.geometric,
                shapes.misc,
                spy,
                fadings,
                decorations.pathreplacing
               }

\let\emptyset\varnothing

\newcommand{\ee}{\mathrm{e}}
\newcommand{\dd}{\mathrm{d}}
\newcommand{\NN}{\mathbb{N}}                                     %
\newcommand{\RR}{\mathbb{R}}                                     %

\newcommand{\PP}{\mathbb{P}}

\newcommand{\abs}[1]{\left\lvert#1\right\rvert}                  %
\newcommand{\defeq}{\mathrel{\mathop:}=}                               %

\newcommand{\norel}{\mathrel{\phantom{=}}}                                           %

\newcommand{\iid}{\overset{\mathrm{i.i.d.}}{\sim}}

\renewcommand{\tilde}{\widetilde}
\newcommand{\B}{\overline{B}}
\newcommand{\paul}{}

\DeclareMathOperator{\id}{id}
\DeclareMathOperator{\diam}{diam}

\DeclareMathOperator{\IM}{Im}

\DeclareMathOperator{\lip}{Lip}
\DeclareMathOperator{\vc}{VC}
\DeclareMathOperator{\diag}{\Delta}

\DeclareMathOperator{\relu}{ReLU}
\DeclareMathOperator{\rang}{rank}

\let \eps \varepsilon
\let \epsilon \varepsilon

\theoremstyle{plain} 
\newtheorem{theorem}{Theorem}[section]

\newtheorem{assumption}[theorem]{Assumption}
\newtheorem{lemma}[theorem]{Lemma}
\newtheorem{proposition}[theorem]{Proposition}

\theoremstyle{definition} %
\newtheorem{definition}[theorem]{Definition}

\theoremstyle{remark} %
\newtheorem{remark}[theorem]{Remark}

\crefname{theorem}{theorem}{theorems}
\crefname{Prop}{Proposition}{Propositions}
\crefname{Lem}{Lemma}{Lemmas}
\crefname{Kor}{Corollary}{Corollaries}
\crefname{Bem}{Remark}{Remarks}
\crefname{Bsp}{Example}{Examples}
\crefname{Def}{Definition}{Definitions}
\crefname{Alg}{Algorithm}{Algorithms}

\numberwithin{equation}{section}

\makeatletter
\renewcommand*{\eqref}[1]{%
  \hyperref[{#1}]{\textup{\tagform@{\ref*{#1}}}}%
}
\makeatother

\makeatletter
\@ifundefined{textcommabelow}{%
  \DeclareTextCommandDefault\textcommabelow[1]
    {\hmode@bgroup\ooalign{\null#1\crcr\hidewidth\raise-.31ex
     \hbox{\check@mathfonts\fontsize\ssf@size\z@
     \math@fontsfalse\selectfont,}\hidewidth}\egroup}%
}{}
\makeatother

\definecolor{pgcol}{rgb}{0,0,1.} %
\definecolor{tjcol}{cmyk}{.74, 0, 1, .41} %
\definecolor{hmcol}{rgb}{0.9,.5,0} %

\definecolor{islamicgreen}{rgb}{0.0, 0.56, 0.0}
\definecolor{darkpastelgreen}{rgb}{0.01, 0.75, 0.24}
 \newcommand*{\fres}[2]{ {\left.\kern-\nulldelimiterspace #1 \vphantom{\big|} \right|_{\kern-1pt #2} }}

\newcommand{\EE}{\mathbb{E}}
\newcommand{\arrow}[1]{\overset{\rightarrow}{#1}}

\usepackage{url}

\begin{document}
\allowdisplaybreaks

\hyphenation{Lip-schitz}

\title[Upper and lower bounds for the Lipschitz constant of random neural networks]{Upper and lower bounds for the Lipschitz constant of random neural networks}

\author{Paul Geuchen}
\address[P. Geuchen]{Mathematical Institute for Machine Learning and Data Science (MIDS),
Catholic University of Eichstätt--Ingolstadt (KU),
Auf der Schanz 49, 85049 Ingolstadt, Germany}
\email{paul.geuchen@ku.de}
\thanks{}

\author{Dominik Stöger}
\address[D. Stöger]{Mathematical Institute for Machine Learning and Data Science (MIDS),
	Catholic University of Eichstätt--Ingolstadt (KU), Auf der Schanz 49, 85049 Ingolstadt, Germany}
\email{dominik.stoeger@ku.de}
\thanks{}

\author{Thomas Telaar}
\address[T. Telaar]{Chair of Mathematics --- Statistics,
Catholic University of Eichstätt--Ingolstadt (KU),
Ostenstraße 28,
85072 Eichstätt,
Germany}
\email{thomas.heindl@mail.de}
\thanks{}

\author{Felix Voigtlaender}\thanks{}
\address[F. Voigtlaender]{Mathematical Institute for Machine Learning and Data Science (MIDS),
Catholic University of Eichstätt--Ingolstadt (KU), Auf der Schanz 49, 85049 Ingolstadt, Germany}
\email{felix.voigtlaender@ku.de}



\date{\today}

\begin{abstract}    
Empirical studies have widely demonstrated that neural networks are highly sensitive to small,
adversarial perturbations of the input. %
The worst-case robustness against these so-called adversarial examples can be quantified by the Lipschitz constant of the neural network.
In this paper, we study upper and lower bounds for the Lipschitz constant of \textit{random ReLU neural networks}.
Specifically, we assume that the weights and biases follow a generalization of the He initialization, where general symmetric
distributions for the biases are permitted.
For deep networks of fixed depth and sufficiently large width, our established upper bound is larger than the lower bound by a factor that is logarithmic in the width.
In contrast, for shallow neural networks we characterize the Lipschitz constant up to an absolute numerical constant
that is independent of all parameters.
\end{abstract}

\maketitle

\section{Introduction}

Deep neural networks have achieved remarkable success across a diverse range of applications in the past decade. 
However, it is well-known that small, adversarially chosen perturbations of the input can drastically alter the output of the neural network.
These perturbations, which often can be chosen to be invisible to the human eye,
are known in the literature as \textit{adversarial examples}.
One of the first papers to empirically examine this phenomenon was \cite{szegedy2013intriguing}.
This observation has lead to an extensive body of research,
focusing both on devising increasingly sophisticated adversarial attacks
(see, for example,  \cite{papernot2016limitations,moosavi2016deepfool,kurakin2018adversarial,athalye2018obfuscated,carlini2018audio})
and on constructing robust defenses against such attacks
(see, for example, \cite{madry2018towards,schott2018towards,pang2019improving,dong2020adversarial,bui2021unified}).
Arguably, this has resulted in an "arms race" between adversarial attacks and defenses. 
The theoretical understanding of adversarial examples, however, remains limited.

A natural measure for the worst-case robustness against adversarial examples is given
by the Lipschitz constant of a neural network function $\Phi: \mathbb{R}^d \rightarrow \mathbb{R}$,
which is defined via
\begin{equation*}
  \lip (\Phi )
  :=
  \underset{x,y \in \mathbb{R}^d, \, x \neq y}{\sup}  \frac{\vert \Phi (x)- \Phi (y) \vert}{\Vert x - y \Vert_2},
\end{equation*}
where by $\Vert \cdot \Vert_2$ we denote the $\ell_2$-norm of a vector.
Indeed, the Lipschitz constant measures how much a small change of the input can perturb the output. 
A small Lipschitz constant may imply that no adversarial examples exist,
whereas a large Lipschitz constant is indicative of the existence of adversarial examples.
For these reasons, several works have proposed algorithms to determine
upper bounds for the Lipschitz constant of a trained deep neural network---see,
e.g., \cite{fazlyab2019efficient,ebihara2023local,shi2022efficiently}---in order to quantify the worst-case robustness against adversarial examples.

\emph{Random neural networks}, i.e., neural networks whose weights and biases are chosen at random,
are often used as a first step in theoretically understanding empirically observed phenomena in deep learning;
see for instance \cite{dirksen2022separation,bartlett2021adversarial,daniely2020most}.
In this spirit, in this paper, we establish upper and lower bounds for the Lipschitz constant of random ReLU neural networks.
Our paper makes the following two main contributions:
\begin{enumerate}
  \item In the case of shallow neural networks with ReLU activation function
        and random weights following a variant of the He initialization \cite{he2015delving},
        we prove high-probability upper and lower bounds for the Lipschitz constant.
        These bounds match up to an absolute numerical constant.

  \item In the case of deep neural networks with ReLU activation function,
        we again derive high-probability upper and lower bounds for the Lipschitz constant.
        For this, we assume that our networks are wide enough, i.e.,
        they satisfy $N\gtrsim L^2 d$, where $N$ is the width of the neural network,
        $d$ is the dimension of the input, and $L$ is the depth of the neural network.
        If the depth $L$ is assumed to be a constant, our upper and lower bounds match up to a factor that is logarithmic in the width.
\end{enumerate}
Regarding the biases, in our results we allow arbitrary symmetric distributions, 
whereas earlier works on the Lipschitz constant of neural networks \cite{nguyen2021tight,buchanan2021deep}
require all biases to be initialized to zero, as in the classical He initialization.

The paper is structured as follows:
\Cref{sec:Background} provides technical background regarding random ReLU networks.
In Section \ref{sec:mainresults}, we present our main theoretical results
and in Section \ref{sec:preliminaries} we review some technical tools needed for our proofs.
The proofs regarding the upper bound of the Lipschitz constant are given in \Cref{sec:upper}. 
The corresponding lower bound is derived in \Cref{sec:low_bound}.

\subsection{Related Work}\label{subsec:relatedwork}

In \cite{bubeck2021law}, it has been conjectured that any shallow neural network
(with one hidden layer)  which interpolates the training data must have large Lipschitz constant,
unless the neural network is extremely overparameterized.
In \cite{bubeck2021universal}, this conjecture was proven and generalized to a much broader class
of models beyond (shallow) neural networks. 
However, while this result states that neural networks interpolating the data
can only be adversarially robust if they are extremely overparameterized,
it does not imply that overparameterized networks are inherently adversarially robust;
this is because only a \emph{lower} bound for the Lipschitz constant (which tends to zero as the number of network parameters goes to infinity) is presented.
Moreover, when the data is no longer \emph{perfectly} interpolated,
the lower bound for the Lipschitz constant in \cite{bubeck2021universal} no longer applies.
This leaves open the possibility of training adversarially robust neural networks
without huge overparameterization under such conditions.
For these reasons, a more detailed analysis of the Lipschitz constant of neural networks is essential.

\emph{Upper} bounds for the Lipschitz constants of random neural networks were derived in \cite{nguyen2021tight,buchanan2021deep}
in the context of the lazy training regime \cite{chizat2019lazy}. 
In contrast to our work, these two results do not contain any lower bounds. 
We refer to \Cref{remark:compare} for a detailed comparison of the results in \cite{nguyen2021tight,buchanan2021deep} with our bounds.

A \emph{lower} bound on the gradient of random ReLU networks at a fixed point was derived in \cite{bartlett2021adversarial}
in the context of proving the existence of an adversarial example at a fixed point. This result immediately implies a lower bound on the Lipschitz constant as well,
since random ReLU networks are almost surely differentiable at a fixed point (cf. \Cref{app:diff}).
We refer to \Cref{remark:compare_lower} for a comparison to the lower bound that we derive.

Our analysis of the Lipschitz constant is based on bounding the \emph{supremum} of the norm of the gradients of a random ReLU network.
Regarding the gradient at a fixed point, \cite{hanin2020products} even characterized the asymptotic distribution of the norm 
of the directional derivative in a fixed direction, which is shown to be a log normal distribution.
However, similar to \cite{nguyen2021tight} this bound only considers the feature maps, i.e., the neural network excluding the
output layer.  

A line of research related to studying the Lipschitz constant of neural networks is concerned with proving the existence
of adversarial perturbations in \emph{random} neural networks.
The first result was obtained in \cite{shamir2019simple}, where it was shown that
by only perturbing a few entries of the input vector one can construct an adversarial example.
Realizing that these perturbations may be very large,
this result was further improved in \cite{daniely2020most},
where it was shown that adversarial examples can be constructed with much smaller perturbations.
However, this result applies only to random ReLU networks for which the width of each layer
is small relative to the width of the previous layer and for which the penultimate layer still needs to be large enough (depending on the input dimension).
In the case of shallow networks, \cite{bubeck2021single} improved the latter result for networks
with ReLU activation or smooth activation whose width depends subexponentially on the input dimension.
The follow-up work \cite{bartlett2021adversarial} extended this result to the multi-layer setting.
Recently, this has been further improved in \cite{montanari2023adversarial}
which presented a result that requires no restriction on the width of the neural network
and which applies to general non-constant locally Lipschitz continuous activations.
All of the above results, however, are of a somewhat different nature compared to the results of the paper at hand.
The above results show, for an arbitrarily \emph{fixed} input, the existence,
with high probability, of a small adversarial perturbation. 
In contrast, our paper considers the adversarial robustness over all possible inputs and,
thereby, characterizes worst-case adversarial robustness.

Finally, let us note that going beyond random neural networks,
in \cite{wang2022adversarial} the existence of adversarial examples was recently shown
for trained neural networks in the lazy training regime \cite{chizat2019lazy}.
Moreover, in \cite{vardi2022gradient,frei2023double} it was shown that in neural networks
with one hidden layer, the bias of gradient flow towards margin maximization induces
the existence of adversarial examples.
In \cite{bombari2023beyond}, adversarial robustness of (trained) random features
and neural tangent kernel regression models was studied.
Here, adversarial robustness was measured with respect to a sensitivity measure,
which is the $\ell_2$-norm of the gradient at a data point scaled by the $\ell_2$-norm of this data point.
 This sensitivity measure was analyzed for one fixed vector from the input data distribution,
 i.e., the \emph{average-case} adversarial robustness was analyzed.

\subsection{Notation}

For any mathematical statement $A$, we set $\mathbbm{1}_A =1$ if $A$ is true and $\mathbbm{1}_A = 0$ otherwise.
The notation $\mathbbm{1}_A$ will also be used for the indicator function of a set $A$. 
We write $a_+ \defeq \max \{0,a\}$ for $a \in \RR$.
{\color{black} For two quantities $\alpha$ and $\beta$ we write $\alpha \lesssim \beta$ if there exists an absolute constant $C>0$ with $\alpha \leq  C\cdot\beta$.
Moreover, we write $\alpha \gtrsim \beta$ if $\beta \lesssim \alpha$ and $\alpha \asymp \beta$ if both $\alpha \lesssim \beta$ and $\beta \lesssim \alpha$.}

For two vectors $x,y \in \RR^k$, we write $\langle x,y \rangle \defeq \sum_{i=1}^k x_iy_i$
for the inner product of $x$ and $y$.
Moreover, we denote by $\Vert x \Vert_2 \defeq \sqrt{\langle x, x \rangle}$ the \emph{Euclidean norm} of $x$. 
Given a metric space $T$ with metric $\varrho$, for $x \in T$ and $r>0$ we write
\begin{equation*}
  B^\varrho_{T} (x,r)
  \defeq \left\{ y \in T : \ \varrho(x,y) < r\right\},
  \quad
  \overline{B}^\varrho_{T} (x,r)
  \defeq \left\{ y \in T : \ \varrho(x,y) \leq r\right\}.
\end{equation*}
When considering the Euclidean norm on $\RR^k$, we denote the open and closed ball by $B_k(x,r)$
and $\overline{B}_k(x,r)$, respectively.
Furthermore, we denote $\mathbb{S}^{k-1} \defeq \left\{ x \in \RR^k: \ \Vert x \Vert_2 = 1\right\}$.
For a subset $T \subseteq \RR^k$, we denote the diameter of $T$ as
$\diam(T) \defeq \underset{x,y \in T}{\sup} \Vert x - y \Vert_2 \in [0, \infty]$. 

For a matrix $A \in \RR^{k_1 \times k_2}$, we define 
\begin{equation*}
  \Vert A \Vert_2 \defeq \underset{ x \in \B_{k_2}(0,1)}{\sup} \Vert Ax \Vert_2.
\end{equation*}
 By $A_{i,-}$ and $A_{-,j}$ we denote the $i$-th row and $j$-th column of $A$, respectively.
 We write $v \odot w \in \RR^k$ for the \emph{Hadamard product} of two vectors $v, w \in \RR^k$,
 which is defined as $(v \odot w)_j \defeq v_j \cdot w_j $ for $j \in \{1,...,k\}$.

For any set $A$ and an element $a \in A$, we denote by $\delta_a$ the \emph{Dirac-measure} at $a$.
For $\sigma > 0$, we write $\mathcal{N}(0, \sigma^2)$ for the standard normal distribution
with expectation $0$ and variance $\sigma^2$. 
\paul{We write $\mathcal{N}(0, I_k)$ for the distribution of a $k$-dimensional random vector with independent $\mathcal{N}(0,1)$-entries.}
Moreover, we write $\EE[X]$ for the expectation of a random variable $X$.
By $\PP^X$ we denote the probability measure induced by some random variable $X$,
i.e., $\PP^X(A) = \PP(X \in A) = \PP (X^{-1}(A))$. Furthermore, we denote $A^c$ for the complement of any event $A$.

\section{Background}\label{sec:Background}

In the present paper, we derive bounds on the Lipschitz constant of random ReLU networks at initialization.
We consider a variant of the so-called \emph{He initialization} as introduced in \cite{he2015delving}.
In \cite{he2015delving}, the entries of the weight matrices $W^{(\ell)} \in \RR^{N^{(\ell + 1)} \times N^{(\ell)}}$ are identically and independently drawn from a Gaussian distribution
with zero mean and standard deviation $\sqrt{2/N^{(\ell)}}$
and the biases are initialized to zero.
We consider a standard deviation of $\sqrt{2/N^{(\ell+1)}}$, which is commonly used in theoretical studies of random ReLU networks (cf. \cite{buchanan2021deep,allen2019convergence,dirksen2022separation}).
{\color{black}This choice of the random initialization slightly simplifies the proofs and is natural since with this initialization, each layer is \emph{isometric in expectation}, 
meaning that $\EE [\Vert \relu(W^{(\ell)}x)\Vert_2^2] = \Vert x \Vert_2^2$ 
for every vector $x \in \RR^{N^{(\ell)}}$ that is fed into the $\ell$th layer of the network.
Since the $\relu$ is positively homogeneous, our results readily imply corresponding bounds for other initialization choices where the network weights are normally distributed, possibly even with a different variance in different layers.
In particular, for the usual He-initialization, one has to multiply all our upper/lower bounds for Lipschitz constants by $1/\sqrt{d}$.}

Moreover, we allow more general biases than in \cite{he2015delving}.
Specifically, we require that the biases are also drawn independently from certain probability distributions.
In order to derive upper bounds for the Lipschitz constant (see \Cref{sec:upper}) 
and lower bounds in the case of shallow networks (see \Cref{sec:low_bound_shallow}),
we do not need to impose \emph{any} assumptions on these probability distributions.
To derive lower bounds in the case of deep networks (see \Cref{subsec:deep_lower}), however,
we require these probability distributions to be symmetric.
Note that assuming symmetry of the distributions of the biases is very natural and in particular covers
the zero initialization that is used in \cite{he2015delving}.
In the following assumption, we formally introduce the considered random initialization.

\begin{assumption} \label{assum:1}
We consider ReLU networks with $d$ input neurons, 1 output neuron, a width of $N$ and $L$ hidden layers.
Precisely, we consider maps 
\begin{equation} \label{eq:relu-network}
  \Phi: \quad \RR^d \to \RR,
  \quad
  \Phi(x)
  \defeq \left( V^{(L)} \circ \relu \circ V^{(L-1)} \circ  \hdots \circ \relu \circ V^{(0)}\right) (x).
\end{equation}
Here, $\relu$ denotes the function defined as 
\begin{equation*}
  \relu (x) \defeq \max \{0,x\}, \quad x \in \RR,
\end{equation*} 
and the application of $\relu$ in \eqref{eq:relu-network} is to be understood componentwise, i.e., 
\begin{equation*}
  \relu ((x_1, ..., x_N)) = (\relu (x_1),..., \relu(x_N)). 
\end{equation*}
Furthermore, the maps $V^{(\ell)}$ for $0 \leq \ell \leq L$ are affine-linear maps,
which in detail means the following:
There are matrices $W^{(0)} \in \RR^{N \times d}$, $W^{(\ell)} \in \RR^{N \times N}$
for $1 \leq \ell \leq L-1$ and $W^{(L)} \in \RR^{1 \times N}$
as well as \emph{biases} $b^{(\ell)} \in \RR^N$ for $0 \leq \ell < L$ and $b^{(L)} \in \RR$ such that
\begin{equation*}
  V^{(\ell)} (x) = W^{(\ell)}x + b^{(\ell)}
  \quad \text{for every } 0 \leq \ell \leq L.
\end{equation*}

We assume that the matrices $W^{(\ell)}$ as well as the biases $b^{(\ell)}$ are randomly
chosen in the following way:
For $0 \leq \ell < L$, we have 
\begin{equation*}
  \left( W^{(\ell)}\right)_{i,j} \overset{\mathrm{i.i.d.}}{\sim} \mathcal{N}(0,2/N),
  \quad
  \left( W^{(L)}\right)_{1,j} \overset{\mathrm{i.i.d.}}{\sim} \mathcal{N}(0,1).
\end{equation*}
Concerning the biases, we assume
\begin{equation*}
  \left( b^{(\ell)}\right)_i \sim \mathcal{D}^{(\ell)}_i,
  \quad 0 \leq \ell \leq L.
\end{equation*}
Here, each $\mathcal{D}^{(\ell)}_i$ is an arbitrary probability distribution over $\RR$.
Furthermore, the entries of the random variables $W^{(0)}, ..., W^{(L)}, b^{(0)}, ..., b^{(L)}$ are assumed to be jointly independent.
\end{assumption}

The above assumption, which suffices for deriving the upper bound on the Lipschitz constant for networks of arbitrary depth and the lower bound in the case of shallow networks,
imposes almost no restrictions on the distribution of the biases.
However, for deriving lower bounds \paul{for deep networks} we will use the following more restrictive assumption
on the initialization of the biases.
\begin{assumption}\label{assum:2}
  Let $\Phi: \RR^d \to \RR$ be a random ReLU network satisfying \Cref{assum:1}.
  Moreover, we assume that each $\mathcal{D}_i^{(\ell)}$ is a symmetric probability distribution over $\RR$
  (e.g., $\mathcal{N}(0, \sigma^2)$ for some parameter $\sigma > 0$, uniform on some interval $[-m,m] \subseteq \RR$, or initialized to zero).
\end{assumption}

We restrict ourselves to square matrices
(i.e., we assume that each layer of the network has the same amount of neurons).
We believe, however, that our proofs still work in the case of layers with varying widths
that are uniformly bounded from below by $\Omega (dL^2)$.

\section{Main results}\label{sec:mainresults}

In this section, we discuss the main results of the paper.
In the case of shallow networks, we present upper and lower bounds on the Lipschitz constant
that hold almost without any restrictions on the input dimension $d$ and
the network width $N$.
Moreover, the upper and lower bounds match up to an absolute multiplicative constant.
In the case of deep networks, we have to impose the additional assumption $N \gtrsim d L^2$. 
Furthermore, there is a gap between the upper and lower bounds;
this gap grows exponentially with the depth $L$ of the network
and logarithmically with the network width $N$.
\subsection{Shallow networks}

We provide a high-probability version and an expectation version for the upper and lower bounds. 
The result for the upper bound reads as follows:

\begin{theorem}\label{thm:main_1}
  There exist absolute constants $C, c_1>0$ satisfying the following:
  If $\Phi: \RR^d \to \RR$ is a random shallow ReLU network of arbitrary width $N \in \NN$
  with random weights and random biases following \Cref{assum:1}, then the following hold:
  \begin{enumerate}
    \item We have
          \begin{align*}
            \lip(\Phi) \leq C \cdot \sqrt{d} 
          \end{align*}
          with probability at least $(1-2\exp(-\min\{d,N\})) \cdot (1-2\exp(-c_1 \cdot \max\{d,N\}))_+$.
    
    \item $\displaystyle \EE \left[\lip(\Phi) \right] \leq C  \cdot \sqrt{d}.$
  \end{enumerate}
\end{theorem}

The proof of \Cref{thm:main_1} can be found in \Cref{sec:shallow}.
 
The upper bound from above is complemented by the following lower bound,
which indeed shows that the bounds are tight up to multiplicative constants.

\begin{theorem}\label{thm:main_shallow_lower}
  There exist absolute constants $c,c_1>0$ satisfying the following: If $\Phi: \RR^d \to \RR$ is a random shallow ReLU network with width $N$ and with random weights and biases following \Cref{assum:1}, then the following hold:
\begin{enumerate}
\item{The inequality
\begin{equation*}
\lip(\Phi) \geq \frac{1}{4\sqrt{2}} \cdot \sqrt{d}
\end{equation*}
holds with probability at least $(1- 2 \exp(-c_1N))_+(1-2\exp(-c_1d))_+$.}
\item{ If $d,N > \frac{\ln(2)}{c_1}$, it holds
\begin{equation*}
\EE [\lip(\Phi)] \geq c \cdot \sqrt{d}.
\end{equation*}
}
\end{enumerate}
\end{theorem}

The proof of \Cref{thm:main_shallow_lower} can be found in \Cref{sec:low_bound_shallow}.

The two results combined show that in the case of shallow networks it holds
\begin{equation*}
  \lip(\Phi) \asymp \sqrt{d},
\end{equation*}
where the involved constants are absolute.
This holds with high probability and in expectation and moreover
without almost any restrictions on the width $N$ or the input dimension $d$.

\subsection{Deep networks}

In the case of deep networks, our upper bound for the Lipschitz constant
reads as follows:

\begin{theorem}\label{thm:main_2}
  There exist absolute constants $C, c_1 > 0$ such that, if $N > d + 2$ and $\Phi: \RR^d \to \RR$
  is a random ReLU network of width $N$ and with $L$ hidden layers with random weights
  and biases following \Cref{assum:1}, the following hold:
  \begin{enumerate}
  \item The inequality
        \begin{align*}
          \lip(\Phi)
          \leq C \cdot (3\sqrt{2})^{L}
                 \cdot \sqrt{L}
                 \cdot \sqrt{\ln \left(\frac{\ee N}{d+1}\right)}
                 \cdot \sqrt{d} 
        \end{align*}
        holds with probability at least $(1-2\exp(- d))\left((1-2\exp(-c_1N))_+\right)^L$.
        \vspace*{0.1cm}
  
  \item \(
          \displaystyle
          \EE \left[\lip(\Phi) \right]
          \leq C \cdot (2 \sqrt{2})^{L}
                 \cdot \sqrt{L}
                 \cdot \sqrt{\ln \left(\frac{\ee N}{d+1}\right)}
                 \cdot \sqrt{d}.
        \)
  \end{enumerate}
\end{theorem}

The proof of \Cref{thm:main_2} can be found in \Cref{sec:deep}.
We remark that for the upper bounds in the case of deep networks above,
we impose the assumption $N > d+2$ on the network width, which is not present
in our results for shallow networks.

Again, we complement the upper bound by a lower bound,
which shows---at least for fixed depth $L$---that the bounds are tight up to
a factor that is logarithmic in the width of the network.

\begin{theorem}\label{thm:main_3}
  There exist absolute constants $C, c,c_1 > 0$ with the following property: 
  If $N \geq CdL^2$ and if $\Phi:\RR^d \to \RR$ is a random ReLU network of width $N$, with $L$ hidden layers 
  with random weights and biases following \Cref{assum:2}, 
  then the following two properties are satisfied:
\begin{enumerate}
\item It holds
\begin{equation*}
\lip(\Phi) \geq \frac{1}{4} \sqrt{d}
\end{equation*}
with probability at least $\left(1 - \frac{1}{2^N}-\exp(-N/(CL^2))\right)^L \cdot (1- 2 \exp (-c_1 d ))_+$. 
\item If we additionally assume $N \geq CL^2 \ln(4L)$ and $d> \frac{\ln(2)}{c_1}$, it holds \begin{equation*}
\EE \left[\lip(\Phi)\right] \geq c \cdot \sqrt{d}.
\end{equation*}
\end{enumerate}
\end{theorem}

\Cref{thm:main_3} is proven in \Cref{subsec:deep_lower}.
At first glance, it might be irritating that the probability in (1) depends on $N/(CL^2)$.
Here, it should be noted that the condition $N \geq CdL^2$ is equivalent to $N/(CL^2) \geq d$.
Therefore, the bound on the Lipschitz constant in particular holds with high probability depending
on the input dimension $d$.
We remark that in contrast to the upper bound, where we assume that $N > d+2$,
for the lower bound we assume that $N \gtrsim d L^2$.
Moreover, we assume that the biases are drawn from symmetric distributions,
\paul{whereas this assumption is not needed in \Cref{thm:main_2,thm:main_shallow_lower,thm:main_1}. 
However, we emphasize once more that the assumption that the biases are drawn from symmetric distributions
is natural and includes in particular the initialization to zero (as originally introduced in \cite{he2015delving}), 
a Gaussian initialization, a symmetric uniform initialization as well as a Rademacher initialization.}
{\color{black}Moreover, we note that the fact that the numerical constant in \Cref{thm:main_3}(1) is better than the one in \Cref{thm:main_shallow_lower}(1) is merely a proof artifact 
since we did not try to optimize numerical constants.}

The above results provide bounds on the Lipschitz constant of deep ReLU networks.
The upper and lower bounds are tight up to a factor that
depends exponentially on the depth $L$, but is only logarithmic in $N/d$.

Note that for both the upper and lower bound for \emph{deep} networks, we impose some additional assumptions
on the input dimension $d$, the width $N$, and the depth $L$,
whereas no such assumption is needed in the case of shallow networks.
Specifically, for the upper bound we require $N > d+2$ and for the lower bound $N \gtrsim d L^2$.

\begin{remark}\label{remark:compare}
In this remark, we compare the upper bounds for the Lipschitz constant of random $\relu$
networks which we derive in this work (see \Cref{thm:main_1,thm:main_2})
to the bounds that were shown in \cite{buchanan2021deep,nguyen2021tight}
in the context of the lazy training regime \cite{chizat2019lazy}.

We first note that neither of the papers \cite{buchanan2021deep,nguyen2021tight}
provides \emph{lower} bounds for the Lipschitz constant.

The analysis in \cite{nguyen2021tight} concentrates on bounding the Lipschitz constant
of the feature maps and does not include the final layer.
In our setting ($L$ hidden layers, each with $N$ neurons and input dimension $d$
with weights following the variant of the He initialization),
the bound in \cite[Theorem~6.2]{nguyen2021tight} shows that the output of the penultimate layer
has Lipschitz constant at most $\mathcal{O} (2^L \max \{1, \sqrt{d/N}\})$ up to log factors,
with high probability.
If one naively combines this bound with the (sharp) bound of $\sqrt{N}$
for the Lipschitz constant of the function mapping the output of the penultimate layer
to the output of the network, one achieves a bound of $\mathcal{O}(2^L\max\{\sqrt{N}, \sqrt{d}\})$,
which is (for fixed depth $L$ and $N \gg d$) significantly weaker
than the bound of $\mathcal{O}({C^L\sqrt{d}})$ (up to log factors) that we derive.
Moreover, the analysis in \cite{nguyen2021tight} heavily uses the fact that the biases are all zero,
whereas our techniques can handle quite general distributions of the biases. 
 
A sharper version of this bound has been obtained as an auxiliary result in \cite[Theorem~B.5]{buchanan2021deep}: 
Under the additional assumption $N \gtrsim d^4L \ln^4 (N)$,
this result implies that the Lipschitz constant is bounded by $\mathcal{O}(\sqrt{d})$ up to log factors,
with high probability.
In particular, it is remarkable that the bound is independent of the network depth $L$,
whereas our bound depends \emph{exponentially} on $L$. 
Removing the exponential dependence on $L$ in our setting is interesting but beyond the scope of the present work. 
Adapting the techniques from \cite{buchanan2021deep} to our setting
with quite general distributions of the biases is not straightforward,
since \cite{buchanan2021deep} heavily uses that all biases are zero
in order to reduce the problem of bounding the global Lipschitz constant
to bounding the Lipschitz constant on the sphere.
Finally, we point out that the upper bound in \cite{buchanan2021deep} implies
that our lower bound of $\Omega(\sqrt{d})$ is sharp,
at least in the case of all biases being zero and $N \gtrsim d^4L \ln^4(N) + dL^2$.
\end{remark}
\begin{remark}\label{remark:compare_lower}
This remark is dedicated to comparing the lower bounds that we derive in this work (see \Cref{thm:main_3,thm:main_shallow_lower})
with a lower bound on the norm of the gradient of random ReLU networks at a fixed point that was derived in \cite[Lemma~2.2]{bartlett2021adversarial}.
Noting that the initialization scheme considered in \cite{bartlett2021adversarial} is slightly different from the setting that we consider,
transferred to our setting the lower bound from \cite{bartlett2021adversarial} states that the norm of the gradient at a single fixed point is lower bounded by $\Omega\left(\sqrt{d} \cdot \sqrt{2}^{-L}\right)$ with high probability. 
This lower bound immediately implies the same lower bound for the Lipschitz constant of random ReLU
networks, since such networks are almost surely differentiable at a fixed point (see \Cref{app:diff}). While this lower bound is for fixed depth $L$ at par 
with the lower bound of $\Omega(\sqrt{d})$ that we derive, it is important to note that our bound does \emph{not} depend on the depth $L$ of the network,
whereas the bound from \cite{bartlett2021adversarial} has \emph{exponential} decay in $L$. 
Moreover, in \cite{bartlett2021adversarial} it is also assumed that all the biases are equal to zero, whereas we allow more general distributions for the biases.

In principle, it might be possible to adapt the techniques from \cite{bartlett2021adversarial} to derive a lower bound
that is similar to ours but this is not done in \cite{bartlett2021adversarial}. 
As far as we are aware, our paper is the first to state and derive a sharp lower bound matching the upper bound from \cite{buchanan2021deep}.
\end{remark}

\section{Technical preliminaries}\label{sec:preliminaries}

In this section, we discuss the necessary technical preliminaries for our proofs.

\subsection{The gradient of ReLU networks}
\label{subsec:gradient}

To obtain bounds on the Lipschitz constant of a neural network,
we will use the gradient of the network.
To conveniently express this gradient, we introduce the following notation:
For a vector $x \in \RR^N$ we define the diagonal matrix $\diag(x) \in \RR^{N \times N}$ via
\begin{equation*}
  \left(\diag(x)\right)_{i,j}
  \defeq \begin{cases}
            \mathbbm{1}_{x_i > 0},& i=j, \\
            0,                    &i \neq j .
          \end{cases}
\end{equation*}
This leads to the following recursive representation of a ReLU network $\Phi$
with the same notation as in \Cref{assum:1}:
Let $x =: x^{(0)} \in \RR^d$ and define recursively 
\begin{align} \label{eq:d-matrices}
  D^{(\ell)}(x)
  &\defeq \diag(W^{(\ell)}x^{(\ell)} + b^{(\ell)}), \nonumber\\
  x^{(\ell +1)}
  &\defeq D^{(\ell)}(x) \cdot \left(W^{(\ell)}x^{(\ell)} + b^{(\ell)}\right)
         = \relu(W^{(\ell)}x^{(\ell)} + b^{(\ell)}),
  \quad 0 \leq \ell < L.
\end{align}
Then it holds $\Phi(x) = W^{(L)}x^{(L)} + b^{(L)}$.

A ReLU network is not necessarily everywhere differentiable,
since the ReLU itself is not differentiable at $0$.
Nevertheless, the following proposition states how the gradient of a ReLU network
can be computed almost everywhere.

\begin{proposition}[cf.~{\cite[Theorem III.1]{berner2019towards}}]\label{prop:grad_relu}
Let $\Phi : \RR^d \to \RR$ be a ReLU network. Then it holds for almost every $x \in \RR^d$ that $\Phi$ is differentiable at $x$ with
\begin{equation*}
\left(\nabla \Phi (x) \right)^T = W^{(L)} \cdot D^{(L-1)}(x) \cdot W^{(L-1)} \cdots D^{(0)}(x)\cdot W^{(0)}.
\end{equation*}
Here, the matrices $D^{(\ell)}(x)$ are defined as in \eqref{eq:d-matrices} and $\Phi$ is any realization of a random ReLU network as in \Cref{assum:1}. 
\end{proposition}

\subsection{The Lipschitz constant of ReLU networks}
\label{subsec:lip}

The following well-known proposition establishes a relation between the Lipschitz constant
and the gradient of a function.

\begin{proposition}\label{prop:lipgrad}
  Let $f: \RR^d \to \RR$ be Lipschitz continuous and $M \subseteq \RR^d$ any measurable subset
  of $\RR^d$ with Lebesgue measure $\lambda^d (\RR^d \setminus M) = 0$
  such that $f$ is differentiable in every $x \in M$.
  Then it holds that
  \begin{equation*}
    \lip(f) = \underset{x \in M}{\sup} \Vert \nabla f (x) \Vert_2.
  \end{equation*}
\end{proposition} 

For the sake of completeness, we provide a proof of \Cref{prop:lipgrad} in \Cref{sec:prelim_proofs}.
It should be noted that the existence of a set $M$
as in the proposition follows from the fact that $f$ is Lipschitz continuous;
this is known as Rademacher's theorem (cf.\ \cite[Section~3.1.2]{evans_measure_1992}). 

Every ReLU network $\Phi$ is Lipschitz continuous as the composition of Lipschitz continuous functions.
Hence, from \Cref{prop:grad_relu,prop:lipgrad}, we infer 
\begin{equation}\label{eq:lowbound}
\lip(\Phi)
  \leq \underset{x \in \RR^d}{\sup}
         \Vert
           W^{(L)} \cdot D^{(L-1)}(x) \cdot W^{(L-1)} \cdots D^{(0)}(x)\cdot W^{(0)}
         \Vert_2.
\end{equation}
We remark that in general, one does \emph{not} necessarily have equality
in \eqref{eq:lowbound}, since the supremum is taken over all of $\RR^d$,
including points at which $\Phi$ might not be differentiable.

As an example, consider the shallow ReLU network $\Phi:\RR^2 \to \RR$ with three hidden neurons 
built from the matrices $W^{(0)}\defeq \begin{pmatrix}1 & -1 \\ -1 & 1 \\ 2 & -1\end{pmatrix}$ and $W^{(1)} \defeq \begin{pmatrix} -1&1&1\end{pmatrix}$ and all biases equal to zero. 
Then a direct computation shows that for each vector $(x,y) \in \RR^2$ it holds
\begin{equation*}
\left(W^{(1)}D^{(0)}(x,y)W^{(0)}\right)^T
= \begin{pmatrix}
    -\mathbbm{1}_{x>y}- \mathbbm{1}_{y > x} + 2\mathbbm{1}_{2x>y} \\
    \mathbbm{1}_{x>y} + \mathbbm{1}_{y>x}- \mathbbm{1}_{2x>y}\end{pmatrix}
= \begin{cases}
    \begin{pmatrix}1\\0\end{pmatrix},  & x\neq y, 2x>y, \\[0.4cm]
    \begin{pmatrix}-1\\1\end{pmatrix}, & x\neq y, 2x\leq y, \\[0.4cm]
    \begin{pmatrix}2\\-1\end{pmatrix}, & x= y, 2x>y, \\[0.4cm]
    \begin{pmatrix}0\\0\end{pmatrix},  & x= y, 2x\leq y.
\end{cases}
\end{equation*}
Since $\Phi$ is Lipschitz continuous, differentiable on $\{(x,y) \in \RR^2: \ x\neq y \ \text{and} \ y\neq 2x\}$ and this set has full measure, we conclude using \Cref{prop:lipgrad} that $\lip(\Phi) = \sqrt{2}$. On the other hand, the above computation shows \begin{equation*}
\underset{(x,y) \in \RR^2}{\sup} \Vert W^{(1)}D^{(0)}(x,y)W^{(0)} \Vert_2 = \sqrt{5}>\lip(\Phi).
\end{equation*}

The estimate \eqref{eq:lowbound} is essential for our derivation of \emph{upper} bounds
for Lipschitz constants of random ReLU networks.
On the other hand, for any $x_0 \in \RR^d$ with the property that $\Phi$ is differentiable at $x_0$ 
it holds
\begin{equation}\label{eq:upbound}
  \lip(\Phi) \geq \Vert \nabla \Phi(x_0) \Vert_2 \, ,
\end{equation}
which also follows directly from \Cref{prop:grad_relu,prop:lipgrad}. 
This will be useful to derive \emph{lower} bounds for Lipschitz constants of random ReLU networks.

For technical reasons, since we form expressions such as
$\EE [\lip (\Phi)]$, it is important to observe that the map
\begin{equation*}
  (W^{(0)}, \dots, W^{(L)}, b^{(0)}, \dots, b^{(L)}) \mapsto \lip(\Phi),
\end{equation*}
where $\Phi$ is the ReLU network built from $(W^{(0)}, \dots, W^{(L)}, b^{(0)}, \dots, b^{(L)})$,
is measurable.
This follows from the continuity of the map
\begin{equation*}
  (x,y) \mapsto \frac{\abs{ \Phi(x) - \Phi(y)} }{\Vert x - y \Vert_2}
\end{equation*}
on $\{ (x,y) \in \RR^{d} \times \RR^d \colon x \neq y \}$ for a fixed network $\Phi$,
combined with the separability of the latter set.

\subsection{Covering numbers and VC-dimension}

Let $(T, \varrho)$ be a metric space.
For $\varepsilon > 0$, we define the $\varepsilon$-\emph{covering number} of $T$ as
\begin{equation*}
  \mathcal{N}(T,\varrho,\varepsilon)
  \defeq \inf
         \left\{
           \vert K \vert : \ K \subseteq T, \ \bigcup_{k \in K} \overline{B}^\varrho_T(k,\eps) = T
         \right\}
   \in   \NN \cup \{\infty\},
\end{equation*} 
where $\vert K\vert$ denotes the cardinality of a set $K$.
Any set $K \subseteq T$ with $\bigcup_{k \in K} \overline{B}^\varrho_T(k,\eps) = T$
is called an $\eps$-\emph{net} of $T$.

It is well-known that the $\eps$-covering number of the unit ball in $k$-dimensional Euclidean space 
with respect to the $\Vert \cdot \Vert_2$-norm can be upper bounded by $\left(1 + \frac{2}{\eps}\right)^k$ (see, e.g., \cite[Corollary 4.2.13]{vershynin_high-dimensional_2018}). 
We, however, need a slightly modified version of this result, which we prove in \Cref{sec:prelim_proofs}.
\begin{proposition}\label{prop:covering_ball}
  Let $\eps>0$ and $V \subseteq \RR^k$ be a linear subspace of $\RR^k$.
  Then it holds 
  \begin{equation*}
    \mathcal{N}(\overline{B}_k(0,1) \cap V, \Vert \cdot \Vert_2, \eps) \leq \left(\frac{2}{\eps} + 1\right)^{\dim(V)}.
  \end{equation*}
\end{proposition}

For a set $\mathcal{F}$ of Boolean functions defined on a set $\Omega$ (i.e., $\mathcal{F} \subseteq \{f: \ \Omega \to \{0,1\}\}$),
we denote by $\vc(\mathcal{F})$ its \emph{VC-dimension}, i.e., 
\begin{equation*}
  \vc(\mathcal{F})
  \defeq \sup
         \left\{
           \abs{\Lambda}
           : \
           \Lambda \subseteq \Omega, \ 
           \left\vert \left\{ \fres{f}{\Lambda}: \ f \in \mathcal{F}\right\}\right\vert = 2^{\abs{\Lambda}}
         \right\}.
\end{equation*}
We refer to \cite[Chapter 6]{shalev2014understanding}
and \cite[Chapter 8.3]{vershynin_high-dimensional_2018} for details on the concept of the VC-dimension.

It is well-known that the VC-dimension of the class of homogeneous halfspaces in $k$-dimensional Euclidean space equals $k$ (see, e.g., \cite[Theorem 9.2]{shalev2014understanding}). 
As in the case of the covering number of Euclidean balls, we need a slightly modified version of this result. The proof is also deferred to \Cref{sec:prelim_proofs}.
\begin{proposition}\label{prop:vc_half_spaces_2}
Let $k \in \NN$ and $V \subseteq \RR^k$ be a linear subspace. 
For $\alpha \in \RR^k$, we define
\begin{equation*}
f_\alpha: \quad \RR^{k} \to \{0,1\},
    \quad x \mapsto \mathbbm{1}_{\alpha^Tx > 0}.
\end{equation*}
Let $\mathcal{F} \defeq \{\fres{f_\alpha}{V} : \ \alpha \in V\}$. Then it holds $\vc(\mathcal{F}) = \dim(V)$.
\end{proposition}
We will further make use of the following estimate, which enables us to bound the covering number of a class of Boolean functions with respect to $L^2(\mu)$ for some probability measure $\mu$ using the VC-dimension of this set.

\begin{proposition}[{cf.\ \cite[Theorem 8.3.18]{vershynin_high-dimensional_2018}}] \label{prop:covering_vc}
  There exists an absolute constant $C > 0$ with the following property:
  For any class of measurable Boolean functions $\mathcal{F} \neq \emptyset$ on some probability space $(\Omega, \mathscr{A}, \mu)$
  and for every $\eps \in (0,1)$, we have
  \begin{equation*}
    \mathcal{N}(\mathcal{F}, L^2(\mu), \eps)
    \leq \left(\frac{2}{\eps}\right)^{C \cdot  \vc(\mathcal{F})}.
  \end{equation*}
\end{proposition}

\subsection{Sub-gaussian random variables \paul{and Gaussian width}}

Sub-gaussian random variables occur frequently over the course of our paper;
we therefore briefly recall their definition in this subsection.
We refer to  \cite[Chapters 2.5 and 3.4]{vershynin_high-dimensional_2018}
for a detailed introduction to this topic.

A real-valued random variable $X$ is called \emph{sub-gaussian} if 
\begin{equation*}
  \PP (\vert X \vert \geq t ) \leq 2 \exp(-t^2 / C_1^2)
\end{equation*}
holds with a constant $C_1 > 0$ and all $t > 0$.
This is equivalent (see \cite[Proposition~2.5.2]{vershynin_high-dimensional_2018}) to the existence of a constant $C_2>0$ satisfying
\begin{equation}\label{eq:sub-gaussian}
  \EE \left[ \exp(X^2 / C_2^2)\right] \leq 2.
\end{equation}
The \emph{sub-gaussian} norm of $X$ (see \cite[Equation~(2.13)]{vershynin_high-dimensional_2018}) is defined as the infimum of all numbers
satisfying \eqref{eq:sub-gaussian}, i.e.,
\begin{equation*}
  \Vert X \Vert_{\psi_2}
  \defeq \inf \{t >0: \ \EE \left[ \exp(X^2 / t^2)\right] \leq 2\}.
\end{equation*}
Following \cite[Definition~3.4.1]{vershynin_high-dimensional_2018}, a $k$-dimensional \emph{random vector} $X$ is called sub-gaussian
if and only if $\langle X,x \rangle$ is sub-gaussian for every $x \in \RR^k$.
The sub-gaussian norm of $X$ is defined as
\begin{equation*}
  \Vert X \Vert_{\psi_2}
  \defeq \underset{x \in \mathbb{S}^{k-1}}{\sup} \Vert \langle X,x \rangle \Vert_{\psi_2}.
\end{equation*}

In \Cref{sec:low_bound}, we will use the notion of Gaussian width since this quantity appears in the high-probability version of the matrix deviation inequality \cite[Theorem~3]{Liaw2017}. Specifically, we will need to bound the Gaussian width of low-dimensional balls in a higher-dimensional space. Also here, we defer the proof to \Cref{sec:prelim_proofs}.
\begin{proposition}\label{prop:gauss_width}
Let $k \in \NN$ and $\emptyset \neq T \subseteq \RR^k$. Following \cite[Definition~7.5.1]{vershynin_high-dimensional_2018}, we then define
\begin{align*}
w(T) \defeq  \underset{g \sim \mathcal{N}(0, I_k)}{\EE}  \left[ \underset{x \in T}{\sup}\  \langle g, x \rangle \right].
\end{align*}
$w(T)$ is called the \emph{Gaussian width} of $T$. 

Let $V \subseteq \RR^k$ be a linear subspace of $\RR^k$. Then it holds
\begin{equation*}
w(\overline{B}_k(0,1) \cap V) \leq \sqrt{\dim(V)}.
\end{equation*}
\end{proposition}

\subsection{Omitting the ReLU activation}

It is natural to ask how the Lipschitz constant of a ReLU network is related to the Lipschitz constant of the corresponding linear network, i.e., using the identity function instead of the $\relu$ as the activation function.
Precisely, for a given random $\relu$ network $\Phi: \RR^d \to \RR$ that is defined as in \Cref{sec:Background}, we define the corresponding linear network via
\begin{equation}\label{eq:tildelinear}
\widetilde{\Phi} \defeq V^{(L)} \circ V^{(L-1)} \circ \cdots \circ  V^{(1)} \circ V^{(0)}.
\end{equation}
We first note that $\widetilde{\Phi}$ is an affine map with
\begin{equation*}
\lip(\widetilde{\Phi}) = \Vert W^{(L)}\cdot W^{(L-1)} \cdots W^{(1)} \cdot W^{(0)}\Vert_2.
\end{equation*}
At first glance, one might be tempted to think that the ReLU activation reduces the Lipschitz constant at most, since the ReLU itself is Lipschitz continuous with Lipschitz constant 1. 
The following proposition, however, demonstrates that this is in general \emph{not} the case, even for shallow networks. We defer the proof to \Cref{sec:prelim_proofs}.
\begin{proposition}\label{prop:Cbound}
Let $\Phi: \RR^2 \to \RR$ be a random shallow ReLU network with width $N=3$ satisfying \Cref{assum:1}. 
Moreover, let $C>0$ be arbitrary.
Then with positive probability it holds for $\widetilde{\Phi}$ as in \eqref{eq:tildelinear} that
\begin{equation*}
\lip(\Phi) > C \cdot \lip(\widetilde{\Phi}).
\end{equation*}
\end{proposition}
Hence, we see that in general it does \emph{not} hold
\begin{equation*}
\lip(\Phi) \lesssim \lip(\widetilde{\Phi}) \quad \text{almost surely.}
\end{equation*}
Nevertheless, \Cref{thm:main_1,thm:main_shallow_lower} show that at least for shallow networks, we indeed have that
\begin{equation*}
\lip(\Phi) \asymp \lip(\widetilde{\Phi}) \quad \text{with high probability}
\end{equation*} 
but proving this is nontrivial.

Remarkably, the converse estimate to what is considered in \Cref{prop:Cbound} even holds deterministically, at least for shallow networks.
\begin{proposition}\label{prop:shallow_low_linear}
Let $d, N \in \NN$ and let $\Phi: \RR^d \to \RR$ be a fixed (deterministic) shallow ReLU network with $N$ hidden neurons. Then it holds for $\widetilde{\Phi}$ as in \eqref{eq:tildelinear} that
\begin{equation*}
\lip(\Phi) \geq \frac{1}{2} \lip(\widetilde{\Phi}).
\end{equation*}
\end{proposition}
\begin{proof}
In this proof, we use a different notation for  shallow ReLU networks than the one introduced in \Cref{sec:Background}. 
Specifically, we write 
\begin{equation*}
\Phi(x) \defeq \sum_{i=1}^N \left[c_i \relu(\langle a_i, x \rangle + b_i)\right] + \beta \quad \text{and} \quad \widetilde{\Phi}(x) \defeq \sum_{i=1}^N \left[c_i \cdot (\langle a_i, x \rangle + b_i)\right] + \beta
\end{equation*}
with $a_1, ..., a_N \in \RR^d$, $b_1, ..., b_N, c_1, ..., c_N, \beta \in \RR$.

Let 
\begin{equation*}
v_0 \defeq \sum_{i=1}^N c_ia_i \in \RR^d.
\end{equation*}
If $v_0 = 0$, there is nothing to show, since in that case we have $\lip(\widetilde{\Phi}) = \Vert v_0 \Vert_2 = 0$. 
Hence, we assume $v_0 \neq 0$ and define $v \defeq \frac{v_0}{\Vert v_0 \Vert_2}$.
We denote $\alpha_i \defeq \langle a_i, v \rangle$ for $i \in \{1,...,N\}$ and define the sets
\begin{align*}
I_+ &\defeq \{i \in \{1,...,N\}: \ \alpha_i > 0\}, \\
 I_- &\defeq \{i \in \{1,...,N\}: \ \alpha_i < 0\}, \\
  I_0 &\defeq \{i \in \{1,...,N\}: \ \alpha_i = 0\} \quad \text{and} \\ 
I &\defeq I_+ \cup I_-.
\end{align*}
We then see for every $t \in \RR$ that 
\begin{align*}
\Phi(tv) &= \sum_{i \in I_+} \left[c_i \relu(t \langle a_i, v \rangle + b_i)\right] + \sum_{i \in I_-} \left[c_i \relu(t \langle a_i, v \rangle + b_i)\right] + \sum_{i \in I_0} \left[c_i \relu(t\langle a_i,v \rangle + b_i)\right] + \beta \\
&= \sum_{i \in I_+} \left[c_i \relu(t \alpha_i + b_i)\right] + \sum_{i \in I_-} \left[c_i \relu(t \alpha_i + b_i)\right] + \sum_{i \in I_0} \left[c_i \relu(b_i)\right] + \beta.
\end{align*}
Fix $t_0 \in \RR$ with 
\begin{equation*}
t_0 > \max \left\{-\frac{b_i}{\alpha_i}: \ i \in I\right\}.
\end{equation*}
For $i \in I_+$, this implies
\begin{equation}\label{eq:t_0_1}
t_0\alpha_i + b_i > \left(- \frac{b_i}{\alpha_i}\right) \alpha_i + b_i = 0 
\end{equation}
and for $i \in I_-$, we see
\begin{equation}\label{eq:t_0_2}
t_0\alpha_i + b_i < \left(- \frac{b_i}{\alpha_i}\right) \alpha_i + b_i = 0.
\end{equation}
By continuity, there exists $\delta > 0$ such that the inequalities \eqref{eq:t_0_1} and \eqref{eq:t_0_2} hold for any $t \in (t_0 - \delta, t_0 + \delta)$.
In particular, it holds
\begin{equation*}
\Phi(t v) = \sum_{i \in I_+} \left[c_i\cdot (t \alpha_i + b_i)\right]+\sum_{i \in I_0} \left[c_i \relu(b_i)\right] + \beta
\end{equation*}
for every $t \in (t_0 - \delta, t_0 + \delta)$. 
This means that $t \mapsto \Phi(tv)$ is differentiable on $(t_0 - \delta, t_0 + \delta)$ with
\begin{equation*}
\frac{\dd}{\dd t}\Big|_{t= t_0}\left[\Phi(tv)\right] = \sum_{i \in I_+} c_i \alpha_i.
\end{equation*}
We then get
\begin{equation*}
\left\vert \sum_{i \in I_+} c_i \alpha_i \right\vert = \left\vert \frac{\dd}{\dd t}\Big|_{t=t_0}\left[\Phi(tv)\right]\right\vert = \lim_{t \to t_0} \frac{\vert \Phi(tv) - \Phi(t_0v)\vert}{\vert t - t_0 \vert} = \lim_{t \to t_0} \frac{\vert \Phi(tv) - \Phi(t_0v)\vert}{\Vert tv - t_0v \Vert_2} \leq \lip(\Phi).
\end{equation*}
Similarly, by picking $t_0 < \min \left\{-\frac{b_i}{\alpha_i}: \ i \in I\right\}$, we get 
\begin{equation*}
\lip(\Phi) \geq \left\vert \sum_{i \in I_-} c_i \alpha_i \right\vert.
\end{equation*}
Hence, combining these two estimates, we arrive at
\begin{equation*}
\lip(\Phi) \geq \frac{1}{2} \left[ \left\vert \sum_{i \in I_+} c_i \alpha_i \right\vert + \left\vert \sum_{i \in I_-} c_i \alpha_i \right\vert\right] \geq \frac{1}{2} \left[\left\vert\sum_{i \in I} c_i \alpha_i  \right\vert\right] = \frac{1}{2} \left[\left\vert\sum_{i  = 1}^N c_i \alpha_i  \right\vert\right], 
\end{equation*}
where the last equality follows from $\alpha_i = 0$ for every $i \in I_0$. To get the final claim, note that
\begin{align*}
\left\vert\sum_{i  = 1}^N c_i \alpha_i  \right\vert = \left\vert\sum_{i  = 1}^N c_i \langle a_i, v \rangle  \right\vert =\left\vert \left\langle\sum_{i=1}^N c_i a_i, v \right\rangle \right\vert= \langle v_0, v \rangle = \Vert v_0 \Vert_2 = \lip(\widetilde{\Phi}), 
\end{align*}
as was to be shown. 
\end{proof}
This bound will be useful in order to derive a lower bound for the Lipschitz constant of \emph{shallow} ReLU networks. 
Specifically, in order to get a lower bound on the Lipschitz constant of shallow random ReLU networks, 
it suffices to establish a lower bound for $\Vert W^{(1)} W^{(0)} \Vert_2$ for a Gaussian matrix $W^{(0)} \in \RR^{N \times d}$ and a Gaussian row vector $W^{(1)} \in \RR^{1 \times N}$. We refer to \Cref{sec:low_bound_shallow} for the details. 

Unfortunately, an analogous bound does not hold for \emph{deep} networks anymore, even for depth $L =2$, as is stated in the following proposition, the proof of which can also be found in \Cref{sec:prelim_proofs}.
\begin{proposition}\label{prop:not_working_deep}
We consider a ReLU network $\Phi: \RR \to \RR$ with depth $L=2$ and width $N=1$ satisfying \Cref{assum:1}. Moreover, we assume that $\PP(b^{(1)} \leq 0)> 0$. Then, it holds with positive probability for $\widetilde{\Phi}$ as in \Cref{eq:tildelinear} that 
\begin{equation*}
\lip(\Phi) = 0 < \lip(\widetilde{\Phi}).
\end{equation*}
\end{proposition}
{\color{black}\Cref{prop:not_working_deep} does not rule out the possibility that the inequality $\lip(\Phi) \gtrsim \lip(\tilde{\Phi})$ holds with high probability as $d \to \infty$. 
However, the proposition shows that the case of deep networks is qualitatively different from the case of shallow networks, where $\lip(\Phi) \gtrsim \lip(\tilde{\Phi})$ holds 
even deterministically (see \Cref{prop:shallow_low_linear}).
Hence, one must carry out a more detailed analysis of the gradients in the case of deep networks (see \Cref{subsec:deep_lower}).}

\section{Proof of the upper bound} \label{sec:upper}

Referring to \eqref{eq:lowbound}, the goal of this section is to establish upper bounds for
\begin{equation*}
\underset{x \in \RR^d}{\sup}\Vert W^{(L)} \cdot D^{(L-1)}(x) \cdot W^{(L-1)} \cdots D^{(0)}(x)\cdot W^{(0)}\Vert_2
\end{equation*}
that hold with high probability and in expectation, where the randomness is over the random matrices $W^{(0)} , ..., W^{(L)}$ and the random bias vectors $b^{(0)}, ..., b^{(L-1)}$. For this to make sense one first needs to know that 
\begin{equation} \label{eq:suppp}
\underset{x \in \RR^d}{\sup}\Vert W^{(L)} \cdot D^{(L-1)}(x) \cdot W^{(L-1)}\cdots D^{(0)}(x)\cdot W^{(0)}\Vert_2
\end{equation} 
is indeed measurable; we refer to \Cref{app:measurable} for a proof of this fact. 

Throughout the entire section, we assume that \Cref{assum:1} is satisfied. 
Since the random matrices $W^{(\ell)}$ and the random biases $b^{(\ell)}$ are jointly independent it is possible to calculate the expectation iteratively by first assuming that the matrices
 $W^{(0)}, ... ,W^{(L-1)}$ and the biases $b^{(0)}, ..., b^{(L-1)}$ are fixed and deriving an upper bound when only $W^{(L)}$ is assumed to be random. Then as the final step we are also going to allow randomness in $W^{(0)}, ..., W^{(L-1)}$ and $b^{(0)}, ..., b^{(L-1)}$ to get the desired result. In other words, we are conditioning on $W^{(0)}, ..., W^{(L-1)}, b^{(0)}, ..., b^{(L-1)}$.

The central tool for deriving these bounds is \emph{Dudley's inequality} which can be found for example in \cite[Theorems~5.25~and~5.29]{van2014probability}. We refer to \Cref{app:dudley} for details on Dudley's inequality. The key idea of this section is contained in the following proposition. 
\begin{proposition}\label{prop:key}
Let the matrices $W^{(0)}, ..., W^{(L-1)}$ and the biases $b^{(0)},..., b^{(L-1)}$ be fixed and set $\Lambda \defeq \Vert W^{(L-1)} \Vert_2 \cdots \Vert W^{(0)} \Vert_2$. For $x \in \RR^d$ and $z \in \overline{B}_d(0,1)$ we define
\begin{equation*}
Y_{z,x} \defeq D^{(L-1)}(x) W^{(L-1)}\cdots D^{(0)}(x) \cdot W^{(0)}z \in \RR^N
\end{equation*}
and further
\begin{equation*}
\mathcal{L} = \mathcal{L}(d,N,L,W^{(0)}, ..., W^{(L-1)}, b^{(0)},..., b^{(L-1)})\defeq \left\{ Y_{z,x}: \ x \in \RR^d, \ z \in \overline{B}_d(0,1)\right\} \subseteq \RR^N. 
\end{equation*}
Then there exists an absolute constant $C>0$ such that the following holds: Given any $u \geq 0$, we have
\begin{equation*}
\underset{x \in \RR^d}{\sup} \Vert W^{(L)} \cdot D^{(L-1)}(x)\cdot W^{(L-1)} \cdots D^{(0)}(x) \cdot W^{(0)}\Vert_2 \leq  C \cdot \left( \int_0^{\Lambda} \sqrt{\ln \left(\mathcal{N}(\mathcal{L}, \Vert \cdot \Vert_2, \eps)\right)} \ \dd \eps + u \Lambda \right)
\end{equation*}
with probability at least $(1 - 2\exp(-u^2))$ (with respect to the choice of $W^{(L)}$). Moreover, 
\begin{equation*}
\underset{W^{(L)}}{\EE} \left[ \underset{x \in \RR^d}{\sup} \Vert W^{(L)} \cdot D^{(L-1)}(x) \cdot W^{(L-1)}\cdots D^{(0)}(x)\cdot W^{(0)}\Vert_2\right] \leq C \cdot \int_0^{\Lambda}  \sqrt{\ln \left(\mathcal{N}(\mathcal{L}, \Vert \cdot \Vert_2, \eps)\right)} \ \dd \eps.
\end{equation*}
\end{proposition}
\begin{proof}
For $y \in \mathcal{L}$ it holds that there are $x \in \RR^d$ and $z \in \overline{B}_d(0,1)$ satisfying 
\begin{equation*}
y=Y_{z,x} = D^{(L-1)}(x) W^{(L-1)}\cdots D^{(0)}(x) W^{(0)}z. 
\end{equation*}
We compute
\begin{align*}
\Vert y \Vert_2 &\leq \underbrace{\Vert D^{(L-1)}(x)\Vert_2}_{\leq 1} \cdot \Vert W^{(L-1)}\Vert_2 \cdots \underbrace{\Vert D^{(0)}(x) \Vert_2}_{\leq 1} \cdot \Vert W^{(0)} \Vert_2 \cdot \underbrace{\Vert z \Vert_2}_{\leq 1} \\
&\leq \Vert W^{(L-1)} \Vert_2 \cdots \Vert W^{(0)} \Vert_2 = \Lambda.
\end{align*}
Hence, since $0 = D^{(L-1)}(0) W^{(L-1)}\cdots D^{(0)}(0)W^{(0)}0 \in \mathcal{L}$ it follows 
\begin{equation}\label{eq:cov=1}
\mathcal{N}(\mathcal{L}, \Vert \cdot \Vert_2, \eps) = 1 \quad \text{for } \eps \geq  \Lambda.
\end{equation}
To get the final result we rewrite
\begin{align*}
\underset{x \in \RR^d}{\sup} \Vert W^{(L)}D^{(L-1)}(x)\cdot W^{(L-1)}\cdots D^{(0)}(x) W^{(0)} \Vert_2  &=  \underset{x \in \RR^d, z \in \overline{B}_d(0,1)}{\sup} \left\langle \left(W^{(L)}\right)^T, Y_{z,x}\right\rangle  \\
&=   \underset{Y \in \mathcal{L}}{\sup} \left\langle \left(W^{(L)}\right)^T, Y \right\rangle.
\end{align*}
From the observation $\mathcal{L} \subseteq \overline{B}_N(0, \Lambda)$ we infer $\diam(\mathcal{L}) \leq 2\Lambda$. \Cref{prop:dudley} and \eqref{eq:cov=1} then yield the claim by noticing once again that $0 \in \mathcal{L}$.
\end{proof}
Given the above proposition, the problem of bounding the Lipschitz constant of random ReLU networks has been transferred to bounding the covering numbers of the set $\mathcal{L}$. Finding upper bounds for these covering numbers is the essential task in the following two subsections. In fact, in the following we show that
\begin{equation*}
\mathcal{N}(\mathcal{L}, \Vert \cdot \Vert_2, \eps) \leq \left(\frac{9 \Vert W^{(L-1)} \Vert_2 \cdots \Vert W^{(0)} \Vert_2}{\eps} \right)^{C(d+1)} \cdot \left(\frac{\ee N}{d+1}\right)^{L(d+1)},
\end{equation*}
at least if $N > d+2$. Here, $C>0$ is an absolute constant. However, in the case of shallow networks, i.e., $L=1$, it is possible to show that the above inequality holds without the additional factor $\left(\frac{\ee N}{d+1}\right)^{L(d+1)}$, without the assumption $N > d+2$ and $d$ can be replaced by $\min\{N,d\}$, which in the end leads to a sharper bound on the Lipschitz constant. Therefore, the cases of shallow and deep networks are treated separately in \Cref{sec:shallow,sec:deep}, respectively. The final bounds on the covering numbers can be found in \Cref{lem:cov_num_bound,lem:cov_bound}.

\subsection{The shallow case}

\label{sec:shallow}
Firstly, we consider shallow neural networks, i.e., networks that only have a single hidden layer and can hence be written as
\begin{equation*}
\left(x \mapsto (W^{(1)} \cdot x + b^{(1)})\right) \circ \relu \circ \left( x \mapsto (W^{(0)}\cdot x + b^{(0)}) \right).
\end{equation*}
As already explained above, we are from now on going to assume that the matrix $W^{(0)}$ and the vector $b^{(0)}$ are fixed and only assume randomness in $W^{(1)}$ and $b^{(1)}$.
\begin{lemma} \label{lem:scalarproduct_alternative_form}
Let $W^{(0)} \in \RR^{N \times d}$ and $b^{(0)} \in \RR^N$ be fixed. We recall that for $\alpha \in \RR^{d+1}$ we define
\begin{equation*}
f_\alpha: \quad \RR^{d+1} \to \{0,1\}, \quad x \mapsto \mathbbm{1}_{\alpha^Tx > 0}.
\end{equation*}
Furthermore, for any vector $z \in \RR^{d}$, let
\begin{equation*}
\tau_{\alpha, z} \defeq (W^{(0)}z) \odot \left( f_\alpha\left((W^{(0)}_{1,-}, b^{(0)}_1)^T\right) , ..., f_\alpha \left((W^{(0)}_{N,-}, b^{(0)}_N)^T\right)\right) \in \RR^N.
\end{equation*}
Then the following two statements hold:
\begin{enumerate}
\item{$Y_{z,x}=  \tau_{(x,1)^T, z}$ for every $x \in \RR^d$ and $z \in \B_d(0,1)$,}
\item{\label{item:lem_2}$\Vert \tau_{\alpha,z} \Vert_2 \leq \Vert W^{(0)} \Vert_2 \cdot \Vert z \Vert_2$ for all $\alpha \in \RR^{d+1}, z \in \RR^d$.}
\end{enumerate}
Here, $Y_{z,x}$ is as introduced in \Cref{prop:key}.
\end{lemma}

\begin{proof}
\leavevmode
\begin{enumerate}
\item{ Let $x \in \RR^d$ and $z \in \B_d(0,1)$. For every $i \in \{1,...,N\}$ we calculate
\begin{align*}
(D^{(0)}(x) \cdot W^{(0)} \cdot z)_i &= \mathbbm{1}_{W^{(0)}_{i,-} x + b^{(0)}_i > 0} \cdot \left(W^{(0)}z\right)_i \\
&= \left(W^{(0)}z\right)_i \cdot f_{(x,1)^T}\left(\left(W^{(0)}_{i,-}, b^{(0)}_i\right)^T\right) = \left(\tau_{(x,1)^T,z}\right)_i,
\end{align*}
which yields the claim.}
\item{It is immediate that
\begin{equation*}
\Vert \tau_{\alpha, z} \Vert_2 \leq \Vert W^{(0)}z \Vert_2 \cdot \underset{i=1,...,N}{\max} \underbrace{\abs{f_\alpha\left(\left(W^{(0)}_{i,-}, b^{(0)}_i\right)^T\right)}}_{\leq 1} \leq \Vert W^{(0)} \Vert_2 \cdot \Vert z \Vert_2. \qedhere
\end{equation*}
}
\end{enumerate}
\end{proof}

The desired bound for the covering number of $\mathcal{L}$ in the case of shallow networks is contained in the following lemma. 
\begin{lemma} \label{lem:cov_num_bound}
Assume that $W^{(0)} \in \RR^{N \times d}$ and $b^{(0)} \in \RR^N$ are fixed. There exists an absolute constant $C>0$ such that, writing $k \defeq \rang(W^{(0)})$, for every $\varepsilon \in (0, \Vert W^{(0)} \Vert_2)$ it holds
\begin{equation*}
\mathcal{N}(\mathcal{L}, \Vert \cdot \Vert_2, \eps) \leq \left(\frac{9 \Vert W^{(0)} \Vert_2}{\varepsilon}\right)^{C\cdot k}.
\end{equation*}
Here, $\mathcal{L}$ is as introduced in \Cref{prop:key}.
\end{lemma}
\begin{proof}
Without loss of generality, we assume that $W^{(0)} \neq 0$, since otherwise $(0, \Vert W^{(0)} \Vert_2) = \emptyset$. Note that according to \Cref{lem:scalarproduct_alternative_form} we can write
\begin{equation*}
\mathcal{L} = \left\{\tau_{(x,1)^T, v}: \ x \in \RR^d, v \in \B_d(0,1) \right\}.
\end{equation*}
We can even weaken this identity and infer
\begin{equation*}
\mathcal{L} = \left\{\tau_{(x,1)^T, v}: \ x \in \ker(W^{(0)})^{\perp}, v \in \B_d(0,1) \cap \ker(W^{(0)})^{\perp} \right\}.
\end{equation*}
Here, $\ker(W^{(0)})^{\perp}$ denotes the orthogonal complement of $\ker(W^{(0)})$. We further note 
\begin{equation*}
\dim(\ker(W^{(0)})^\perp) = k.
\end{equation*}

Let $\eps \in (0, \Vert W^{(0)} \Vert_2)$. From \Cref{prop:covering_ball} we infer the existence of a natural number $M \in \NN$ with $M \leq \left( \frac{8 \Vert W^{(0)} \Vert_2}{\varepsilon} + 1\right)^k$ and $v_1, ..., v_M \in \B_d(0,1) \cap \ker(W^{(0)})^\perp$ such that
\begin{equation*}
\B_d(0,1) \cap \ker(W^{(0)})^\perp \subseteq \bigcup_{i=1}^M \overline{B}_d\left(v_i,\frac{\varepsilon}{ 4 \Vert W^{(0)} \Vert_2}\right) \cap \ker(W^{(0)})^\perp.
\end{equation*}
For $i \in \{1,...,M\}$ let $w_i \defeq W^{(0)}v_i$. Fix $i \in \{1,...,M\}$ and first assume $w_i \neq 0$. Define a probability measure $\mu_i$ on $\ker(W^{(0)})^{\perp} \times \RR$ by
\begin{equation*}
\mu_i \defeq \frac{1}{\Vert w_i \Vert_2^2} \cdot \sum_{\ell = 1}^N (w_i)_\ell^2 \cdot \delta_{\left(W^{(0)}_{\ell,-}, b^{(0)}_\ell\right)^T}.
\end{equation*}
Here, we note that by definition it holds $\left(W^{(0)}_{\ell,-}\right)^T \in \ker(W^{(0)})^\perp$ for every $\ell \in \{1,...,N\}$.

Let $\mathcal{F} \defeq \left\{ \fres{f_\alpha}{\ker(W^{(0)})^\perp \times \RR} : \ \alpha \in \ker(W^{(0)})^\perp \times \RR\right\}$, where $f_\alpha$ is as introduced in the previous \Cref{lem:scalarproduct_alternative_form}. From \Cref{prop:vc_half_spaces_2} we infer that 
\begin{equation*}
\vc(\mathcal{F})  = k+1. 
\end{equation*}
Further, \Cref{prop:covering_vc} shows for every $\delta \in (0,1)$ that
\begin{equation*}
\mathcal{N}(\mathcal{F}, L^2(\mu_i), \delta) \leq \left( \frac{2}{\delta}\right)^{C'(k+1)}
\end{equation*}
with an absolute constant $C'>0$. Thus, there exists $K_i \in \NN$ with $K_i \leq \left(\frac{8 \Vert W^{(0)} \Vert_2}{\varepsilon}\right)^{C'(k+1)}$ and vectors $\alpha_1^{(i)},..., \alpha_{K_i}^{(i)} \in \ker(W^{(0)})^\perp \times \RR$ such that
\begin{equation*}
\mathcal{F} \subseteq \bigcup_{j=1}^{K_i} \overline{B}_\mathcal{F}^{L^2(\mu_i)} \left(\fres{f_{\alpha_j^{(i)}}}{\ker(W^{(0)})^\perp \times \RR},\frac{\epsilon}{4\Vert W^{(0)} \Vert_2}\right).
\end{equation*}
If $w_i=0$ let $\alpha_j^{(i)} \defeq 0 \in \ker(W^{(0)})^\perp \times \RR$ for $1\leq j \leq K_i \defeq 1$. 

Now, let $v \in \B_d(0,1) \cap \ker(W^{(0)})^\perp$ and $x \in \ker(W^{(0)})^\perp$ be arbitrary. Then there exists $i \in \{1,...,M\}$ such that 
\begin{equation} \label{eq:v_bound}
\Vert v - v_i \Vert_2 \leq \frac{\varepsilon}{4\Vert W^{(0)} \Vert_2}.
\end{equation} 
Let us first consider the case $w_i = W^{(0)}v_i \neq 0$. Then there exists $j \in \{1,..., K_i\}$ such that
\begin{equation*}
\left\Vert \fres{f_{(x,1)^T}}{\ker(W^{(0)})^\perp \times \RR} - \fres{f_{\alpha_j^{(i)}}}{\ker(W^{(0)})^\perp \times \RR}\right\Vert_{L^2(\mu_i)} \leq \frac{\epsilon}{4\Vert W^{(0)} \Vert_2}.
\end{equation*} 
We compute
\begin{align}
\left\Vert \tau_{(x,1)^T,v}- \tau_{\alpha_j^{(i)}, v_i}\right\Vert_2 &\leq \left\Vert \tau_{(x,1)^T,v} - \tau_{(x,1)^T, v_i}\right\Vert_2 + \left\Vert \tau_{(x,1)^T, v_i} - \tau_{\alpha_j^{(i)},v_i}\right\Vert_2 \nonumber\\
&= \left\Vert \tau_{(x,1)^T, v- v_i}\right\Vert_2 + \left\Vert \tau_{(x,1)^T, v_i} - \tau_{\alpha_j^{(i)},v_i}\right\Vert_2 \nonumber\\
\overset{\text{Lemma}~\ref{lem:scalarproduct_alternative_form}~\eqref{item:lem_2}}&{\leq} \Vert W^{(0)} \Vert_2 \cdot \Vert v - v_i \Vert_2 + \left\Vert \tau_{(x,1)^T, v_i} - \tau_{\alpha_j^{(i)},v_i}\right\Vert_2 \nonumber\\
\label{eq:first_bound}
\overset{\eqref{eq:v_bound}}&{\leq} \frac{\varepsilon}{4} + \left\Vert \tau_{(x,1)^T, v_i} - \tau_{\alpha_j^{(i)},v_i}\right\Vert_2.
\end{align}
Finally, we note because of $w_i = W^{(0)}v_i$ and by definition of $\mu_i$ that
\begin{align}
\left\Vert \tau_{(x,1)^T, v_i} - \tau_{\alpha_j^{(i)}, v_i}\right\Vert_2^2 &= \sum_{\ell = 1}^N (W^{(0)}v_i)_\ell^2 \cdot \left( f_{(x,1)^T}\left((W^{(0)}_{\ell, -}, b^{(0)}_\ell)^T\right) - f_{\alpha_j^{(i)}}\left((W^{(0)}_{\ell, -}, b^{(0)}_\ell)^T\right)\right)^2 \nonumber\\
&= \Vert w_i \Vert_2^2 \cdot \left\Vert \fres{f_{(x,1)^T}}{\ker(W^{(0)})^\perp \times \RR} - \fres{f_{\alpha_j^{(i)}}}{\ker(W^{(0)})^\perp \times \RR}\right\Vert^2_{L^2(\mu_i)} \nonumber \\
&\leq \Vert w_i \Vert_2^2 \cdot \left(\frac{\eps}{4\Vert W^{(0)} \Vert_2}\right)^2 
\label{eq:second_bound}\leq \Vert W^{(0)} \Vert_2^2 \cdot \Vert v_i \Vert_2^2 \cdot \left(\frac{\eps}{4\Vert W^{(0)} \Vert_2}\right)^2 \leq \left(\frac{\eps}{4}\right)^2,
\end{align}
and this trivially remains true in the case $w_i = 0$ if we choose $j=1$.

Overall, \eqref{eq:first_bound} and \eqref{eq:second_bound} together imply in any case that
\begin{equation*}
\left\Vert \tau_{(x,1)^T, v} - \tau_{\alpha_j^{(i)},v_i} \right\Vert_2 \leq \frac{\eps}{2}.
\end{equation*}
Hence, the set 
\begin{equation*}
\left\{ \tau_{\alpha_j^{(i)}, v_i}: \ 1 \leq i \leq M, \ 1 \leq j \leq K_i\right\}
\end{equation*}
is an $\frac{\eps}{2}$-net of $\mathcal{L}$ with respect to $\Vert \cdot \Vert_2$. However, this set does not necessarily have to be a subset of $\mathcal{L}$. Yet, using \cite[Exercise 4.2.9]{vershynin_high-dimensional_2018} for the first inequality, we get
\begin{align*}
\mathcal{N}(\mathcal{L}, \Vert \cdot \Vert_2, \eps) &\leq \sum_{i=1}^M K_i \leq \left(\frac{8 \Vert W^{(0)} \Vert_2}{\eps} + 1\right)^k \cdot \left(\frac{8\Vert W^{(0)} \Vert_2}{\eps}\right)^{C'(k+1)} \\
\overset{\eps < \Vert W^{(0)} \Vert_2}&{\leq} \left(\frac{9\Vert W^{(0)} \Vert_2}{\eps}\right)^{C'(k+1) + k} 
\leq \left( \frac{9\Vert W^{(0)} \Vert_2}{\eps}\right)^{(2C'+1)k},
\end{align*}
so the claim follows choosing $C= 2C' + 1$.
\end{proof}

The derived bound for the covering number of $\mathcal{L}$ leads to the following bound when we only assume randomness in $W^{(1)}$. 
\begin{proposition}\label{thm:lower_bound_1}
There exists an absolute constant $C>0$ such that for fixed $W^{(0)} \in \RR^{N \times d}$ and $b^{(0)} \in \RR^N$, writing $k = \rang(W^{(0)})$, the following hold:
\begin{enumerate}
\item{
For any $u \geq 0$, we have
\begin{equation*}
\underset{x \in \RR^d}{\sup} \left\Vert W^{(1)} \cdot D^{(0)}(x) \cdot W^{(0)}\right\Vert_2 \leq C \cdot \Vert W^{(0)} \Vert_2 \cdot (\sqrt{k} + u)
\end{equation*}
with probability at least $1 - 2\exp(-u^2)$ (with respect to the choice of $W^{(1)}$).}
\item{$\displaystyle
\underset{W^{(1)}}{\EE} \left[ \underset{x \in \RR^d}{\sup} \left\Vert W^{(1)} \cdot D^{(0)}(x) \cdot W^{(0)}\right\Vert_2\right] \leq C \cdot \sqrt{k} \cdot \Vert W^{(0)} \Vert_2.$
}
\end{enumerate}
\end{proposition}
\begin{proof}
Without loss of generality we assume $k \geq 1$. We observe
\begin{align*}
 \int_0^{\Vert W^{(0)}\Vert_2} \sqrt{\ln (\mathcal{N}(\mathcal{L}, \Vert \cdot \Vert_2, \eps))} \ \dd\eps \overset{\text{Lemma}~\ref{lem:cov_num_bound}}&{\leq}  \int_0^{\Vert W^{(0)} \Vert_2}\sqrt{C_1 \cdot k} \cdot \sqrt{\ln \left( \frac{9 \Vert W^{(0)} \Vert_2}{\eps}\right)} \ \dd\eps \\
&= \sqrt{C_1} \cdot \sqrt{k} \cdot 9 \Vert W^{(0)} \Vert_2 \cdot \int_0^{\frac{1}{9}} \sqrt{\ln (1/\sigma)} \ \dd\sigma \\
& \leq C_2 \cdot \sqrt{k} \cdot \Vert W^{(0)} \Vert_2.
\end{align*}
Here, $C_1>0$ is the absolute constant from \Cref{lem:cov_num_bound} and $C_2 \defeq 9\cdot\sqrt{C_1} \cdot \int_0^{1/9} \sqrt{\ln (1/\sigma)} \ \dd \sigma$. At the equality, we applied the substitution $\frac{1}{\sigma} = \frac{9 \Vert W^{(0)} \Vert_2}{\eps}$. We combine this estimate with \Cref{prop:key} and get for any $u \geq 0$ that
\begin{align*}
\underset{x \in \RR^d}{\sup} \left\Vert W^{(1)} \cdot D^{(0)}(x) \cdot W^{(0)}\right\Vert_2 &\leq C_3 \cdot \left(C_2 \cdot \sqrt{k} \cdot \Vert W^{(0)} \Vert_2 + u \cdot \Vert W^{(0)} \Vert_2 \right) \\
&\leq C_3 \cdot \max\{1, C_2\} \cdot \Vert W^{(0)} \Vert_2 \cdot (\sqrt{k} + u)
\end{align*}
with probability at least $1 - 2\exp(-u^2)$, as well as
\begin{equation*}
\underset{W^{(1)}}{\EE} \left[ \underset{x \in \RR^d}{\sup} \Vert W^{(1)}D^{(0)}(x) W^{(0)} \Vert_2\right] \leq C_3 \cdot C_2 \cdot \sqrt{k} \cdot \Vert W^{(0)} \Vert_2,
\end{equation*}
where $C_3 > 0$ is the absolute constant from \Cref{prop:key}. Hence, the claim follows by letting $C \defeq C_3\max\{C_2,1\}$.
\end{proof}
Until now, we have conditioned on $W^{(0)}, b^{(0)}$. Reintroducing the randomness with respect to $W^{(0)}, b^{(0)}$ leads to the following statement.
\begin{proposition} \label{thm:pre_main}
There exist absolute constants $C, c_1 > 0$ such that, writing $k \defeq \min\{d,N\}$, the following hold:
\begin{enumerate}
\item{ \label{item1:pre_main}
For any $u,t \geq 0$ it holds
\begin{equation*}
\underset{x \in \RR^d}{\sup} \Vert W^{(1)} \cdot D^{(0)}(x) \cdot W^{(0)}\Vert_2 \leq C\cdot \left(1 + \frac{\sqrt{d} + t}{\sqrt{N}}\right)(\sqrt{k} + u)
\end{equation*}
with probability at least $(1-2\exp(-u^2))_+ \cdot (1-2\exp(-c_1t^2))_+$ with respect to the choice of $W^{(0)}, W^{(1)}, b^{(0)}, b^{(1)}$.
}
\item{ \label{item2:pre_main}
$ \displaystyle
\EE \left[ \underset{x \in \RR^d}{\sup} \left\Vert W^{(1)} \cdot D^{(0)}(x) \cdot W^{(0)}\right\Vert_2\right] \leq {\color{black} C \cdot \sqrt{d}}.
$
}
\end{enumerate}
\end{proposition}
\begin{proof}
We first note that for any matrix $W^{(0)}$ it holds $\rang(W^{(0)}) \leq k$. 

Let us first deal with Part \eqref{item1:pre_main}. Let $C_2 > 0$ be the (absolute) constant from \Cref{thm:lower_bound_1} and $C_3 \defeq \sqrt{2}C_2$. For fixed $u,t \geq 0$ let 
\begin{equation*}
A \defeq \left\{ (W^{(1)}, W^{(0)}, b^{(0)}): \ \underset{x \in \RR^d}{\sup} \Vert W^{(1)} \cdot D^{(0)}(x) \cdot W^{(0)}\Vert_2 \leq C_3 \left(1 + \frac{\sqrt{d} + t}{\sqrt{N}}\right)(\sqrt{k} + u)\right\}.
\end{equation*}
Furthermore let 
\begin{equation*}
A_1 \defeq \left\{ (W^{(0)}, b^{(0)}) : \ \Vert W^{(0)} \Vert_2 \leq \sqrt{2} \left(1 + \frac{\sqrt{d} + t}{\sqrt{N}}\right)\right\}
\end{equation*}
and for fixed $(W^{(0)}, b^{(0)})$, let
\begin{equation*}
A_2\left(W^{(0)}, b^{(0)}\right) \defeq  \left\{ W^{(1)}: \ \underset{x \in \RR^d}{\sup} \Vert W^{(1)} \cdot D^{(0)}(x) \cdot W^{(0)}\Vert_2 \leq C_2 \cdot \Vert W^{(0)}  \Vert_2 \cdot (\sqrt{k} + u)\right\}.
\end{equation*}
Note that then
\begin{equation*}
(W^{(0)}, b^{(0)}) \in A_1, \ W^{(1)} \in A_2\left(W^{(0)}, b^{(0)}\right) \quad \Longrightarrow \quad  (W^{(1)}, W^{(0)}, b^{(0)}) \in A.
\end{equation*}
From \Cref{thm:lower_bound_1} and since probabilities are always non-negative we infer
\begin{equation*}
\PP^{W^{(1)}} \left(A_2\left(W^{(0)}, b^{(0)}\right)\right) \geq (1 - 2\exp(-u^2))_+
\end{equation*}
for any $(W^{(0)}, b^{(0)})$. Furthermore, it holds 
\begin{equation*}
\Vert W^{(0)} \Vert_2 = \sqrt{\frac{2}{N}} \cdot \left\Vert \sqrt{\frac{N}{2}} \cdot W^{(0)} \right\Vert_2 \leq \sqrt{\frac{2}{N}}\cdot (\sqrt{N} + \sqrt{d} + t) = \sqrt{2}\left(1 + \frac{\sqrt{d} + t}{\sqrt{N}}\right)
\end{equation*}
with probability at least $(1 - 2 \exp(-c_1 t^2))_+$ for some absolute constant $c_1 > 0$, as follows from \cite[Corollary 7.3.3]{vershynin_high-dimensional_2018} by noting that the matrix $\sqrt{\frac{N}{2}} W^{(0)}$ has independent $\mathcal{N}(0,1)$-entries. The claim of Part \eqref{item1:pre_main} then follows from \Cref{prop:highprob}.

Let us now deal with Part \eqref{item2:pre_main}. Using \Cref{thm:lower_bound_1} we derive
\begin{equation*}
\underset{W^{(0)},b^{(0)},W^{(1)}}{\EE} \left[ \underset{x \in \RR^d}{\sup} \left\Vert W^{(1)} \cdot D^{(0)}(x) \cdot W^{(0)}\right\Vert_2 \right]\leq C_2 \cdot \sqrt{k} \cdot \underset{W^{(0)}}{\EE} \ \Vert W^{(0)} \Vert_2.
\end{equation*}
From \cite[Theorem 7.3.1]{vershynin_high-dimensional_2018} we get
\begin{equation*}
\underset{W^{(0)}}{\EE} \ \Vert W^{(0)} \Vert_2 = \sqrt{\frac{2}{N}} \cdot \underset{W^{(0)}}{\EE} \   \left\Vert \sqrt{\frac{N}{2}} \cdot W^{(0)} \right\Vert_2 \leq\sqrt{\frac{2}{N}}(\sqrt{N} + \sqrt{d}) =\sqrt{2} \left(1 + \sqrt{\frac{d}{N}}\right).
\end{equation*}
{\color{black}
Hence, we obtain 
\[
\EE \left[ \underset{x \in \RR^d}{\sup} \left\Vert W^{(1)} \cdot D^{(0)}(x) \cdot W^{(0)}\right\Vert_2\right] \leq C_3  \cdot \left(\sqrt{k} + \sqrt{\frac{kd}{N}}\right).
\]
If $k= N$ (so that $N \leq d$), we get 
\[
\sqrt{k} + \sqrt{\frac{kd}{N}} = \sqrt{N} + \sqrt{d} \overset{N \leq d}{\leq} 2 \cdot \sqrt{d}
\]
and if $k=d$ (so that $d \leq N$), we have 
\[
\sqrt{k} + \sqrt{\frac{kd}{N}} = \sqrt{d} + \frac{d}{\sqrt{N}} \overset{\sqrt{d} \leq \sqrt{N}}{\leq} 2 \cdot \sqrt{d}.
\]
Thus, the final claim follows letting $C \defeq 2C_3$.}
\end{proof}

The transfer of \Cref{thm:pre_main} to obtain an upper bound of the Lipschitz constant of shallow ReLU networks follows directly from \eqref{eq:lowbound}.
\begin{theorem} \label{thm:1_main}
There exist absolute constants $C, c_1 > 0$ such that if $\Phi: \RR^d \to \RR$ is a random shallow ReLU network with width $N$ satisfying \Cref{assum:1} and writing $k \defeq \min\{d,N\}$, the following hold:
\begin{enumerate}
\item{For any $u,t \geq 0$, we have
\begin{equation*}
\lip(\Phi) \leq C \cdot \left(1 + \frac{\sqrt{d} + t}{\sqrt{N}}\right)(\sqrt{k} + u)
\end{equation*}
with probability at least $(1-2\exp(-u^2))_+ \cdot (1-2\exp(-c_1t^2))_+$.}
\item{$ \displaystyle
\EE \left[\lip(\Phi) \right]  \leq C \cdot \sqrt{d}.
$}
\end{enumerate}
\end{theorem}

By plugging in certain values for $u$ and $t$, we can now prove \Cref{thm:main_1}.
\renewcommand*{\proofname}{Proof of \Cref{thm:main_1}}
\begin{proof}
Let $\widetilde{C}$ and $\widetilde{c_1}$ be the relabeled constants from \Cref{thm:1_main} and $k \defeq \min\{d,N\}$, as well as $\ell \defeq \max\{d,N\}$. To show Part (1), we set $u = \sqrt{k}$ and $t = \sqrt{\ell}$. Then, since the inequality $\sqrt{\ln(2)} < \sqrt{\ln(\ee)} = 1$ holds, we get $u \geq \sqrt{\ln(2)}$ and thus $1-2\exp(-u^2) \geq 0$. \Cref{thm:1_main} shows 
\begin{equation*}
\lip(\Phi) \leq 2\widetilde{C} \cdot \left(1 + \frac{\sqrt{d} +\sqrt{\ell}}{\sqrt{N}}\right) \sqrt{k} 
\end{equation*}
with probability at least $(1-2\exp(-k)) \cdot (1-2\exp(-\widetilde{c_1} \ell))_+$. If $d \leq N$ we get
\begin{equation*}
2\widetilde{C}\cdot \left(1 + \frac{\sqrt{d} +\sqrt{\ell}}{\sqrt{N}}\right) \sqrt{k} \leq 2\widetilde{C}\cdot  \left(1 + \frac{\sqrt{N} + \sqrt{N}}{\sqrt{N}}\right) \cdot \sqrt{d} = 6\widetilde{C} \cdot \sqrt{d}
\end{equation*}
and if $d \geq N$ we infer
\begin{align*}
2\widetilde{C}\cdot \left(1 + \frac{\sqrt{d} +\sqrt{\ell}}{\sqrt{N}}\right) \sqrt{k} = 2\widetilde{C}\cdot \left(1 + \frac{\sqrt{d} +\sqrt{d}}{\sqrt{N}}\right) \sqrt{N}= 2\widetilde{C}\cdot \left(\sqrt{N} + \sqrt{d} + \sqrt{d}\right) \leq 6\widetilde{C} \cdot \sqrt{d}.
\end{align*}
{\color{black}Part (2) follows directly from \Cref{thm:1_main}(2).} 
\end{proof}
\renewcommand*{\proofname}{Proof}

\subsection{The deep case}
\newcommand{\D}{\mathscr{D}}
\label{sec:deep}
In the following, we treat the case of deep networks, meaning $L > 1$. The proofs of this section also apply in the case of shallow networks, 
but are only relevant in the case of deep networks, 
since in the case of shallow networks better bounds have been derived in the preceding subsection. Again, we first assume that the matrices $W^{(0)}, ..., W^{(L-1)}$ and the biases $b^{(0)}, ..., b^{(L-1)}$ are fixed and the matrix $W^{(L)}$ is initialized randomly according to \Cref{assum:1}. In this setting, we denote
\begin{equation*}
\Lambda \defeq \Vert W^{(L-1)} \Vert_2 \cdots \Vert W^{(0)} \Vert_2.
\end{equation*}

As the first step, we state a bound on the cardinality of the set of all possible combinations $(D^{(L-1)}(x), ..., D^{(0)}(x))$ 
of the $D$-matrices occurring in the formula for $\nabla \Phi(x)$ from \Cref{prop:grad_relu}.
This cardinality is known as \textit{the number of activation patterns} in the literature \cite{hanin2019deep}
and is related but not necessarily identical to the number of linear regions,
which has been studied in, e.g., \cite{zaslavsky1975facing,montufar2014number}.
Naively, one can bound the number of activation patterns by $2^{LN}$.
A much sharper bound that, however, only holds on bounded input sets has been derived in \cite{hanin2019deep}.
Moreover, a bound sharper than $2^{LN}$ that holds on the entire input space
has been derived in \cite[Proposition~3]{montufar2017notes}.
The following inequality is similar to the one presented in \cite{montufar2017notes}.
In order to keep the paper self-contained we decided to include the short proof below.
\begin{lemma}\label{lem:D_card}
Assume $d+2 <N$. For fixed $W^{(0)},..., W^{(L-1)}$ and $b^{(0)}, ..., b^{(L-1)}$ we define
\begin{equation*}
\mathscr{D} \defeq \left\{ \left(D^{(L-1)}(x),..., D^{(0)}(x)\right): \ x \in \RR^d\right\},
\end{equation*}
with $D^{(0)}(x),..., D^{(L-1)}(x)$ as defined in \Cref{subsec:gradient}.
Then it holds that
\begin{equation*}
\vert \mathscr{D} \vert \leq \left(\frac{\ee N}{d+1}\right)^{L(d+1)}.
\end{equation*}
\end{lemma}
\begin{proof}
For any $0\leq \ell \leq L-1$ we define
\begin{equation*}
\D^{(\ell)} \defeq \left\{ \left(D^{(\ell)}(x), ..., D^{(0)}(x)\right): \ x \in \RR^d\right\}
\end{equation*}
and claim that it holds
\begin{equation}\label{eq:ind_d}
\vert \D^{(\ell)} \vert \leq \left(\frac{\ee N}{d+1}\right)^{(\ell + 1)(d+1)},
\end{equation}
which we will show by induction over $\ell$. 

To this end, we start with the case $\ell = 0$. Using the notation introduced in \cref{prop:vc_half_spaces_2}, we see for any $i \in \{1,...,N\}$ and $x \in \RR^d$ that 
\begin{equation*}
\left(D^{(0)}(x)\right)_{i,i} = f_{(x,1)^T}\left((W^{(0)}_{i,-}, b^{(0)}_{i})^T\right).
\end{equation*}
By definition of $D^{(0)}(x)$, this yields
\begin{align*}
\vert\D^{(0)} \vert &= \left\vert \left\{ \left(f_{(x,1)^T}\left((W^{(0)}_{1,-}, b^{(0)}_{1})^T\right),..., f_{(x,1)^T}\left((W^{(0)}_{N,-}, b^{(0)}_{N})^T\right)\right): \ x \in \RR^d\right\}\right\vert \\
&\leq \left\vert \left\{ \left(f_{\alpha}\left((W^{(0)}_{1,-}, b^{(0)}_{1})^T\right),..., f_{\alpha}\left((W^{(0)}_{N,-}, b^{(0)}_{N})^T\right)\right): \ \alpha \in \RR^{d+1}\right\}\right\vert \leq \left(\frac{\ee N}{d+1}\right)^{d+1}.
\end{align*}
To obtain the last inequality we employed Sauer's lemma (cf. \cite[Lemma 6.10]{shalev2014understanding}) and the estimate for the VC-dimension of halfspaces (cf. \Cref{prop:vc_half_spaces_2}) and the assumption $d+2 < N$.

We now assume that the claim holds for some fixed $0 \leq \ell \leq L-2$. We then see 
\begin{align*}
\D^{(\ell + 1)} &= \left\{ \left(D^{(\ell+1)}(x), ..., D^{(0)}(x)\right): \ x \in \RR^d\right\} \\
&= \! \!\biguplus_{(C^{(\ell)}, ..., C^{(0)}) \in \D^{(\ell)}} \! \left\{ \left(D^{(\ell + 1)}(x), C^{(\ell)}, ..., C^{(0)}\right): \! \ x \in \RR^d \!\text{ with } \! D^{(j)}(x)=C^{(j)} \ \!\text{for } \! \text{all } \! j\in \{0,...,\ell\}\right\}
\end{align*}
and hence
\begin{equation} \label{eq:D_bound}
\abs{\D^{(\ell + 1)}} = \sum_{(C^{(\ell)}, ..., C^{(0)}) \in \D^{(\ell)}} \abs{\left\{ D^{(\ell + 1)}(x): \ x \in \RR^d \text{ with } D^{(j)}(x)=C^{(j)} \ \text{for all } j\in \{0,...,\ell\}\right\}}.
\end{equation}
We thus fix $(C^{(\ell)}, ..., C^{(0)}) \in \D^{(\ell)}$ and seek to bound $\abs{\mathcal{A}}$ where
\begin{equation*}
\mathcal{A} \defeq \left\{ D^{(\ell + 1)}(x): \ x \in \RR^d, \  D^{(j)}(x) = C^{(j)} \ \text{for all } j\in \{0,...,\ell\}\right\}.
\end{equation*}
With $x^{(\ell)}$ as in \Cref{subsec:gradient} (for $1 \leq \ell \leq L$), it is immediate that
\begin{align*}
&\norel\abs{\mathcal{A}} \\
&= \abs{\left\{\left(\mathbbm{1}_{W^{(\ell + 1)}_{1,-}x^{(\ell+1)} + b^{(\ell + 1)}_1 > 0}, ..., \mathbbm{1}_{W^{(\ell + 1)}_{N,-}x^{(\ell+1)} + b^{(\ell + 1)}_N > 0}\right):  x \in \RR^d \!\text{ with }  \! \forall \  0\leq j \leq \ell\!:\!D^{(j)}(x)=C^{(j)}\right\}}.
\end{align*}
Fix $x \in \RR^d$ with $D^{(j)}(x)=C^{(j)}$ for every $j \in \{0,..., \ell\}$. From the definition of $x^{(\ell+1)}$ (see \Cref{subsec:gradient}) we may write
\begin{equation*}
x^{(\ell+1)} = C^{(\ell)}W^{(\ell)} \cdots C^{(0)}W^{(0)}x + \bar{c},
\end{equation*} 
where $\bar{c} = \bar{c}(C^{(0)}, ..., C^{(\ell)}, W^{(0)}, ..., W^{(\ell)}, b^{(0)}, ..., b^{(\ell)})\in \RR^N$ is a fixed vector. We thus get for any $i \in \{1,...,N\}$ that
\begin{equation*}
W^{(\ell + 1)}_{i,-}x^{(\ell+1)} + b^{(\ell+1)}_i = W^{(\ell + 1)}_{i,-}C^{(\ell)}W^{(\ell)} \cdots C^{(0)}W^{(0)}x + c_i,
\end{equation*}
where $c = c(C^{(0)}, ..., C^{(\ell)}, W^{(0)}, ..., W^{(\ell+1)}, b^{(0)}, ..., b^{(\ell+1)}) \in \RR^N$ is fixed. Writing 
\begin{equation*}
V \defeq W^{(\ell + 1)}C^{(\ell)}W^{(\ell)} \cdots C^{(0)}W^{(0)} \in \RR^{N \times d} 
\end{equation*}
we infer
\begin{align*}
\abs{\mathcal{A}} &= \left\vert \left\{ \left(f_{(x,1)^T}\left((V_{1,-}, c_{1})^T\right),..., f_{(x,1)^T}\left((V_{N,-}, c_{N})^T\right)\right): \ x \in \RR^d\right\}\right\vert \\
&\leq \left\vert \left\{ \left(f_{\alpha}\left((V_{1,-}, c_{1})^T\right),..., f_{\alpha}\left((V_{N,-}, c_{N})^T\right)\right): \ \alpha \in \RR^{d+1}\right\}\right\vert \leq \left(\frac{\ee N}{d+1}\right)^{d+1}
\end{align*}
where we again used \cite[Lemma 6.10]{shalev2014understanding} and \Cref{prop:vc_half_spaces_2} for the last inequality, again noting that we assumed $d+2 < N$. 

Combining this result with \eqref{eq:D_bound} and the induction hypothesis, we see
\begin{equation*}
\abs{\D^{(\ell+1)}} \leq \abs{\D^{(\ell)}} \cdot \left(\frac{\ee N}{d+1}\right)^{d+1} \leq \left(\frac{\ee N}{d+1}\right)^{(\ell + 2)(d+1)}.
\end{equation*}
By induction this shows that \eqref{eq:ind_d} holds for all $\ell \in \{0,...,L-1\}$. The statement of the lemma then follows by noting that $\D = \D^{(L-1)}$.
\end{proof}

Having established a bound for the cardinality of the set $\mathscr{D}$, we can now bound the covering numbers of $\mathcal{L}$ in the case of deep networks. 
\begin{lemma}\label{lem:cov_bound}
Let $N > d+2$ and $W^{(0)},..., W^{(L-1)}$ and $b^{(0)}, ..., b^{(L-1)}$ be fixed. Moreover, let
\begin{equation*}
\Lambda \defeq \Vert W^{(L-1)} \Vert_2 \cdots \Vert W^{(0)} \Vert_2.
\end{equation*}
Then for every arbitrary $\varepsilon \in (0, \Lambda)$ it holds that
\begin{equation*}
\mathcal{N}(\mathcal{L}, \Vert \cdot \Vert_2, \varepsilon) \leq \left(\frac{3 \Lambda}{\eps} \right)^d \cdot \left(\frac{\ee N}{d+1}\right)^{L(d+1)}.
\end{equation*}
Here, $\mathcal{L}$ is as defined in \Cref{prop:key}. 
\end{lemma}
\begin{proof}
We can assume that $\Lambda > 0$ since otherwise $(0, \Lambda) = \emptyset$. Using the notation $\mathscr{D}$ as introduced previously in \cref{lem:D_card}, we see immediately that
\begin{equation*}
\mathcal{L} = \bigcup_{(C^{(L-1)}, ..., C^{(0)}) \in \D} \left[ C^{(L-1)}W^{(L-1)} \cdots C^{(0)}W^{(0)} \B_d(0,1) \right]
\end{equation*}
and hence it holds for every $\eps > 0$ that
\begin{equation}\label{eq:cov_num}
\mathcal{N}(\mathcal{L}, \Vert \cdot \Vert_2, \eps) \leq \sum_{(C^{(L-1)}, ..., C^{(0)}) \in \D} \mathcal{N}\left(C^{(L-1)}W^{(L-1)} \cdots C^{(0)}W^{(0)} \B_d(0,1), \Vert \cdot \Vert_2, \eps \right).
\end{equation}
This can be seen from the fact that the union of $\eps$-nets of $C^{(L-1)}W^{(L-1)} \cdots C^{(0)}W^{(0)} \B_d(0,1)$, where $(C^{(L-1)}, ..., C^{(0)})$ runs through all elements of $\D$, is an $\eps$-net of $\mathcal{L}$.

Fix $(C^{(L-1)}, ..., C^{(0)}) \in \D$ and define $V \defeq C^{(L-1)}W^{(L-1)} \cdots C^{(0)}W^{(0)}$. 
From \Cref{prop:covering_ball} we infer that there are $w_1, ..., w_M \in \B_d(0,1)$ such that
\begin{equation*}
\B_d(0,1) \subseteq \bigcup_{i=1}^M \overline{B}_d\left(w_i,\frac{\eps}{\Lambda}\right)
\end{equation*}
with $M \leq \left(\frac{2 \Lambda}{\eps} + 1\right)^d$. 
Let $v_i \defeq Vw_i$ for every $i \in \{1,...,M\}$. 
Hence, it holds $v_i \in V\B_d(0,1)$ for every $i \in \{1,...,M\}$. Let $u \in V \B_d(0,1)$ be arbitrary and choose $u' \in \B_d(0,1)$ satisfying $ u = Vu'$. 
By choice of $w_1, ..., w_M$ there exists $i \in \{1,...,M\}$ with 
\begin{equation*}
\Vert u' - w_i \Vert_2 \leq \frac{\eps}{\Lambda}.
\end{equation*}
But then it holds
\begin{align*}
  \Vert u - v_i \Vert_2 =\Vert u - Vw_i \Vert_2 
&\leq \Vert V  \Vert_2 \cdot \Vert u' - w_i \Vert_2 \\
&\leq \underbrace{\Vert C^{(L-1)} \Vert_2}_{\leq 1} \cdot \Vert W^{(L-1)} \Vert_2 \cdots \underbrace{\Vert C^{(0)} \Vert_2}_{\leq 1} \cdot \Vert W^{(0)} \Vert_2 \cdot \frac{\eps}{\Vert W^{(L-1)} \Vert_2 \cdots \Vert W^{(0)} \Vert_2} \\
&\leq \eps.
\end{align*}
Hence, we conclude
\begin{equation*}
\mathcal{N}\left(V\B_d(0,1), \Vert \cdot \Vert_2, \eps \right) \leq \left(\frac{2 \Lambda}{\eps} + 1\right)^d \leq \left(\frac{3 \Lambda}{\eps}\right)^d,
\end{equation*}
since $\eps \leq \Lambda$. Combining this estimate with Equation \eqref{eq:cov_num} and \cref{lem:D_card} yields the claim.
\end{proof}

The following proposition establishes the final bound when the expectation is calculated only with respect to $W^{(L)}$, i.e., conditioning on $W^{(0)}, ..., W^{(L-1)}, b^{(0)}, ..., b^{(L-1)}$. 
\begin{proposition}\label{thm:deep_final}
Let the matrices $W^{(L-1)}, ..., W^{(0)}$ and the biases $b^{(L-1)}, ..., b^{(0)}$ be fixed and define $\Lambda \defeq \Vert W^{(L-1)} \Vert_2 \cdots \Vert W^{(0)} \Vert_2$. Moreover, assume that $d+2 <  N$. Then the following hold, with an absolute constant $C>0$:
\begin{enumerate}
\item{ \label{item:deep_highprob} For any $u \geq 0$, we have
\begin{equation*}
\underset{x \in \RR^d}{\sup} \Vert W^{(L)}D^{(L-1)}(x)W^{(L-1)}\cdots D^{(0)}(x) W^{(0)} \Vert_2 \leq C \cdot \Lambda \cdot \sqrt{L} \cdot \sqrt{\ln \left(\frac{\ee N}{d+1}\right)} \cdot (\sqrt{d} + u)
\end{equation*}
with probability at least $(1 - 2\exp(-u^2))$ (with respect to the choice of $W^{(L)}$).
}
\item{
$\displaystyle
\underset{W^{(L)}}{\EE} \left[\underset{x \in \RR^d}{\sup} \Vert W^{(L)}D^{(L-1)}(x)W^{(L-1)}\cdots D^{(0)}(x) W^{(0)} \Vert_2 \right] \leq C \cdot \Lambda \cdot \sqrt{L} \cdot \sqrt{\ln \left(\frac{\ee N}{d+1}\right)} \cdot \sqrt{d}.
$
}
\end{enumerate}
\end{proposition}
\begin{proof}
Using \Cref{lem:cov_bound} and the elementary inequality $\sqrt{x+y}  \leq \sqrt{x} + \sqrt{y}$ for $x,y \geq 0$ we infer
\begin{align*}
\sqrt{\ln \left(\mathcal{N}(\mathcal{L}, \Vert \cdot \Vert_2, \eps)\right)} &\leq \sqrt{\ln \left(\left(\frac{3 \Lambda}{\eps} \right)^d \cdot \left(\frac{\ee N}{d+1}\right)^{L(d+1)}\right)} \\
&\leq \sqrt{d} \cdot \sqrt{\ln \left(\frac{3 \Lambda}{\eps}\right)} + \sqrt{L}\cdot \sqrt{d+1} \cdot \sqrt{\ln \left(\frac{\ee N}{d+1}\right)}
\end{align*}
for any $\eps \in (0, \Lambda)$ where $\mathcal{L}$ is as in \Cref{prop:key}. This yields
\begin{equation} \label{eq:proof1}
\int_0^\Lambda \sqrt{\ln \left(\mathcal{N}(\mathcal{L}, \Vert \cdot \Vert_2, \eps)\right)} \ \dd \eps \leq \sqrt{d} \cdot \int_0^\Lambda \sqrt{\ln \left(\frac{3 \Lambda}{\eps}\right)} \ \dd \eps + \Lambda \cdot \sqrt{L}\cdot \sqrt{d+1} \cdot \sqrt{\ln \left(\frac{\ee N}{d+1}\right)}.
\end{equation}
Using the substitution $\sigma = \frac{\eps}{3\Lambda}$ we get
\begin{align} \label{eq:proof2}
\int_0^\Lambda \sqrt{\ln \left(\frac{3 \Lambda}{\eps}\right)} \ \dd \eps = 3\Lambda \cdot \int_0^{\frac{1}{3}} \sqrt{\ln \left( 1/\sigma\right)}\ \dd \sigma \leq C_1 \cdot \Lambda
\end{align}
with $C_1 \defeq 3 \cdot \int_0^{1/3} \sqrt{\ln \left( 1/\sigma\right)} \dd \sigma$. Overall, we thus see
\begin{align}
\int_0^\Lambda \sqrt{\ln \left(\mathcal{N}(\mathcal{L}, \Vert \cdot \Vert_2, \eps)\right)} \ \dd \eps \overset{\eqref{eq:proof1}, \eqref{eq:proof2}}&{\leq} \Lambda \cdot \left( C_1  \cdot \sqrt{d} + \sqrt{L}\cdot \sqrt{d+1} \cdot \sqrt{\ln \left(\frac{\ee N}{d+1}\right)}\right) \nonumber\\
\overset{d < d+1 \leq 2d}&{\leq} \sqrt{2} \cdot \max\{1, C_1\}  \cdot \Lambda \cdot \left( \sqrt{d} + \sqrt{L} \cdot \sqrt{d} \cdot \sqrt{\ln \left(\frac{\ee N}{d+1}\right)}\right) \nonumber\\
\label{alig:int}
\overset{L \geq 1, N \geq d+1}&{\leq} \underbrace{2\sqrt{2} \cdot \max\{1, C_1\}}_{=: C_2}  \cdot \Lambda \cdot \sqrt{L} \cdot \sqrt{\ln \left(\frac{\ee N}{d+1}\right)} \cdot \sqrt{d}. 
\end{align}
We can now prove \eqref{item:deep_highprob}. From \Cref{prop:key} we obtain an absolute constant $C_3>0$ such that for any $u \geq 0$ the estimate
\begin{align*}
\underset{x \in \RR^d}{\sup} \Vert W^{(L)} \cdot D^{(L-1)}(x)W^{(L-1)} \cdots D^{(0)}(x)W^{(0)}\Vert_2 &\leq C_3 \left( C_2  \cdot \Lambda \cdot \sqrt{L} \cdot \sqrt{\ln \left(\frac{\ee N}{d+1}\right)} \cdot \sqrt{d} + u \cdot \Lambda \right) \\
&\leq C_3C_2 \cdot \Lambda \cdot \sqrt{L} \cdot \sqrt{\ln \left(\frac{\ee N}{d+1}\right)} \cdot (\sqrt{d} + u)
\end{align*}
holds with probability at least $(1 - 2 \exp(-u^2))$ with respect to $W^{(L)}$. 

For the expectation bound, simply note that 
\begin{align*}
\underset{W^{(L)}}{\EE} \left[ \underset{x \in \RR^d}{\sup} \Vert W^{(L)} \cdot D^{(L-1)}(x)W^{(L-1)} \cdots D^{(0)}(x)W^{(0)}\Vert_2\right] \leq C_3C_2 \cdot \Lambda \cdot \sqrt{L} \cdot \sqrt{\ln \left(\frac{\ee N}{d+1}\right)} \cdot \sqrt{d}
\end{align*}
follows directly from \Cref{prop:key}. Hence, the claim follows letting $C \defeq C_2C_3$.
\end{proof}

Incorporating randomness in $W^{(0)}, ..., W^{(L-1)}$ and $b^{(0)}, ..., b^{(L-1)}$ leads to the following theorem.
\begin{theorem} \label{thm:deep_finall}
There exist absolute constants $C, c_1 > 0$ such that for $N > d+2$, random weight matrices $W^{(0)},...,W^{(L)}$ and random bias vectors $b^{(0)},...,b^{(L)}$ as in \Cref{assum:1} the following hold:
\begin{enumerate}
\item{
For every $u,t \geq 0$ we have 
\begin{align*}
&\norel \underset{x \in \RR^d}{\sup} \Vert W^{(L)}D^{(L-1)}(x)W^{(L-1)}\cdots D^{(0)}(x) W^{(0)} \Vert_2 \\
&\leq C \cdot \left(1 + \frac{\sqrt{d} + t}{\sqrt{N}}\right)\left(2\sqrt{2} + \frac{\sqrt{2} t}{\sqrt{N}}\right)^{L-1} \cdot \sqrt{L} \cdot \sqrt{\ln \left(\frac{\ee N}{d+1}\right)} \cdot (\sqrt{d} + u)
\end{align*}
with probability at least $(1-2\exp(-u^2))_+\left((1-2\exp(-c_1 t^2))_+\right)^L$ with respect to $W^{(0)},..., W^{(L)}$ and $b^{(0)},...,b^{(L)}$.
}
\item{
$\displaystyle
 \norel\EE\left[\underset{x \in \RR^d}{\sup} \Vert W^{(L)}D^{(L-1)}(x)W^{(L-1)}\cdots D^{(0)}(x) W^{(0)} \Vert_2 \right] $ \\
$\displaystyle \leq C \cdot \left( 1 + \frac{\sqrt{d}}{\sqrt{N}}\right) \cdot (2 \sqrt{2})^{L-1} \cdot \sqrt{L} \cdot \sqrt{\ln \left(\frac{\ee N}{d+1}\right)} \cdot \sqrt{d}.$
}
\end{enumerate}
\end{theorem}
\begin{proof}
Let $\widetilde{C}$ be the relabeled constant from \Cref{thm:deep_final} and define $C \defeq \sqrt{2} \widetilde{C}$. We start with (1): In view of \Cref{thm:deep_final,prop:highprob} it suffices to show that for every $t \geq 0$ the estimate
\begin{equation*}
\Lambda \leq \sqrt{2}\left(1 + \frac{\sqrt{d} + t}{\sqrt{N}}\right)\left(2\sqrt{2} + \frac{\sqrt{2}t}{\sqrt{N}}\right)^{L-1}
\end{equation*}
holds with probability at least $\left((1-2\exp(-c_1 t^2))_+\right)^L$, where
\begin{equation*}
\Lambda \defeq \Vert W^{(L-1)}\Vert_2 \cdots \Vert W^{(0)} \Vert_2.
\end{equation*}
To show this, note that \cite[Corollary 7.3.3]{vershynin_high-dimensional_2018} yields
\begin{equation*}
\Vert W^{(0)} \Vert_2 = \sqrt{\frac{2}{N}} \left\Vert \sqrt{\frac{N}{2}} W^{(0)}\right\Vert_2\leq \sqrt{2} \left(1 + \frac{\sqrt{d} + t}{\sqrt{N}}\right)
\end{equation*}
with probability at least $(1- 2\exp(-c_1 t^2))_+$, as well as
\begin{equation*}
\Vert W^{(\ell)} \Vert_2 = \sqrt{\frac{2}{N}} \left\Vert \sqrt{\frac{N}{2}} W^{(\ell)} \right\Vert_2 \leq \sqrt{\frac{2}{N}} (2 \sqrt{N} + t)= 2\sqrt{2} + \frac{\sqrt{2}t}{\sqrt{N}}
\end{equation*}
with probability at least $(1- 2\exp(-c_1 t^2))_+$ for any $1 \leq \ell \leq L-1$. Hence, the claim follows from an iterative application of \Cref{prop:highprob}.

To conclude (2) we apply \cite[Theorem 7.3.1]{vershynin_high-dimensional_2018} to the matrices $\sqrt{\frac{N}{2}} W^{(\ell)}$ for $0 \leq \ell \leq L-1$. This yields $\underset{W^{(\ell)}}{\EE} \Vert W^{(\ell)} \Vert_2 \leq 2 \sqrt{2}$ for every $1 \leq \ell \leq L-1$ and further $\underset{W^{(0)}}{\EE} \Vert W^{(0)}\Vert_2\leq \sqrt{2} \left(1 + \frac{\sqrt{d}}{\sqrt{N}} \right)$. The independence of the matrices $W^{(\ell)}$ combined with \Cref{thm:deep_final} then yields the claim.
\end{proof}
{\color{black} With a simple modification of the proof of \Cref{thm:deep_finall} one can obtain the same high-probability bound even without assuming that the matrices 
$W^{(0)},...,W^{(L-1)}$ and the biases $b^{(0)},...,b^{(L-1)}$ are jointly independent, at the cost of changing the quantity $\left((1-2\exp(-c_1 t^2))_+\right)^L$ to 
$1-2L\exp(-c_1 t^2)$. 
Note that one still needs to assume that $W^{(L)}$ and $b^{(L)}$ are independent of $W^{(0)},...,W^{(L-1)}$ and $b^{(0)},...,b^{(L-1)}$.}

Now we obtain an upper bound on the Lipschitz constant directly using \Cref{thm:deep_finall} and \eqref{eq:lowbound}:
\begin{theorem}\label{thm:final_deep_lipschitz}
Let $\Phi: \RR^d \to \RR$ be a random ReLU network of width $N$ and with $L$ hidden layers satisfying \Cref{assum:1}. Moreover, let $d+2<N$. Then the following hold, for certain absolute constants $C, c_1 > 0:$
\begin{enumerate}
\item{For every $u,t \geq 0$, we have
\begin{align*}
\lip(\Phi)\leq C \cdot \left(1 + \frac{\sqrt{d} + t}{\sqrt{N}}\right)\left(2\sqrt{2} + \frac{\sqrt{2} t}{\sqrt{N}}\right)^{L-1} \cdot \sqrt{L} \cdot \sqrt{\ln \left(\frac{\ee N}{d+1}\right)} \cdot (\sqrt{d} + u)
\end{align*}
with probability at least $(1-2\exp(-u^2))_+\left((1-2\exp(-c_1 t^2))_+\right)^L$.
}
\item{
$\displaystyle
\EE \left[\lip(\Phi) \right]
\leq C \cdot \left( 1 + \frac{\sqrt{d}}{\sqrt{N}}\right) \cdot (2 \sqrt{2})^{L-1} \cdot \sqrt{L} \cdot \sqrt{\ln \left(\frac{\ee N}{d+1}\right)} \cdot \sqrt{d}.
$
}
\end{enumerate}
\end{theorem}
 By plugging in special values for $t$ and $u$ and using $d \leq N$ we can now prove \Cref{thm:main_2}.

\renewcommand*{\proofname}{Proof of \Cref{thm:main_2}}
\begin{proof}
Let $\widetilde{C}$ and $\widetilde{c_1}$ be the relabeled constants from \Cref{thm:final_deep_lipschitz} and let $C \defeq 6\widetilde{C}$ and $c_1 \defeq \widetilde{c_1}$. 
Part (1) follows from \Cref{thm:final_deep_lipschitz} by plugging in $u = \sqrt{d}$ and $t = \sqrt{N}$, which yields
\begin{equation*}
\lip(\Phi)\leq \widetilde{C} \cdot 3 \cdot (3\sqrt{2})^{L-1} \cdot \sqrt{L} \cdot \sqrt{\ln \left(\frac{\ee N}{d+1}\right)} \cdot 2\cdot \sqrt{d} = 6\widetilde{C} \cdot (3\sqrt{2})^{L-1} \cdot \sqrt{L}\cdot \sqrt{\ln \left(\frac{\ee N}{d+1}\right)} \cdot \sqrt{d}
\end{equation*}
with probability at least $(1-2\exp(-d))\left((1-2\exp(-c_1 N))_+\right)^L$,
where we also used $d \leq N$ and $1-2\exp(-u^2) = 1-2\exp(-d) \geq 0$.
Part (2) follows immediately from $d \leq N$ and part (2) of \Cref{thm:final_deep_lipschitz}.
\end{proof}
\renewcommand*{\proofname}{Proof}

\newcommand{\gell}{\mathcal{G}^{(\ell)}}
\newcommand{\hell}{\mathcal{H}^{(\ell)}}
\newcommand{\gelll}{\mathcal{G}^{(\ell+1)}}
\newcommand{\G}{\mathcal{G}'}

\section{Proof of the lower bound} \label{sec:low_bound}
In this section, we establish \emph{lower} bounds on the Lipschitz constant of randomly initialized ReLU networks. 
The strategy for deriving these lower bounds for shallow networks differs significantly from the approach for deep networks. 
The lower bound in the case of shallow networks (see \Cref{sec:low_bound_shallow}) follows from the fact that the 
Lipschitz constant of a shallow ReLU network can be lower bounded by the Lipschitz constant of the corresponding \emph{linear} network (see \Cref{prop:shallow_low_linear}), combined with
 concentration properties of Gaussian matrices and vectors. 
In the case of deep networks (see \Cref{subsec:deep_lower}), we follow the approach that was already described in \eqref{eq:upbound}: 
Our strategy is to fix a point $x^{(0)} \in \RR^d \setminus \{0\}$ and derive lower bounds for the expression
\begin{equation*}
\Vert W^{(L)} D^{(L-1)}(x_0)W^{(L-1)} \cdots D^{(0)}(x_0) W^{(0)} \Vert_2.
\end{equation*}

\subsection{The shallow case}\label{sec:low_bound_shallow}
In this subsection, we deal with the case of shallow networks, i.e., $L=1$. 
We make heavy use of \Cref{prop:shallow_low_linear}, which states that we only have to consider the corresponding \emph{linear} network in that case, 
i.e., the network that arises from a ReLU network by omitting the ReLU activation. 
This reduces the problem to bounding the norm of the product of a Gaussian matrix with a Gaussian vector (from below). 
\begin{theorem}\label{thm:shallow_low_bound_2}
There exists a constant $c>0$ with the following property: If $\Phi:\RR^d \to \RR$ is a random shallow ReLU network with width $N$ satisfying \Cref{assum:1}, then for every $t,u \geq 0$ it holds
\begin{equation*}
\lip(\Phi) \geq \frac{1}{\sqrt{2}} \cdot \left(1 - \frac{u}{\sqrt{N}}\right)_+  \cdot (\sqrt{d} - t)_+
\end{equation*}
with probability at least $(1-2\exp(-ct^2))_+(1-2\exp(-cu^2))_+$. Recall that we write $a_+ = \max\{0,a\}$ for any number $a \in \RR$.
\end{theorem}
\begin{proof}
From \Cref{prop:shallow_low_linear} we infer that it holds
\begin{equation*}
\lip(\Phi) \geq \frac{1}{2}\lip(\widetilde{\Phi}) = \frac{1}{2} \cdot \Vert W^{(1)}\cdot W^{(0)} \Vert_2.
\end{equation*}
Therefore, in the following we only consider the expression $\Vert W^{(1)}\cdot W^{(0)} \Vert_2$.

To this end, we introduce the notations $U^{(0)} \defeq (W^{(0)})^T \in \RR^{d \times N}$ and $U^{(1)} \defeq (W^{(1)})^T \in \RR^{N}$. 
Moreover, we fix $t,u \geq 0$. Consider the set
\begin{equation*}
A_1 \defeq \{W^{(1)}: \ \Vert W^{(1)} \Vert_2 = \Vert U^{(1)} \Vert_2 \geq (\sqrt{N} - u)_+\}.	
\end{equation*} 
Since the entries of $W^{(1)}$ are $\mathcal{N}(0,1)$-distributed, and since the norm of a Gaussian random vector concentrates around the square root of its size (see, e.g., \cite[Theorem~3.1.1~\&~Equation~(2.14)]{vershynin_high-dimensional_2018}, we infer
\begin{equation*}
\PP^{W^{(1)}}(A_1) \geq (1 - 2\exp(-cu^2))_+
\end{equation*}
with a suitably chosen constant $c>0$. We now \emph{fix} the vector $W^{(1)}$ and consider the set
\begin{equation*}
A_2 (W^{(1)}) \defeq \left\{W^{(0)}: \ \Vert U^{(0)} U^{(1)}\Vert_2 \geq \frac{\sqrt{2} \Vert U^{(1)} \Vert_2}{\sqrt{N}} \cdot (\sqrt{d}-t)_+\right\}.
\end{equation*}
Firstly, assume that $W^{(1)} \neq 0$. Note then that $\frac{\sqrt{N}}{\sqrt{2} \Vert U^{(1)} \Vert_2} \cdot U^{(0)} U^{(1)} \sim \mathcal{N}(0, I_d)$, as follows from the independence of the rows of $U^{(0)}$ and \cite[Exercise~3.3.3~(a)]{vershynin_high-dimensional_2018}. 
Therefore, using again \cite[Theorem~3.1.1~\&~Equation~(2.14)]{vershynin_high-dimensional_2018} we get
\begin{align*}
&\norel \quad\frac{\sqrt{N}}{\sqrt{2} \Vert U^{(1)} \Vert_2} \cdot \Vert U^{(0)} U^{(1)} \Vert_2 \geq (\sqrt{d} - t)_+ \\
&\Leftrightarrow \quad  \Vert U^{(0)} U^{(1)} \Vert_2 \geq \frac{\sqrt{2} \Vert U^{(1)} \Vert_2}{\sqrt{N}} \cdot (\sqrt{d}-t)_+
\end{align*}
with probability at least $1- 2\exp(-ct^2)$ and the last inequality remains true in the case $U^{(1)}  = 0$. 
Hence, we see
\begin{equation*}
\PP^{W^{(0)}} (A_2(W^{(1)})) \geq (1 - 2\exp(-ct^2))_+.
\end{equation*}
For any tuple $(W^{(0)}, W^{(1)})$ with $W^{(1)} \in A_1$ and $W^{(0)} \in A_2(W^{(1)})$ we get
\begin{align*}
\Vert U^{(0)} U^{(1)}\Vert_2 \geq \frac{\sqrt{2} \Vert U^{(1)} \Vert_2}{\sqrt{N}} \cdot (\sqrt{d}-t)_+ \geq \frac{\sqrt{2} (\sqrt{N}-u)_+}{\sqrt{N}} \cdot (\sqrt{d}-t)_+ = \sqrt{2} \cdot \left(1 - \frac{u}{\sqrt{N}}\right)_+ \cdot (\sqrt{d} - t)_+.
\end{align*}
Therefore, \Cref{prop:highprob} yields the claim.
\end{proof}
Again, we plug in special values for $u$ and $t$ to derive the main result.

\renewcommand*{\proofname}{Proof of \Cref{thm:main_shallow_lower}}
\begin{proof}
Let $c_2>0$ be the constant appearing in \Cref{thm:shallow_low_bound_2}. We then pick $u = \frac{\sqrt{N}}{2}$ and $t = \frac{\sqrt{d}}{2}$ and directly get
\begin{equation*}
\lip(\Phi) \geq \frac{1}{\sqrt{2}} \cdot \frac{1}{2} \cdot \frac{1}{2} \cdot \sqrt{d} = \frac{1}{4\sqrt{2}} \cdot \sqrt{d}
\end{equation*}
with probability at least $(1-2\exp(-c_2 N/4))_+(1-2\exp(-c_2 d/4))_+$. 
Hence, the first claim follows picking $c \defeq c_2/4$.

For the expectation bound, we assume $d,N > \frac{\ln(2)}{c}$ and use Markov's inequality to get
\begin{align*}
\EE[\lip(\Phi)] &\geq \PP\left(\lip(\Phi) \geq \frac{1}{4\sqrt{2}} \cdot \sqrt{d}\right) \cdot \frac{1}{4\sqrt{2}} \cdot \sqrt{d} \\
&\geq (1- 2 \exp(-cN))(1-2\exp(-cd)) \cdot \frac{1}{4\sqrt{2}} \cdot \sqrt{d} \\
&\geq \underbrace{\big(1- 2 \exp(-c\cdot\left(\lfloor \ln(2) / c\rfloor + 1\right))\big)\big(1-2\exp(-c\cdot\left(\lfloor \ln(2) / c\rfloor + 1\right))\big) \cdot \frac{1}{4\sqrt{2}}}_{=: c_1} \cdot \sqrt{d} \\
&= c_1 \cdot \sqrt{d}. \qedhere
\end{align*}
\end{proof}
\renewcommand*{\proofname}{Proof}

\subsection{The deep case} \label{subsec:deep_lower}
In this subsection, we deal with the case of deep networks, i.e., $L \geq 2$. 
We note that the condition $L \geq 2$ is not needed in order for our proofs to work and in particular that our proofs also work in the case of shallow networks. 
However, for what follows we will need an additional assumption on the distribution of the biases (see \Cref{assum:2}) and the additional condition $N \gtrsim dL^2$ to prove our final lower bound on the Lipschitz constant,
whereas these assumptions are \emph{not} needed in the case of shallow networks. 
That is why we presented a different proof for shallow networks in \Cref{sec:low_bound_shallow}.

The basic observation is that, if we fix a point $x_0 \in \RR^d \setminus \{0\}$, the weights $W^{(0)}, ..., W^{(\ell - 1)}$ and biases $b^{(0)}, ..., b^{(\ell -1 )}$ and assume that the output of the $\ell$-th layer is non-zero (i.e., $x^{(\ell)} \neq 0$ with $x^{(\ell)}$ as in \Cref{eq:d-matrices}), 
then the matrix $\sqrt{N} D^{(\ell)}(x_0) W^{(\ell)}$ has \emph{isotropic, independent and sub-gaussian rows} (with respect to the randomness in $W^{(\ell)}$ and $b^{(\ell)}$), 
which is shown in \Cref{thm: isotropic_rows,thm: dev_conditions}. 
Here, a random vector $X \in \RR^k$ is called isotropic iff
\begin{equation*}
\EE[X X^T] = I_{k \times k}
\end{equation*}
with $I_{k \times k}$ denoting the $k$-dimensional identity matrix. 

Afterwards, using the \emph{matrix deviation inequality} (see \cite[Theorem~3]{Liaw2017}) 
we show that the product $D^{(L-1)}(x_0)W^{(L-1)} \cdots D^{(0)}(x_0) W^{(0)}$ is almost isometric 
with high probability which then implies the claim. 

We start by showing that for some fixed $x_0 \in \RR^d \setminus \{0\}$ the matrices 
\begin{equation*}
\sqrt{N} D^{(\ell)}(x_0) W^{(\ell)}
\end{equation*}
have isotropic, independent and sub-gaussian rows 
when conditioning on the previous weights and biases $W^{(0)}, ..., W^{(\ell - 1)}, b^{(0)}, ..., b^{(\ell - 1)}$ and assuming $x^{(\ell)} \neq 0$. 
Since $W^{(0)} \in \RR^{N \times d}$ whereas in contrast $W^{(\ell)} \in \RR^{N \times N}$ for $1 \leq \ell \leq L-1$,
we are going to prove the result generally for matrices $W \in \RR^{N \times k}$. 
Let us therefore first introduce the basic assumptions of what follows.
\begin{assumption} \label{assum_1}
Let $W \in \RR^{N \times k}$ be a random matrix and $b \in \RR^N$ a random vector with 
\begin{equation*}
    W_{i,j} \sim \mathcal{N}\left(0, 2/N\right), \quad b_i \sim \mathcal{D}_i, \quad \text{for } 1 \leq i \leq N \quad \text{and} \quad 1 \leq j \leq k,
\end{equation*}
where each $\mathcal{D}_i$ is a symmetric probability distribution on $\RR$. Furthermore, we assume that all the entries of $W$ and $b$ are jointly independent. 
\end{assumption}
First, we show that the rows of the matrix are indeed isotropic, which is done in the following lemma. 
\label{isotropic}
\begin{lemma} \label{thm: isotropic_rows}
Let \Cref{assum_1} be satisfied and fix any vector $x \in \RR^k \setminus \{0\}$. Then each row of
\begin{equation*}
    \sqrt{N} \cdot \diag(Wx + b) \cdot W \in \RR^{N \times k}
\end{equation*}
is an isotropic random vector. Here, $\Delta$ is defined as in \Cref{subsec:gradient}.
\end{lemma}
\begin{proof}
We first consider the case $x = (\alpha,0,...,0)^T \in \RR^k$ with $\alpha \in \RR \setminus \{0\}$. Let $i \in \{1,...,N\}$ and define $V \defeq \sqrt{N} \cdot \diag(Wx + b) \cdot W$. It is well-known that $V_{i,-}$ is an isotropic random vector if and only if
\begin{equation*}
    \EE \langle \left(V_{i,-}\right)^T, y \rangle^2 = \Vert y \Vert_2^2
\end{equation*}
for every $y \in \RR^k$; see \cite[Lemma 3.2.3]{vershynin_high-dimensional_2018}. Therefore, take any arbitrary vector $y \in \RR^k$. A direct calculation yields
\begin{align}
    \EE \left[\langle \left(V_{i,-}\right)^T,y\rangle^2\right] &= \EE  \left[\left(\sum_{\ell=1}^k V_{i,\ell} \hspace{0.05cm} y_\ell\right)^2 \right] = \EE \left[\sum_{j, \ell = 1}^k V_{i,\ell} \hspace{0.05cm} y_\ell V_{i, j}\hspace{0.05cm} y_j \right] = \sum_{j, \ell = 1}^k y_j y_\ell \EE \left[V_{i,j} V_{i,\ell}\right] \nonumber\\
    &= N \cdot \sum_{j, \ell = 1}^k y_j y_\ell \EE \left(\mathbbm{1}_{(Wx + b)_i > 0} \cdot W_{i,j} \cdot W_{i, \ell}\right) \nonumber\\
    \label{eq: first_computation}
    &= N \cdot \sum_{j, \ell = 1}^k y_j y_\ell \EE \left(\mathbbm{1}_{\alpha W_{i,1} + b_i > 0} \cdot W_{i,j} \cdot W_{i, \ell}\right).
\end{align}
If $j \neq \ell$ and w.l.o.g. $j \neq 1$ (since $j \neq 1$ or $\ell \neq 1$, because otherwise $j = 1 = \ell$), it follows by independence that
\begin{equation} \label{eq: second_computation}
    \EE \left(\mathbbm{1}_{\alpha W_{i,1} + b_i> 0} \cdot W_{i,j} \cdot W_{i, \ell}\right) = \underbrace{\EE \left(W_{i,j}\right)}_{=0} \cdot \EE \left(\mathbbm{1}_{\alpha W_{i,1} +b_i> 0} \cdot W_{i, \ell}\right) = 0.
\end{equation}
If $j = \ell \neq 1$, we see
\begin{equation} \label{eq: third_computation}
    \EE \left(\mathbbm{1}_{\alpha W_{i,1} + b_i> 0} \cdot W_{i,j} \cdot W_{i, \ell}\right) = \EE \left(\mathbbm{1}_{\alpha W_{i,1} + b_i> 0}\right) \cdot \EE \left(W_{i,j}^2\right) = \frac{1}{2} \cdot \frac{2}{N} = \frac{1}{N},
\end{equation}
using that $\alpha W_{i,1} + b_i$ has a symmetric and continuous probability distribution.
Here, the continuity follows from $\alpha \neq 0$ and the fact that the random variables $W_{i,1}$ and $b_i$ are independent with $W_{i,1}$ having an absolutely continuous distribution (see, e.g., \cite[Proposition~9.1.6]{dudley2002real}).
If $j = \ell = 1$, we have
\begin{equation*} 
    \EE \left(\mathbbm{1}_{\alpha W_{i,1} + b_i> 0} \cdot W_{i,j} \cdot W_{i, \ell}\right) = \EE \left(\mathbbm{1}_{\alpha W_{i,1} + b_i > 0} \cdot W_{i,1}^2\right).
\end{equation*}
For simplicity, we write $X = W_{i,1}$ and $Y = b_i$ and note since $(X,Y) \sim (-X, -Y)$ (since $X$ and $Y$ are independent and both symmetrically distributed) that
\begin{equation*}
    \EE \left(\mathbbm{1}_{\alpha X+Y > 0} \cdot X^2\right) = \EE \left(\mathbbm{1}_{-\alpha X-Y > 0} \cdot (-X)^2\right) = \EE \left(\mathbbm{1}_{\alpha X+Y < 0} \cdot X^2\right).
\end{equation*}
This yields
\begin{align}
    \EE \left(\mathbbm{1}_{\alpha X+Y > 0} \cdot X^2\right) &= \frac{1}{2} \left(\EE \left(\mathbbm{1}_{\alpha X+Y > 0} \cdot X^2\right) + \EE \left(\mathbbm{1}_{\alpha X+Y < 0} \cdot X^2\right)\right) \nonumber\\
    \label{eq: fourth_computation}
    &= \frac{1}{2}\EE \left(\left(\mathbbm{1}_{\alpha X+Y > 0}+\mathbbm{1}_{\alpha X+Y < 0}\right) \cdot X^2\right) = \frac{1}{2}\EE (X^2) = 1/N.
\end{align}
Here, we used that $\EE (\mathbbm{1}_{\alpha X + Y = 0} \cdot X^2) = 0$, since $\alpha X$ has an absolutely continuous distribution (note $\alpha \neq 0$) and $X$ and $Y$ are independent so that $\PP (\alpha X + Y = 0) =0$. 
Inserting \eqref{eq: second_computation}, \eqref{eq: third_computation} and \eqref{eq: fourth_computation} into \eqref{eq: first_computation} yields 
\begin{equation*}
    \EE \langle \left(V_{i,-}\right)^T, y\rangle^2 = \left\Vert y \right\Vert_2^2,
\end{equation*}
hence the rows of $V$ are isotropic in the case $x = (\alpha,0,...,0)^T \in \RR^k$ with $\alpha \in \RR \setminus \{0\}$.

Now consider an arbitrary vector $x \in \RR^k \setminus \{0\}$. 
Taking an orthogonal deterministic matrix $U \in \RR^{k \times k}$ with $x =  U\alpha e_1$ for some $\alpha \in \RR \setminus \{0\}$ yields
\begin{equation*}
    \diag(Wx + b) \cdot W = \diag\left(WU\alpha e_1 + b\right) \cdot WU \cdot U^{-1}.
\end{equation*}
For every $j \in \{1,...,N\}$ we see
\begin{align*}
\left((WU)^T\right)_{-,j} = \left(U^T W^T\right)_{-,j} = U^T \left(W^T \right)_{-,j } \sim \mathcal{N}\left(0, \frac{2}{N}I_k \right),
\end{align*}
where in the last step we used the rotation invariance of the Gaussian distribution \cite[Proposition 3.3.2]{vershynin_high-dimensional_2018}. This yields
\begin{equation*}
(WU)_{j,-} \sim \mathcal{N}\left(0, \frac{2}{N}I_k\right)
\end{equation*}
for every $j \in \{1,...,N\}$. Since $(WU)_{j,-}$ only depends on $W_{j,-}$, this implies
\begin{equation*}
(WU)_{i,j} \iid \mathcal{N}(0, 2/N) \quad \text{for all } i \in \{1,...,N\}, j \in \{1,...,k\}.
\end{equation*}
Hence, the matrix $V \defeq \diag\left((WU) \alpha e_1 + b \right) \cdot (WU)$ has isotropic rows by what has just been shown (the case $x = (\alpha, 0, ..., 0)^T$).
But then $V \cdot U^{-1}$ has isotropic rows too. Indeed, let $i \in \{1,...,N\}$. Then we have $\left(\left(V \cdot U^{-1}\right)_{i,-}\right)^T = U \cdot \left(V_{i,-}\right)^T$ where $\left(V_{i,-}\right)^T$ is an isotropic random vector. Writing $Z \defeq \left(V_{i,-}\right)^T$, we see for any vector $ y \in \RR^k$ that
\begin{equation*}
    \EE \langle UZ , y \rangle^2= \EE \langle Z, U^T y \rangle^2 = \left\Vert U^T y \right\Vert_2^2 = \left\Vert y \right\Vert_2 ^2,
\end{equation*}
so $UZ$ is isotropic as was to be shown.
\end{proof}
Our next lemma collects all of the properties of the matrices $\sqrt{N}D^{(\ell)}(x_0) W^{(\ell)}$ that we will need. 
\begin{lemma} \label{thm: dev_conditions}
    Let \Cref{assum_1} be satisfied and fix $x \in \RR^k \setminus \{0\}$. Let $\Delta(v), \ v\in \RR^N$, be as defined in \Cref{subsec:gradient}. Then the rows of $\sqrt{N} \diag(Wx + b) \cdot W$ are jointly independent, isotropic random vectors, and sub-gaussian with
    \begin{equation*}
        \left\Vert \left(\sqrt{N} \cdot \diag(Wx + b) \cdot W\right)_{i,-}\right\Vert_{\psi_2} \leq C \quad \text{for all } i \in \{1,...,N\},
    \end{equation*}
    where $C>0$ is an absolute constant. 
\end{lemma}
\begin{proof}
    Note that
    \begin{equation*}
        \left(\sqrt{N} \diag(Wx + b) \cdot W\right)_{i,-} = \sqrt{N} \cdot \mathbbm{1}_{(Wx+b)_i > 0} \cdot W_{i,-},
    \end{equation*}
    which only depends on $b_i$ and the $i$-th row of $W$, which implies that the rows of this matrix are jointly independent. 
    For every vector $y \in \RR^k$ and every $i \in \{1,...,N\}$ we see
    \begin{align*}
    	 \abs{\left(\sqrt{N} \diag(Wx + b) \cdot W\right)_{i,-} \cdot y } &= \abs{\sum_{j=1}^k \left(\diag(Wx + b) \cdot \sqrt{N}W\right)_{i,j}y_j} = \abs{\sum_{j=1}^k \mathbbm{1}_{(Wx + b)_i > 0} \cdot \sqrt{N}W_{i,j}y_j} \\
    	 &= \mathbbm{1}_{(Wx + b)_i > 0} \cdot \abs{\sum_{j = 1}^k \sqrt{N}W_{i,j}y_j} \leq \abs{\sum_{j = 1}^k \sqrt{N}W_{i,j}y_j} .
    \end{align*}
    Note that the $W_{i,j}y_j$ ($j \in \{1,...,k\}$) are independent with $\sqrt{N}W_{i,j}y_j \sim \mathcal{N}(0, 2 y_j^2)$ and hence
    \begin{equation*}
    	\sum_{j = 1}^k \sqrt{N}W_{i,j}y_j \sim \mathcal{N}(0, 2 \cdot \Vert y \Vert_2^2).
    \end{equation*}
    Hence, by definition (see \cite[Section~2.5.2~and~Definition~3.4.1]{vershynin_high-dimensional_2018}) it follows that the random variable $\left(\sqrt{N} \diag(Wx + b) \cdot W\right)_{i,-} \cdot y$ is sub-gaussian and since $y \in \RR^k$ has been chosen arbitrarily we deduce that $\left(\sqrt{N} \diag(Wx + b) \cdot W\right)_{i,-}$ is sub-gaussian with
    \begin{equation*}
    \left\Vert \left(\sqrt{N} \diag(Wx + b) \cdot W\right)_{i,-}\right\Vert_{\psi_2} \leq \underset{y \in \mathbb{S}^{k-1} }{\sup} \left\Vert \sum_{j = 1}^k \sqrt{N}W_{i,j}y_j\right\Vert_{\psi_2} \leq \sqrt{2} \cdot C_1 =: C,
    \end{equation*}
    where $C_1$ is an absolute constant according to \cite[Example 2.5.8 (i)]{vershynin_high-dimensional_2018}.
    
    The isotropy has already been shown in \Cref{thm: isotropic_rows}.
\end{proof}

We now turn to the proof of the lower bound in the case of deep networks. 
For what follows, we assume that the considered $\relu$ network satisfies \Cref{assum:2}. 
We take a \emph{fixed} vector 
\begin{equation*}
x_0\defeq x^{(0)} \in \RR^d \setminus \{0\}
\end{equation*}
 and define the matrices $D^{(0)}(x_0),..., D^{(L-1)}(x_0)$ 
as introduced in \Cref{subsec:gradient}. 
Since $x_0$ is a fixed vector, we omit the argument and just write $D^{(\ell)}$ instead of $D^{(\ell)}(x_0)$. 

First, we prove that the product matrix $D^{(L-1)} W^{(L-1)} \cdots D^{(0)} W^{(0)}$ is almost isometric with high probability. 
This will be based on the fact that the rows of $\sqrt{N}D^{(\ell)}W^{(\ell)}$ are independent, isotropic random vectors (see \Cref{thm: dev_conditions}). 
However, in order to guarantee these properties, we have to make sure that the output of the previous layer $x^{(\ell)}$ is \emph{not} zero. 
Hence, in the following proposition we carefully keep track of this event as well. 
Moreover, since the ultimate goal is to apply \Cref{eq:upbound}, we have to ensure that the network $\Phi$ is differentiable at $x_0$ with
\begin{equation*}
\left(\nabla \Phi(x_0)\right)^T = W^{(L)}\cdot  D^{(L-1)} \cdot W^{(L-1)}\cdots D^{(0)} \cdot W^{(0)},
\end{equation*}
which is satisfied if all pre-activations throughout the network are non-zero. This is why we also consider this event in the following proposition.
\begin{proposition}\label{prop:gell}
Let $W^{(0)},..., W^{(L)}$ and $b^{(0)},..., b^{(L)}$ as in \Cref{assum:2}, and let $x^{(0)} \in \RR^d$ be fixed and let $D^{(0)},..., D^{(L-1)}$ and $x^{(1)},..., x^{(L)}$ as in \Cref{subsec:gradient}. For every $C>0, u \geq 0$ and $\ell \in \{1,...,L\}$, we write $\gell = \gell(u, C)$ for the event defined via the following three properties:
\begin{enumerate}
\item $(W^{(\ell')}x^{(\ell')} + b^{(\ell')})_i \neq 0$ \quad  for all $\ell' \in \{0,..., \ell-1\}$ and $i \in \{1,...,N\}$,
\item $x^{(\ell')} \neq 0$\quad for all $\ell' \in \{0,..., \ell\}$,
\item $\displaystyle \left(\left(1 - \frac{C \!\cdot\! (\sqrt{d} + u)}{\sqrt{N}}\right)_+\right) ^{\ell}\Vert y \Vert_2 \leq \Vert D^{(\ell-1)}W^{(\ell-1)}\cdots D^{(0)}W^{(0)}y \Vert_2 \leq \left( 1+ \frac{C \!\cdot \!(\sqrt{d} + u)}{\sqrt{N}}\right)^\ell \Vert y \Vert_2 $ holds uniformly over all $y \in \RR^d$.
\end{enumerate}
Recall that $a_+ = \max\{a,0\}$ for any $a \in \RR$.
Then, there exists an absolute constant $C>0$ with 
\begin{equation*}
\PP (\mathcal{G}^{(\ell)} ) = \PP (\mathcal{G}^{(\ell)}(u,C) )\geq \left(\left(1 - \frac{1}{2^N} - \exp(-u^2)\right)_+\right)^\ell
\end{equation*} 
for every $u \geq 0$.
\end{proposition}
\begin{proof}
The proof is via induction over $\ell$, where the constant $C$ is determined later. We note that for fixed $\ell$, the defining conditions of $\gell$ only depend on the weights $W^{(0)},..., W^{(\ell - 1)}$ and the biases $b^{(0)},..., b^{(\ell - 1)}$. 

We start with the case $\ell = 1$. We denote
\begin{align*}
A(1) &\defeq \left\{(W^{(0)},..., W^{(\ell - 1)}, b^{(0)}, ..., b^{(\ell - 1)}): \ \text{(1) is satisfied}\right\} \\
B(1) &\defeq \left\{(W^{(0)},..., W^{(\ell - 1)}, b^{(0)}, ..., b^{(\ell - 1)}): \ \text{(2) is satisfied}\right\} \\
C(1) &\defeq \left\{(W^{(0)},..., W^{(\ell - 1)}, b^{(0)}, ..., b^{(\ell - 1)}): \ \text{(3) is satisfied}\right\}.
\end{align*}
For every $i \in \{1,...,N\}$, we have, using $\ast$ to denote the convolution of two measures,
\begin{equation*}
\left(W^{(0)} x^{(0)} + b^{(0)} \right)_i \sim \mathcal{N}(0, 2/N \cdot \Vert x^{(0)} \Vert_2^2) \ast \mathcal{D}^{(0)}_i,
\end{equation*}
where the latter is an absolutely continuous and symmetric probability distribution. Note that we assume $x^{(0)} \neq 0$. Hence, we first conclude from the joint independence of the above random variables that
\begin{equation*}
\PP^{(W^{(0)}, b^{(0)})} \left(A(1)\right) = 1.
\end{equation*}
Moreover, we have
\begin{equation*}
(W^{(0)}, b^{(0)}) \notin B(1) \quad \Longleftrightarrow \quad \left(W^{(0)} x^{(0)} + b^{(0)} \right)_i \leq 0 \quad \text{for all } i \in \{1,...,N\}.
\end{equation*}
From the joint independence, the symmetry and the fact that the random variables $\left(W^{(0)} x^{(0)} + b^{(0)} \right)_i$ follow an absolutely continuous distribution for every $i$, we infer
\begin{equation*}
\PP^{(W^{(0)},b^{(0)})} \left(B(1)^c\right) = \frac{1}{2^N}.
\end{equation*}
Moreover, note that $\sqrt{N}D^{(0)}W^{(0)} = \sqrt{N}\diag(W^{(0)}x^{(0)}+ b^{(0)}) W^{(0)}$ is a matrix with independent isotropic sub-gaussian rows according to \Cref{thm: dev_conditions}. Therefore, the high probability version of the matrix deviation inequality (see \cite[Theorem 3]{Liaw2017}) yields
\begin{equation*}
\underset{y \in \overline{B}_d(0,1)}{\sup} \left\vert \Vert \sqrt{N}D^{(0)}W^{(0)}y\Vert_2 - \sqrt{N} \Vert y \Vert_2 \right\vert \leq C (\sqrt{d} + u)
\end{equation*}
with probability at least $ 1-\exp(- u ^2)$, where we employed \Cref{prop:gauss_width}. $C> 0$ is an absolute constant according to \Cref{thm: dev_conditions}. This yields
\begin{equation*}
\left(1 - \frac{C(\sqrt{d} + u)}{\sqrt{N}}\right)_+\Vert y \Vert_2 \leq \Vert D^{(0)}W^{(0)}y \Vert_2 \leq \left( 1+ \frac{C(\sqrt{d} + u)}{\sqrt{N}}\right) \Vert y \Vert_2 \quad \text{ for all } y \in \RR^d
\end{equation*}
with probability at least $1 - \exp(-u^2)$. This gives us $\PP(C(1)^c) \leq \exp(-u^2)$. This gives us in total
\begin{equation*}
\PP^{(W^{(0)},b^{(0)})}\left(\mathcal{G}^{(1)}\right) \geq 1 - \PP\left(A(1)^c\right)-\PP\left(B(1)^c\right)-\PP\left(C(1)^c\right) \geq 1 - 0 - \frac{1}{2^N} - \exp(-u^2).
\end{equation*}

Fix $1 \leq \ell < L$, set $V' \defeq  D^{(\ell-1)} W^{(\ell-1)} \cdots D^{(0)} W^{(0)}$ and write $\overset{\rightarrow}{W} \defeq (W^{(0)}, ..., W^{(\ell - 1)})$ for the tuple of the first $\ell$ weight matrices and $\overset{\rightarrow}{b} \defeq (b^{(0)}, ..., b^{(\ell - 1)})$ for the tuple of the first $\ell$ bias vectors and assume by induction that
\begin{equation} \label{eq: ind_hyp}
\PP\left((\arrow{W}, \arrow{b}) \in \gell\right) \geq  \left(\left(1 - \frac{1}{2^N} - \exp(-u^2)\right)_+\right)^\ell.
\end{equation} 
 Furthermore, we write
\begin{align*}
&A(\arrow{W}, \arrow{b}) \defeq \left\{ (W^{(\ell)}, b^{(\ell)}) : \ \left(W^{(\ell)}x^{(\ell)} + b^{(\ell)}\right)_i \neq 0 \text{ for all } i \in \{1,...,N\}\right\} \\
 &B(\arrow{W}, \arrow{b}) \defeq \left\{(W^{(\ell)}, b^{(\ell)}): \ x^{(\ell + 1)} = \relu(W^{(\ell)}x^{(\ell)} + b^{(\ell)}) \neq 0\right\} \\
 & C (\overset{\rightarrow}{W}, \overset{\rightarrow}{b}) \defeq \\
 &\left\{ (W^{(\ell )}, b^{(\ell)}):  \! \left(1\! - \!\frac{C(\sqrt{d}\! + \! u)}{\sqrt{N}}\right)_+ \! \Vert V'y \Vert_2 \!\leq \!\Vert D^{(\ell)}W^{(\ell)}V'y\Vert_2 \!\leq\! \left( 1\!+ \!\frac{C(\sqrt{d} \!+ \!u)}{\sqrt{N}}\right)\Vert V'y \Vert_2 \  \text{for all } \! y \! \in \! \RR^d\right\} 
\end{align*}
for each $\overset{\rightarrow}{W}, \overset{\rightarrow}{b}$, where we note that $V'$ is a function of only $\arrow{W}$ and $\arrow{b}$ (and $x^{(0)}$, which is fixed).
Then we see that
\begin{equation*} 
(\arrow{W}, \arrow{b}) \in \gell, \ (W^{(\ell)}, b^{(\ell)}) \in A(\arrow{W}, \arrow{b}) \cap B(\arrow{W}, \arrow{b}) \cap C(\arrow{W}, \arrow{b}) \quad \Longrightarrow \quad (\arrow{W}, W^{(\ell)}, \arrow{b}, b^{(\ell)}) \in \gelll.
\end{equation*}
In view of \Cref{prop:highprob}, we thus seek to bound 
\begin{equation*}
\PP^{(W^{(\ell)}, b^{(\ell)})} \left(A (\overset{\rightarrow}{W}, \overset{\rightarrow}{b}) \cap B(\arrow{W}, \arrow{b}) \cap C(\arrow{W}, \arrow{b})\right)
\end{equation*}
 from below, where $ (\overset{\rightarrow}{W}, \overset{\rightarrow}{b}) \in \gell$ is fixed. 
 
 To this end, we consider the sets $A(\arrow{W}, \arrow{b}), B(\arrow{W}, \arrow{b}), C(\arrow{W}, \arrow{b})$ individually. Note that the vector $x^{(\ell)}$ is fixed and from $(\arrow{W}, \arrow{b}) \in \gell$ we infer $x^{(\ell)} \neq 0$. Hence,
 \begin{equation*}
 \left(W^{(\ell)} x^{(\ell)} + b^{(\ell)} \right)_i \sim \mathcal{N}(0, 2/N \cdot \Vert x^{(\ell)} \Vert_2^2) * \mathcal{D}^{(\ell)}_i
 \end{equation*}
 for every $i \in \{1,...,N\}$, where the latter is again a symmetric and (absolutely) continuous probability distribution. Similar to the case $\ell = 1$, this gives us 
 \begin{equation*}
 \PP^{(W^{(\ell)}, b^{(\ell)})} \left(A(\arrow{W}, \arrow{b})\right) = 1.
 \end{equation*}
 Moreover, as in the case $\ell = 1$, we get 
\begin{equation*}
(W^{(\ell)}, b^{(\ell)}) \notin B(\arrow{W}, \arrow{b}) \quad \Longleftrightarrow \quad \left(W^{(\ell)} x^{(\ell)} + b^{(\ell)} \right)_i \leq 0 \quad \text{for all } i \in \{1,...,N\}.
\end{equation*}
 Again, from the joint independence, the symmetry and the absolute continuity of the distribution of the random variables 
 \begin{equation*}
 \left(W^{(\ell)} x^{(\ell)} + b^{(\ell)} \right)_i , 
 \end{equation*}
 we get
 \begin{equation*}
 \PP^{(W^{(\ell)}, b^{(\ell)})} \left(B(\arrow{W}, \arrow{b})^c\right) = \frac{1}{2^N}.
 \end{equation*}
 Lastly, according to \Cref{thm: dev_conditions}, the matrix
\begin{equation*}
\sqrt{N} D^{(\ell)}W^{(\ell)} = \sqrt{N} \diag(W^{(\ell)}x^{(\ell)} + b^{(\ell)})W^{(\ell)}
\end{equation*}
has independent, sub-gaussian, isotropic rows. Hence, we may again apply the high-probability version of the matrix deviation inequality (\cite[Theorem 3]{Liaw2017}) and obtain
\begin{equation*}
\underset{y \in \IM(V') \cap \overline{B}_N(0,1)}{\sup} \left\vert \sqrt{N} \Vert D^{(\ell)}W^{(\ell)}y\Vert_2 - \sqrt{N} \Vert y \Vert_2\right\vert \leq C (\sqrt{d} + u) 
\end{equation*}
with probability at least $ 1 - \exp(-u^2)$, where we again used \Cref{prop:gauss_width}, noting that the subspace $\IM(V')$ is at most $d$-dimensional since
\begin{equation*}
V' =  D^{(\ell-1)} W^{(\ell-1)} \cdot \cdot \cdot D^{(0)} W^{(0)} \in \RR^{N \times d}.
\end{equation*}
But this directly implies
\begin{equation*}
\left(1- \frac{C(\sqrt{d} + u)}{\sqrt{N}}\right)_+ \Vert y \Vert_2 \leq \Vert D^{(\ell)}W^{(\ell)}y \Vert_2 \leq \left(1+ \frac{C(\sqrt{d} + u)}{\sqrt{N}}\right) \Vert y \Vert_2 \quad \forall y \in \IM(V')
\end{equation*}
with probability at least $1 - \exp(-u^2)$, which means
\begin{equation*}
\PP^{(W^{(\ell)}, b^{(\ell)})} (C (\overset{\rightarrow}{W}, \overset{\rightarrow}{b})) \geq 1-  \exp(-u^2).
\end{equation*}
Hence, we get
\begin{align*}
&\norel\PP^{(W^{(\ell)}, b^{(\ell)})} \left(A (\overset{\rightarrow}{W}, \overset{\rightarrow}{b}) \cap B(\arrow{W}, \arrow{b}) \cap C(\arrow{W}, \arrow{b})\right) \\
&\geq 1 -  \PP^{(W^{(\ell)}, b^{(\ell)})}\left(A (\overset{\rightarrow}{W}, \overset{\rightarrow}{b})^c\right) - \PP^{(W^{(\ell)}, b^{(\ell)})}\left(B (\overset{\rightarrow}{W}, \overset{\rightarrow}{b})^c\right) - \PP^{(W^{(\ell)}, b^{(\ell)})}\left(C (\overset{\rightarrow}{W}, \overset{\rightarrow}{b})^c\right) \\
&\geq 1 - 0 - \frac{1}{2^N} - \exp(-u^2) = 1 - \frac{1}{2^N} -\exp(-u^2)
\end{align*}
for every fixed $(\arrow{W}, \arrow{b}) \in \gell$.
Using \Cref{prop:highprob}, we obtain
\begin{align*}
\PP \left(\gelll \right) \geq \left(1 - \frac{1}{2^N} -\exp(-u^2)\right)_+ \cdot \PP^{(\arrow{W}, \arrow{b})}(\gell) \geq \left(\left(1 - \frac{1}{2^N} -\exp(-u^2)\right)_+\right)^{\ell + 1},
\end{align*}
as was to be shown. Here, we applied the $+$-operator since probabilities are non-negative. 
\end{proof}
To finalize our result, we introduce randomness in $W^{(L)}$ as well and use an argument based on the singular value decomposition of the product $D^{(L-1)}W^{(L-1)} \cdots D^{(0)}W^{(0)}$.
\begin{proposition} \label{prop:grad_lower}
Let $\Phi$ and $W^{(0)},..., W^{(L)}$ as well as $b^{(0)},..., b^{(L)}$ as in \Cref{assum:2} and let $x_0 \defeq x^{(0)} \in \RR^d$ be fixed. Let $D^{(\ell)} \defeq D^{(\ell)}(x_0)$ for $0 \leq \ell \leq L-1$ with $D^{(\ell)}(x_0)$ as in \Cref{subsec:gradient}. We let $k \defeq \min  \{N,d\}, \  u \geq 0, C> 0$ and $t \geq 0$ and define the event $\mathcal{A} = \mathcal{A}(u,t,C)$ via the properties 
\begin{enumerate}
\item{\begin{align*}
\left(\left(1 - \frac{C\cdot (\sqrt{d} + u)}{\sqrt{N}}\right)_+\right)^{L} \cdot (\sqrt{k} - t) &\leq\Vert W^{(L)} D^{(L-1)} W^{(L-1)} \cdot \cdot \cdot D^{(0)} W^{(0)} \Vert_2 \\
&\leq \left(1 + \frac{C\cdot(\sqrt{d} + u)}{\sqrt{N}}\right)^{L} \cdot (\sqrt{k} + t),
\end{align*} }
\item{$\Phi$ is differentiable at $x_0$ with 
\begin{equation*}
\left(\Phi(x_0)\right)^T =W^{(L)} \cdot D^{(L-1)} \cdot W^{(L-1)}\cdots D^{(0)} \cdot W^{(0)}.
\end{equation*}}
\end{enumerate}
Then there exist absolute constants $C, c_1 > 0$ such that 
\begin{equation*}
\PP\left(\mathcal{A}\right) = \PP\left(\mathcal{A}(u,t,C)\right) \geq \left(\left(1 - \frac{1}{2^N}-\exp(-u^2)\right)_+\right)^L \cdot (1- 2 \exp (-c_1 t^2))_+
\end{equation*}
for every $u,t \geq 0$.
\end{proposition}
\begin{proof}
We again denote $V' \defeq D^{(L-1)}W^{(L-1)}\cdot \cdot \cdot D^{(0)}W^{(0)}$. We decompose $V' = U \Sigma Q^T$ with orthogonal matrices $U \in \RR^{N \times N}$ and $Q \in \RR^{d \times d}$ and a matrix $\Sigma \in \RR^{N \times d}$ of the form 
\begin{equation*}
\Sigma = \left(\begin{matrix} \sigma_1 & & \\
				& \ddots & \\
				& & \sigma_d \\
				\hline 
				0 & \hdots & 0 \\
				\vdots & \ddots & \vdots \\
				0 & \hdots & 0  \end{matrix} \right)  \quad \text{if $N \geq d$} \qquad \text{or} \qquad \Sigma = \left(\begin{array}{ccc|ccc} \sigma_1 & &  & 0 & \hdots & 0\\
				& \ddots &  & \vdots & \ddots & \vdots \\
				& & \sigma_N & 0 & \hdots & 0
				 \end{array} \right) \quad \text{if $N \leq d$}
\end{equation*}
with $\sigma_1 \geq ... \geq \sigma_k \geq 0$ (singular value decomposition). Hence, recalling that $W^{(L)} \in \RR^{1 \times N}$, we get
\begin{align}
	\Vert W^{(L)} V'\Vert_2  &= \Vert  W^{(L)} U \Sigma Q^T \Vert_2 = \Vert W^{(L)} U \Sigma \Vert _2 = \left\Vert \left( W^{(L)}U \Sigma\right)^T \right\Vert_2 \nonumber\\ 
	&= \sqrt{ \sum_{i= 1}^k \left(\sigma_i \cdot \left(W^{(L)}U\right)_i\right)^2}
	\label{eq: firstbound} \ 
	 \begin{cases} \leq \sigma_1 \cdot \Vert W' \Vert_2, \\  \geq \sigma_k \cdot \Vert W' \Vert_2 \end{cases}
\end{align}
with $W' \in \RR^{1 \times k}$ denoting the vector of the first $k$ entries of $W^{(L)}U$. 

Let $C>0$ be the absolute constant from \Cref{prop:gell}. We denote $\overset{\rightarrow}{W} \defeq (W^{(0)}, ..., W^{(L-1)})$ and $\overset{\rightarrow}{b} \defeq (b ^{(0)}, ..., b^{(L-1)})$ and fix $u,t \geq0$. Furthermore, let
\begin{align*}
\mathcal{B}(\overset{\rightarrow}{W}, \overset{\rightarrow}{b})\defeq \left\{ W^{(L)}: \ \sqrt{d} + t \geq \Vert W'(\arrow{W}, \arrow{b}, W^{(L)}) \Vert_2 \geq \sqrt{d} - t\right\},
\end{align*}
where we wrote $W'(\arrow{W}, \arrow{b}, W^{(L)})$ to emphasize the dependence of $W'$ on $\arrow{W}, \arrow{b}$ and $W^{(L)}$.

Note that for $(\arrow{W}, \arrow{b}) \in \mathcal{G}^{(L)}$ (where $\mathcal{G}^{(L)}$ is as in \Cref{prop:gell}), property $(2)$ follows directly and independent of the choice of $W^{(L)}$ and $b^{(L)}$. This is due to defining property (1) in \Cref{prop:gell} and the fact that the $\relu$ is differentiable on $\RR \setminus \{0\}$. Hence, from \eqref{eq: firstbound} we infer
\begin{equation} \label{eq: secondbound}
(\arrow{W}, \arrow{b}) \in \mathcal{G}^{(L)}, \ W^{(L)} \in \mathcal{B}(\arrow{W}, \arrow{b}) \quad \Longrightarrow \quad (\arrow{W}, \arrow{b}, W^{(L)}) \in \mathcal{A},
\end{equation}
where we also applied \cite[Equation (4.5)]{vershynin_high-dimensional_2018}.
 From \Cref{prop:gell} we deduce 
\begin{equation} \label{eq: probbound}
\PP^{(\arrow{W}, \arrow{b})} (\mathcal{G}^{(L)}) \geq \left(\left(1- \frac{1}{2^N}-\exp(-u^2)\right)_+\right)^L.
\end{equation}
Furthermore, for fixed $\arrow{W}$ and $\arrow{b}$, the rotation invariance of the Gaussian distribution \cite[Proposition 3.3.2]{vershynin_high-dimensional_2018} implies
\begin{equation*}
W^{(L)} U \sim W^{(L)}.
\end{equation*}
Therefore, $W'$ is a $k$-dimensional random vector with $(W')^T \sim \mathcal{N}(0,I_k)$. Thus, \cite[Theorem 3.1.1]{vershynin_high-dimensional_2018} yields
\begin{equation*}
\left\Vert \Vert W'\Vert_2 - \sqrt{k}\right\Vert_{\psi_2} \leq C_2 \quad \text{(conditioned on $\arrow{W},\arrow{b}$)}
\end{equation*}
with an absolute constant $C_2 > 0$. 
From \cite[Equation (2.14)]{vershynin_high-dimensional_2018} we get
\begin{equation*}
\PP^{W^{(L)}} \left(\left\vert \Vert W' \Vert_2 - \sqrt{k}\right\vert \geq t\right) \leq 2 \exp (-c_1t^2) \quad \text{for fixed } \arrow{W}, \arrow{b}
\end{equation*}
with an absolute constant $c_1 > 0$,
and hence
\begin{equation*} 
\PP^{W^{(L)}} \left(\sqrt{k} + t \geq \Vert W'\Vert_2 \geq (\sqrt{k} - t)_+ \right) \geq 1 - 2 \exp(-c_1 t^2) \quad \text{for fixed } \arrow{W}, \arrow{b}.
\end{equation*}

From \Cref{prop:highprob} and \eqref{eq: secondbound}, we see 
\begin{align*}
\PP^{(\arrow{W}, \arrow{b}, W^{(L)})}(\mathcal{A}) &\geq (1 - 2 \exp(-c_1 t^2))_+ \cdot \PP^{(\arrow{W}, \arrow{b})} (\mathcal{G}^{(L)}) \\
\overset{\eqref{eq: probbound}}&{\geq} (1 - 2 \exp(-c_1 t^2))_+ \cdot \left(\left(1 -  \frac{1}{2^N}-\exp(-u^2)\right)_+\right)^L,
\end{align*}
as was to be shown.
\end{proof}
We remark that the previous result, in addition to providing a lower bound on the Lipschitz constant of random neural networks, is of independent interest on its own,
since it provides a lower and \emph{upper} bound on the gradient of a random ReLU network at a fixed point $x_0\neq 0$. 
It is an interesting question whether this \emph{pointwise} estimate can be used to get a uniform estimate that holds for every point $x^{(0)}$,
thus yielding an \emph{upper} bound on the Lipschitz constant as well.

Moreover, we note that one can even show that a random $\relu$ network $\Phi$ is \emph{almost surely} differentiable with 
\begin{equation*}
\nabla \Phi (x_0)^T = W^{(L)} \cdot D^{(L-1)} \cdots D^{(0)} \cdot W^{(0)}
\end{equation*} 
at any fixed point $x_0 \neq 0$ (and not only with high probability as stated in \Cref{prop:grad_lower}). 
A proof of this fact (which we expect to be of independent interest) is contained in \Cref{app:diff}.

\Cref{prop:grad_lower} and \Cref{eq:upbound} directly give us the following lower bound on the Lipschitz constant of random ReLU networks.
\begin{theorem}\label{thm:low_bound_ut}
There exist absolute constants $C, c_1 > 0$ with the following property: If $\Phi:\RR^d \to \RR$ is a random ReLU network of width $N$ and with $L$ hidden layers according to the random initialization as described in \Cref{assum:2}, then for any $u,t \geq 0$, writing $k \defeq \min\{d,N\}$, it holds
\begin{equation*}
\lip(\Phi) \geq \left(\left(1 - \frac{C\cdot(\sqrt{d} + u)}{\sqrt{N}}\right)_+\right)^{L} \cdot (\sqrt{k} - t) 
\end{equation*}
with probability at least $ \left(\left(1 - \frac{1}{2^N}-\exp(-u^2)\right)_+\right)^L \cdot (1- 2 \exp (-c_1 t^2))_+$.
\end{theorem}

 By plugging in special values for $t$ and $u$ and assuming $N \gtrsim dL^2$, we can now prove the main result.

\renewcommand*{\proofname}{Proof of \Cref{thm:main_3}}
\begin{proof}
Let $\widetilde{C}$ and $\widetilde{c_1}$ be the relabeled constants from \Cref{thm:low_bound_ut}. We can clearly assume $\widetilde{C} \geq 1$. We define the new constants $C \defeq (4\widetilde{C} )^2$ and $c_1 \defeq \widetilde{c_1}/4$. We assume $N \geq C \cdot d \cdot L^2$ and let $u = \frac{\sqrt{N}}{4\widetilde{C}L}$ and $t = \sqrt{d}/2.$ Note that $N \geq CdL^2$ is equivalent to
\begin{equation*}
\sqrt{d} \leq \frac{\sqrt{N}}{4\widetilde{C}L}.
\end{equation*}
We get
\begin{equation*}
\widetilde{C} \cdot \left(\sqrt{d} + u\right) =\widetilde{C} \cdot \left(\sqrt{d} + \frac{\sqrt{N}}{4\widetilde{C}L}\right) \leq \widetilde{C} \cdot \left(\frac{\sqrt{N}}{4\widetilde{C}L} + \frac{\sqrt{N}}{4\widetilde{C}L}\right) = \frac{\sqrt{N}}{2L} \leq \sqrt{N}
\end{equation*}
and hence
\begin{equation*}
1 - \frac{\widetilde{C}\cdot (\sqrt{d} + u)}{\sqrt{N}} \geq 0.
\end{equation*}
Moreover, $N \geq C \cdot d \cdot L^2 \geq d$ and thus $k = \min \{N,d\} = d$. 
Therefore, \Cref{thm:low_bound_ut} yields
\begin{equation*}
\lip(\Phi)\geq \left(1 - \frac{\widetilde{C}(\sqrt{d}+u) }{\sqrt{N}}\right)^{L} \cdot \frac{1}{2} \cdot \sqrt{d}
\end{equation*}
with probability at least 
\begin{align*}
&\norel\left(\left(1 -  \frac{1}{2^N}-\exp(-u^2)\right)_+\right)^L \cdot (1- 2 \exp (-\widetilde{c_1} d / 4))_+\\
&= \left(1 - \frac{1}{2^N}-\exp(-N/(CL^2))\right)^L \cdot (1- 2 \exp (-c_1 d ))_+.
\end{align*} 
Here, we implicitly used that 
\begin{equation*}
1 - \frac{1}{2^N}-\exp(-N/(CL^2)) \geq 1 - \frac{1}{2^N} - \exp(-d) \geq 0,
\end{equation*}
where we employed $N \geq CdL^2$.
Moreover, note that 
\begin{equation*}
\frac{\widetilde{C}\cdot (\sqrt{d} + u)}{\sqrt{N}} = \widetilde{C} \cdot \left(\frac{\sqrt{d}}{\sqrt{N}} + \frac{1}{4\widetilde{C}L}\right) \leq \widetilde{C} \cdot \left(\frac{1}{4\widetilde{C}L}+\frac{1}{4\widetilde{C}L}\right) = \frac{1}{2L},
\end{equation*}
which yields
\begin{equation*}
 \left(1 - \frac{\widetilde{C}\cdot (\sqrt{d}+u) }{\sqrt{N}}\right)^{L} \geq \left( 1 - \frac{1}{2L}\right)^L \geq 1/2 
\end{equation*}
for every $L \in \NN$, which follows from Bernoulli's inequality. This gives us (1).

For the expectation bound, note that by Markov's inequality we have
\begin{equation*}
\EE[\lip(\Phi)] \geq \PP(\lip(\Phi) \geq \sqrt{d}/4) \cdot \frac{\sqrt{d}}{4} \geq \left(\left(1-\frac{1}{2^N}-\exp(-u^2)\right)^L\cdot (1-2\exp(-c_1d))\right)\cdot \frac{\sqrt{d}}{4}.
\end{equation*}
First note that there is an absolute constant $c_3>0$ with 
\begin{equation*}
1-2\exp(-c_1d) \geq c_3
\end{equation*}
for every $d \geq \left\lfloor \frac{\ln(2)}{c_1}\right\rfloor + 1$. Therefore, it remains to find a uniform bound for 
\begin{equation*}
\left(1-\frac{1}{2^N}-\exp(-u)\right)^L.
\end{equation*}
We apply Bernoulli's inequality and $L \leq N$ to obtain
\begin{align*}
\left(1-\frac{1}{2^N}-\exp(-u^2)\right)^L &\geq 1 - \frac{L}{2^N} - L\exp(-u^2)\geq 1 - \frac{N}{2^N} - L\exp \left(-N/(CL^2)\right) \\
& \geq\frac{1}{2} - L\exp(-N/(CL^2)).
\end{align*}
If we now assume $N \geq CL^2 \ln(4L)$, we get
\begin{equation*}
\exp(-N/(CL^2)) \leq \exp(-CL^2\ln(4L)/(CL^2)) = \exp(- \ln(4L)) = \frac{1}{4L}.
\end{equation*}
Hence, the claim follows letting $c_2 \defeq \frac{c_3}{16}$.
\end{proof}
\renewcommand*{\proofname}{Proof}

\appendix
\section{Postponed proofs for the preliminary results} \label{sec:prelim_proofs}
In this appendix, we provide the postponed proofs for the preliminary results (see \Cref{sec:preliminaries}). Specifically, we prove \Cref{prop:lipgrad,prop:covering_ball,prop:vc_half_spaces_2,prop:gauss_width,prop:Cbound,prop:not_working_deep}. 

We start with the proof of \Cref{prop:lipgrad}. This proposition is similar to other statements in the literature and is probably folklore. However, for the sake of completeness, we include a proof, as we could not locate a convenient reference.
\renewcommand*{\proofname}{Proof of \Cref{prop:lipgrad}}
\begin{proof}
By \cite[Section~4.2.3]{evans_measure_1992} it holds $f \in W^{1, \infty}_{\text{loc}}(\RR^d)$. Here, $f \in W^{1, \infty}_{\text{loc}}(\RR^d)$ denotes the set of functions that are locally (on every open bounded set $U \subseteq \RR^d$) in the Sobolev space $W^{1, \infty}$. Then, \cite[Theorem~1~in~Section~6.2]{evans_measure_1992} implies that there exists a nullset $N \subseteq \RR^d$ such that for every $x \in \RR^d \setminus N$, $f$ is differentiable at $x$, and the partial derivatives agree with the weak partial derivatives on $\RR^d \setminus N$. For $i \in \{1,...,d\}$ let $\widetilde{\partial}_i f$ be an explicit weak $i$-th partial derivative of $f$ satisfying 
\begin{equation*}
\widetilde{\partial}_i f = \partial_i f \quad \text{on }\RR^d \setminus N
\end{equation*}
where $\partial_i$ denotes the usual $i$-th partial derivative. We further write
\begin{equation*}
\widetilde{\nabla} f \defeq \left(\widetilde{\partial}_1 f, ..., \widetilde{\partial}_d f\right)^T.
\end{equation*}

Now, let $\varphi \in C_c^\infty (\RR^d)$ be a smooth function with compact support satisfying $\varphi \geq 0$ and furthermore $\int_{\RR^d} \varphi(x) \ \dd x = 1$. See, e.g., \cite[Section~4.2.1]{evans_measure_1992} for an explicit example of such a function. For $\eps > 0$ let $\varphi_\eps (x) \defeq \eps^{-d} \varphi(x / \eps)$. Since $f$ is continuous, the convolution $f_\eps \defeq f \ast \varphi_\eps$ converges pointwise (even locally uniformly) to $f$ as $\eps \to 0$; see \cite[Theorem~1~in~Section~4.2]{evans_measure_1992}. The same theorem also shows that $f_\eps \in C^\infty (\RR^d)$ with 
\begin{equation*}
\partial_i f_\eps = (\widetilde{\partial}_i f) \ast \varphi_\eps.
\end{equation*}
Hence, since $M \subseteq \RR^d$ is a set of full measure, and $N$ is a null-set,
\begin{align*}
\Vert \nabla f_\eps (x) \Vert_2 &= \left\Vert (\widetilde{\nabla}f \ast \varphi_\eps)(x)\right\Vert_2 = \left\Vert \int_{\RR^d} \widetilde{\nabla} f (y) \varphi_\eps (x-y)\ \dd y\right\Vert_2 = \left\Vert \int_{M \setminus N} \widetilde{\nabla} f (y) \varphi_\eps (x-y)\ \dd y\right\Vert_2 \\
&\leq \int_{M \setminus N} \Vert \nabla f (y)\Vert_2 \vert \varphi_\eps (x-y)\vert\ \dd y \leq \underset{y \in M }{\sup} \Vert \nabla f(y)\Vert_2 =: L
\end{align*}
for every $x \in \RR^d$. This implies 
\begin{equation*}
\left\vert f_\eps(x) - f_\eps(y)\right\vert \leq L \cdot \Vert x-y \Vert_2
\end{equation*}
for all $x,y \in \RR^d$ and $\eps > 0$. But then it also follows
\begin{equation*}
\vert f(x) - f(y) \vert = \lim_{\eps \to 0} \ \vert f_\eps(x) -f_\eps(y)\vert \leq L \cdot \Vert x-y \Vert_2
\end{equation*}
and thus $\lip(f) \leq L$.

To prove the inequality in the other direction, let $x \in M$. Assume without loss of generality that $\nabla f (x) \neq 0$ and let $\nu \defeq \frac{\nabla f (x)}{\Vert \nabla f (x) \Vert_2}$. Then it holds
\begin{align*}
\Vert \nabla f (x) \Vert_2 &= \langle \nabla f (x), \nu\rangle= \lim_{t \to 0} \frac{ f(x+t\nu) - f(x)}{ t}= \lim_{t \to 0} \frac{\vert f(x+t\nu) - f(x)\vert}{\vert t \vert} \\
&=  \lim_{t \to 0} \frac{\vert f(x+t\nu) - f(x)\vert}{\Vert (x + t\nu) - x \Vert_2} \leq \lip(f),
\end{align*}
as was to be shown. Note that we applied the absolute value at the third equality, since the limit is positive. 
\end{proof}

We now move to the proofs of \Cref{prop:covering_ball,prop:vc_half_spaces_2,prop:gauss_width}. To this end, we first formulate a well-known lemma, which is needed for the proofs of these propositions. For the sake of completeness, we decided to include a proof of this lemma in this appendix.
\begin{lemma}\label{lem:Uhelp}
Let $V \subseteq \RR^k$ be a linear subspace and let $\ell = \dim(V)>0$. Then there exists a matrix $U \in \RR^{k \times \ell}$ with $U^TU = I_{\ell \times \ell}$, $\IM(U) = V$ and $\langle Ux_1, Ux_2 \rangle = \langle x_1,x_2 \rangle$ for each $x_1,x_2 \in \RR^\ell$ as well as $\langle U^Ty_1, U^Ty_2 \rangle = \langle y_1,y_2 \rangle$ for every $y_1,y_2 \in V$. Here, $I_{\ell \times \ell}$ denotes the $\ell$-dimensional identity matrix.
\end{lemma}
\renewcommand*{\proofname}{Proof}
\begin{proof}
We pick an orthonormal basis $u_1,...,u_\ell$ of $V$, arrange it columnwise in a matrix and denote the resulting matrix by $U = (u_1 \hspace{0.05cm}\vline \hspace{0.05cm}\cdots \hspace{0.05cm}\vline\hspace{0.05cm} u_\ell)$. From the definition of an orthonormal basis, it follows $U^T U = I_{\ell \times \ell}$ and $\IM(U)=V$. Moreover, for $x_1,x_2 \in \RR^\ell$ we get by definition
\begin{equation*}
\langle Ux_1, Ux_2 \rangle = \langle U^T U x_1, x_2 \rangle \overset{U^T U = I_{\ell \times \ell}}{=} \langle x_1, x_2 \rangle.
\end{equation*}
For $y_1,y_2 \in V$ we can pick $a_1,a_2 \in \RR^\ell$ with $Ua_1 = y_1$ and $Ua_2 = y_2$. This then gives us
\begin{equation*}
\langle U^Ty_1, U^Ty_2 \rangle = \langle U^T U a_1, U^T U a_2 \rangle = \langle a_1, a_2 \rangle = \langle Ua_1, Ua_2 \rangle = \langle y_1, y_2 \rangle,
\end{equation*}
as was to be shown.
\end{proof}
Note that \Cref{lem:Uhelp} in particular implies that both $U$ and $U^T$ preserve the Euclidean norm, i.e., it holds $\Vert Ux \Vert_2 = \Vert x \Vert_2$ for every $x \in \RR^\ell$ and $\Vert U^T y \Vert_2 = \Vert y \Vert_2$ for every $y \in V$ (but not in general for $y \in \RR^k$).

We can now show \Cref{prop:covering_ball,prop:vc_half_spaces_2,prop:gauss_width}. The proof of \Cref{prop:covering_ball} relies on the well-known fact that the $\eps$-covering-number of the Euclidean unit ball in $\RR^k$ can be bounded from above by $(1 + \frac{2}{\eps})^k$, combined with \Cref{lem:Uhelp}.
\renewcommand*{\proofname}{Proof of \Cref{prop:covering_ball}}
\begin{proof}
Let $\ell \defeq \dim(V)>0$.
In case of $\ell = 0$, the claim is trivial since then $\overline{B}_k(0,1)\cap V = \{0\}$.
Hence, we can assume $\ell>0$. 
Choose $U \in \RR^{k \times \ell}$ according to \Cref{lem:Uhelp}. 
Pick $x_1, ..., x_M \in \overline{B}_\ell(0,1)$ with $M \leq \left(\frac{2}{\eps} + 1\right)^{\ell}$ and
\begin{equation*}
\overline{B}_\ell(0,1) \subseteq \bigcup_{i=1}^M \overline{B}_\ell (x_i, \eps),
\end{equation*}
which is possible due to \cite[Corollary 4.2.13]{vershynin_high-dimensional_2018}. 
We then define $y_i \defeq Ux_i \in V$ for $i=1,...,M$. 
Then it holds $\Vert y _i \Vert_2 = \Vert x_i \Vert_2 \leq 1$
and hence, $y_i \in V \cap \overline{B}_k(0,1)$ for every $i = 1,...,M$. 
Let $y \in V \cap \overline{B}_k(0,1)$ be arbitrary and note that $U^Ty \in \overline{B}_\ell(0,1)$.
Hence, there exists $i \in \{1,...,M\}$ with $\Vert x_i - U^Ty \Vert_2 \leq \eps$. But then we get
\begin{equation*}
\Vert y_i - y \Vert_2 = \Vert Ux_i - y \Vert_2 = \Vert U^T (Ux_i - y)\Vert_2 = \Vert x_i - U^T y \Vert_2 \leq \eps,
\end{equation*}
where the second equality follows since $U^T$ is norm-preserving on $V$ and $Ux_i - y \in V$. Since $y$ was arbitrary, we see
\begin{equation}\label{eq:U_final_cov}
\overline{B}_k(0,1) \cap V \subseteq \bigcup_{i=1}^M \overline{B}_k(y_i, \eps) \cap V.
\end{equation}
This yields the claim. 
\end{proof}
For the proof of \Cref{prop:vc_half_spaces_2}, 
we use the well-known fact that the VC-dimension of the set of homogeneous 
half spaces in $\RR^k$ equals $k$, combined with \Cref{lem:Uhelp}.
\renewcommand*{\proofname}{Proof of \Cref{prop:vc_half_spaces_2}}
\begin{proof}
Let $\ell \defeq \dim(V)$. 
In case of $\ell = 0$, we have $\mathcal{F} = \{0\}$ and hence $\vc(\mathcal{F}) = 0$, so that the claim is trivial.
Hence, we can assume $\ell > 0$.
Pick $U \in \RR^{k \times \ell}$ according to \Cref{lem:Uhelp}.
For $\alpha \in V$ we define
\begin{equation*}
\widetilde{f_\alpha}: \quad \RR^\ell \to \RR, \quad x \mapsto \mathbbm{1}_{(U^T\alpha)^T x > 0}\quad \text{and}\quad \widetilde{\mathcal{F}} \defeq \{ \widetilde{f_\alpha}: \ \alpha \in V\}.
\end{equation*}
Since $U^T$ induces a surjective mapping from $V$ to $\RR^\ell$ (because $U^T$ is isometric on $V$ and $\dim(V) = \ell$) we infer
\begin{equation*}
\widetilde{\mathcal{F}} = \left\{f_\alpha : \ \alpha \in \RR^\ell\right\}.
\end{equation*}
From \cite[Theorem 9.2]{shalev2014understanding} we infer $\vc(\widetilde{\mathcal{F}}) = \ell$. 
Therefore, it suffices to show that
\begin{equation*}
\vc(\mathcal{F}) = \vc(\mathcal{\widetilde{F}}).
\end{equation*}
To this end, let $t \leq \vc(\mathcal{F})$ and let $\{v_1, ..., v_t\} \subseteq V$ be a subset of $V$ with $t$ elements that is shattered by $\mathcal{F}$.
Let $\eta \in \{0,1\}^t$ be arbitrary and $\alpha \in V$ such that
\begin{equation*}
(f_\alpha(v_1),..., f_\alpha(v_t)) = \eta.
\end{equation*}
We then consider $\{U^Tv_1, ..., U^T v_t\} \subseteq \RR^\ell$ and note that this set also consists of $t$ elements since $U^T$ induces an injective mapping $V \to \RR^\ell$. 
Fix $j \in \{1,...,t\}$.
Then it holds
\begin{equation*}
\widetilde{f_\alpha}(U^Tv_j) = \mathbbm{1}_{\langle U^T\alpha, U^T v_j\rangle>0} = \mathbbm{1}_{\langle \alpha, v_j \rangle > 0} = f_\alpha(v_j)= \eta_j.
\end{equation*}
But, since $j$ was picked arbitrarily, this shows
\begin{equation*}
(\widetilde{f_\alpha}(U^T v_1), ...,\widetilde{f_\alpha}( U^T v_t)) = \eta
\end{equation*}
and thus, the set $\{U^Tv_1, ..., U^T v_t\}$ is shattered by $\widetilde{\mathcal{F}}$.
Therefore, we conclude $\vc (\mathcal{F}) \leq \vc(\widetilde{\mathcal{F}})$.
Similarly, we can derive the bound in the other direction: Let $t \leq \vc(\widetilde{\mathcal{F}})$ and $\{w_1 ,..., w_t\} \subseteq \RR^\ell$ a subset of $\RR^\ell$ with $t$ elements that is shattered by $\mathcal{F}$. Let $\eta \in \{0,1\}^t$ be arbitrary and $\alpha \in V$ such that
\begin{equation*}
(\widetilde{f_\alpha}(w_1), ...,\widetilde{f_\alpha}(w_t)) = \eta.
\end{equation*}
We then consider $\{Uw_1, ..., U w_t\} \subseteq V$ and note that this set also consists of $t$ elements since $U$ induces an injective mapping $\RR^\ell \to V$. Moreover, for every $j \in \{1,...,t\}$ we have
\begin{equation*}
f_\alpha(Uw_j) = \mathbbm{1}_{\langle \alpha, Uw_j \rangle >0} = \mathbbm{1}_{\langle U^T \alpha, U^T U w_j \rangle > 0} = \widetilde{f_\alpha}(w_j) = \eta_j,
\end{equation*}
which gives us
\begin{equation*}
(f_\alpha(Uw_1),..., f_\alpha(Uw_t)) = \eta
\end{equation*}
and thus, the set $\{Uw_1,..., Uw_t\}$ is shattered by $\mathcal{F}$.
This yields $\vc (\mathcal{F}) \geq \vc(\widetilde{\mathcal{F}})$ and overall, we get the claim.
\end{proof}
The proof of \Cref{prop:gauss_width} makes heavy use of the fact that the Gaussian width of the Euclidean unit ball in $\RR^k$ can be bounded from above by $\sqrt{k}$. Again, we apply \Cref{lem:Uhelp} to transfer the problem to some lower-dimensional space $\RR^\ell$ and use the rotation invariance of the Gaussian distribution.
\renewcommand*{\proofname}{Proof of \Cref{prop:gauss_width}}
\begin{proof}
Let $\ell \defeq \dim(V)$. If $\ell = 0$, then $\overline{B}_k(0,1) \cap V = \{0\}$, and we trivially have $w(\{0\})= 0$. Hence, we can assume $\ell > 0$. Choose $U \in \RR^{k \times \ell}$ according to \Cref{lem:Uhelp}. From the fact that $U$ and $U^T$ preserve the Euclidean norm (on $\RR^\ell$ and $V$, respectively), we get 
\begin{equation*}
\overline{B}_k(0,1) \cap V = U\overline{B}_\ell(0,1).
\end{equation*}
This yields
\begin{align*}
w(\overline{B}_k(0,1) \cap V) &= \underset{g \sim \mathcal{N}(0, I_k)}{\EE}   \left[\underset{x \in \overline{B}_k(0,1) \cap V}{\sup}\  \langle g, x \rangle \right]= \underset{g \sim \mathcal{N}(0, I_k)}{\EE} \left[\underset{v \in U\overline{B}_\ell(0,1)}{\sup}\  \langle g, v \rangle \right]\\
&= \underset{g \sim \mathcal{N}(0, I_k)}{\EE}  \left[\underset{x \in \overline{B}_\ell(0,1)}{\sup}\  \langle g, Ux \rangle \right]= \underset{g \sim \mathcal{N}(0, I_k)}{\EE}  \left[ \underset{x \in \overline{B}_\ell(0,1)}{\sup}\  \langle U^Tg, x \rangle \right].
\end{align*}
By basis completion, pick a matrix $U' \in \RR^{k \times (k - \ell)}$ such that 
\begin{equation*}
U_0 \defeq \begin{pmatrix} U & \vline & U'\end{pmatrix} \in \RR^{k \times k}
\end{equation*}
is orthogonal. 
Then $(U_0)^T$ is also orthogonal and from the rotation invariance of $\mathcal{N}(0, I_k)$ (see, e.g., \cite[Proposition~3.3.2]{vershynin_high-dimensional_2018}) we see that, if we take $g \sim \mathcal{N}(0,I_k)$, it holds
\begin{equation*}
(U_0)^T g = \left(\begin{array}{c} U^T g \\  \hline \vspace{-0.3cm}\\ (U')^T g\end{array}\right) \sim \mathcal{N}(0, I_k),
\end{equation*}
which implies $U^Tg \sim \mathcal{N}(0, I_\ell)$. Therefore, we get
\begin{equation*}
w(\overline{B}_k(0,1) \cap V) = \underset{g \sim \mathcal{N}(0, I_k)}{\EE}  \left[ \underset{x \in \overline{B}_\ell(0,1)}{\sup}\  \langle U^Tg, x \rangle \right]= \underset{g \sim \mathcal{N}(0, I_\ell)}{\EE}  \left[ \underset{x \in \overline{B}_\ell(0,1)}{\sup}\  \langle g, x \rangle \right]= w(\overline{B}_\ell(0,1)).
\end{equation*}
The claim now follows using \cite[Proposition~7.5.2]{vershynin_high-dimensional_2018}, noting that $\diam(\overline{B}_\ell(0,1)) = 2$.
\end{proof}
We now move the proofs of \Cref{prop:Cbound,prop:not_working_deep}.
The idea of the proof of \Cref{prop:Cbound} is to find one specific choice of weight matrices for which the inequality
\begin{equation*}
\lip(\Phi) =  0 < \lip(\widetilde{\Phi})
\end{equation*}
holds for arbitrary biases. Then, using a continuity argument, we infer that the desired inequality even holds with positive probability.
\renewcommand*{\proofname}{Proof of \Cref{prop:Cbound}}
\begin{proof}
Let $b^{(0)} \in \RR^3$ and $b^{(1)} \in \RR$ be arbitrary but from now on fixed bias vectors. 
We consider the shallow network $\Phi_0$ given via the biases $b^{(0)}$ and $b^{(1)}$ and the weights $W^{(0)}_\ast \defeq \begin{pmatrix} 1 & 1\\ 1 & 0 \\ 0 & 1\end{pmatrix}$ and $W^{(1)}_\ast \defeq \begin{pmatrix} 1 & -1 & - 1\end{pmatrix}$ . 
This implies
\begin{equation*}
\Phi_0(x) = \relu(x_1 + x_2+b^{(0)}_1) - \relu(x_1 + b^{(0)}_2) - \relu(x_2 + b^{(0)}_3) + b^{(1)}
\end{equation*}
for all $x \in \RR^2$. 
Moreover, $\widetilde{\Phi_0}(x) = b^{(0)}_1 - b^{(0)}_2  - b^{(0)}_3+ b^{(1)}$ for all $x \in \RR^2$ and hence, $\lip(\widetilde{\Phi_0}) = 0$. 
On the other hand, we have
\begin{equation*}
\Phi_0(x) = x_1 + x_2 + b^{(0)}_1 - x_2 - b^{(0)}_3 +b^{(1)}= x_1 + b^{(0)}_1 - b^{(0)}_3 + b^{(1)},
\end{equation*}
whenever $x_1 \leq - b^{(0)}_2$, and $x_2 \geq \max\{- b^{(0)}_3, - x_1 - b^{(0)}_1\}$,  
whence $\lip(\Phi_0) \geq 1$. Therefore, it holds
\begin{equation*}
\lip(\Phi_0) > C \cdot \lip(\widetilde{\Phi_0}).
\end{equation*}
It remains to show that this even holds with positive probability. 
To see this, we still let $b^{(0)}$ and $b^{(1)}$ be fixed and note that $\lip(\widetilde{\Phi}) = \Vert W^{(1)} W^{(0)} \Vert_2$ depends continuously on the network weights $W^{(0)}$ and $W^{(1)}$. 
Moreover, we fix the point $x_2 \defeq \max\{- b^{(0)}_3, -b^{(0)}_1 + b^{(0)}_2 + 1\}$ and note that the expression $\vert \Phi((-b^{(0)}_2, x_2)) - \Phi(-b^{(0)}_2 - 1, x_2) \vert$ also depends continuously on the network weights. 
Therefore, for $(W^{(0)}, W^{(1)})$ on a sufficiently small open neighborhood of $(W_\ast^{(0)}, W_\ast^{(1)})$, we have
\begin{equation*}
\frac{\lip(\Phi)}{C} \geq \frac{\vert \Phi((-b^{(0)}_2, x_2)) - \Phi(-b^{(0)}_2 - 1, x_2) \vert}{C} > 1/(2C) > \lip(\widetilde{\Phi}).
\end{equation*}
Hence, for fixed biases $b^{(0)}$ and $b^{(1)}$ we note
\begin{equation*}
\PP^{(W^{(0)}, W^{(1)})} \left(\lip(\Phi) > C \cdot \lip(\widetilde{\Phi})\right)> 0. 
\end{equation*}
In total, allowing randomness in $b^{(0)}$ and $b^{(1)}$ as well gives us
\begin{equation*}
\PP \left(\lip(\Phi) > C \cdot \lip(\widetilde{\Phi})\right) = \underset{b^{(0)}, b^{(1)}}{\EE} \ \left[ \underbrace{\PP^{(W^{(0)}, W^{(1)})}\left(\lip(\Phi)> C \cdot \lip(\widetilde{\Phi})\right)}_{> 0}\right] > 0. \qedhere
\end{equation*}
\end{proof}
The proof of \Cref{prop:not_working_deep} is based on identity
\begin{equation*}
\relu(-\relu(x))= 0 \quad \text{for all }x \in \RR,
\end{equation*}
which follows directly from the definition of the $\relu$.
\renewcommand*{\proofname}{Proof of \Cref{prop:not_working_deep}}
\begin{proof}
Whenever $b^{(1)}\leq 0, W^{(1)} < 0$ and $W^{(0)},W^{(2)} \neq 0$ (which happens with positive probability), we get
\begin{equation*}
\relu(\underbrace{W^{(1)}\cdot \underbrace{\relu(W^{(0)}x + b^{(0)} )}_{\geq 0} + b^{(1)}}_{\leq 0}) = 0 \quad \text{for all}\quad x \in \RR,
\end{equation*}	
which implies $\Phi \equiv 0$ and in particular $\lip(\Phi)=0$ in that case. On the other hand, we get
\begin{equation*}
\lip(\widetilde{\Phi}) = \vert W^{(2)}\cdot W^{(1)} \cdot W^{(0)}\vert > 0. \qedhere
\end{equation*} 
\end{proof}

\renewcommand*{\proofname}{Proof}

\section{A note on measurability} \label{app:measurable}
In this appendix we elaborate on the measurability of the function
\begin{equation} \label{eq:map}
(W^{(0)}, ..., W^{(L)}, b^{(0)}, ..., b^{(L)}) \mapsto \underset{x \in \RR^d}{\sup} \Vert W^{(L)}D^{(L-1)}(x)W^{(L-1)} \cdots D^{(0)}(x) W^{(0)}\Vert_2.
\end{equation}
Note that we computed (upper bounds for) the expectation of this function with respect to the weights $W^{(0)}, ..., W^{(L)}$ and the biases $b^{(0)}, ..., b^{(L)}$ in \Cref{sec:upper} in order to establish upper bounds for the Lipschitz constant of ReLU networks. In order for the expectation to make sense, the function needs to be measurable, which we verify (to a sufficient extent) in this appendix.

As a preparation, we recall from \cite[Section~1.5]{cohn2013measure} the notion of the \emph{completion} of a measure. 
Given a measure $\mu$ on a measurable space $(\Omega, \mathscr{A})$, we define
\begin{equation*}
\mathscr{A}_\mu \defeq \left\{ A \subseteq \Omega : \ \exists \ E,F \in \mathscr{A}: \ E \subseteq A \subseteq F \text{ and } \mu(F \setminus E) = 0\right\}.
\end{equation*}
Then, $\mathscr{A}_\mu$ is a $\sigma$-algebra with $\mathscr{A} \subseteq \mathscr{A}_\mu$, and we can extend $\mu$ to a well-defined measure $\overline{\mu}$ on $\mathscr{A}_\mu$ by setting
\begin{equation*}
\overline{\mu}(A) \defeq \mu(E) \quad \text{for} \quad E,F \in \mathscr{A} \quad \text{with} \quad E \subseteq A \subseteq F \quad \text{and} \quad \mu(F \setminus E) = 0.
\end{equation*} 
 Moreover, the measure space $(X, \mathscr{A}_\mu, \overline{\mu})$ is then \emph{complete}, meaning that $A \in \mathscr{A}_\mu$, $\overline{\mu}(A)= 0$ and $B \subseteq A$ together imply $B \in \mathscr{A}_\mu$ for all subsets $A,B \subseteq \Omega$.
 
 Next, given a measurable space $(\Omega, \mathscr{A})$, we say that $A \subseteq \Omega$ is \emph{universally measurable} if $A \in \mathscr{A}_\mu$ for \emph{every} finite measure $\mu: \mathscr{A} \to [0, \infty]$. It is easy to see that the universally measurable sets over $(\Omega, \mathscr{A})$ form a $\sigma$-algebra, and that every $A \in \mathscr{A}$ is universally measurable. 
 
 We show that the map considered in \eqref{eq:map} is universally measurable, i.e., that the preimage of every Borel set is universally measurable. 
\begin{proposition}
Let $\arrow{W} \defeq (W^{(0)}, ..., W^{(L)}) \in \RR^{N \times d} \times \RR^{N \times N} \times \cdots \times \RR^{N \times N} \times \RR^{1 \times N} \cong \RR^{D_1}$ and $\arrow{b} \defeq (b^{(0)}, ..., b^{(L)}) \in \RR^N \times \cdots \times \RR^N \times \RR \cong \RR^{D_2}$. Then the map
\begin{equation*}
\xi: \quad \RR^{D_1} \times \RR^{D_2} \to \RR , \quad (\arrow{W}, \arrow{b}) \mapsto \underset{x \in \RR^d}{\sup} \Vert W^{(L)}D^{(L-1)}(x)W^{(L-1)} \cdots D^{(0)}(x) W^{(0)}\Vert_2
\end{equation*}
is universally measurable. That is, $\xi^{-1}(U)$ is universally measurable over $(\RR^{D_1} \times \RR^{D_2}, \mathscr{B})$ (with $\mathscr{B}$ denoting the Borel $\sigma$-algebra over $\RR^{D_1} \times \RR^{D_2}$) for any Borel measurable set $U \subseteq \RR$. Here, the matrices $D^{(L-1)}(x), ..., D^{(0)}(x)$ are as defined in \Cref{subsec:gradient}.
\end{proposition}
\begin{proof}
According to \cite[Corollary 4.25]{aliprantis_infinite_2006} it suffices to show that $\xi^{-1}((a, \infty))$ is universally measurable for every $a \in \RR$ in order to show that $\xi$ is universally measurable. 

Note that
\begin{align*}
\xi^{-1}((a,\infty)) &= \left\{(\arrow{W}, \arrow{b}) \in \RR^{D_1 + D_2}: \ \exists x \in \RR^d \text{ with } \Vert W^{(L)}D^{(L-1)}(x) \cdots D^{(0)}(x) W^{(0)}\Vert_2 > a\right\} \\
&= \pi \left(\left\{ (\arrow{W}, \arrow{b}, x) \in \RR^{D_1} \times \RR^{D_2} \times \RR^d: \ \Vert W^{(L)}D^{(L-1)}(x) \cdots D^{(0)}(x) W^{(0)}\Vert_2 > a\right\}\right).
\end{align*}
Here, 
\begin{equation*}
\pi: \quad \RR^{D_1 + D_2 + d} \to \RR^{D_1 + D_2}, \quad (\arrow{W}, \arrow{b}, x) \mapsto (\arrow{W}, \arrow{b})
\end{equation*}
denotes the projection from $\RR^{D_1 + D_2 + d}$ onto $\RR^{D_1 + D_2}$. It is easy to see that the map
\begin{equation*}
\RR^{D_1 + D_2 + d} \to \RR, \quad (\arrow{W}, \arrow{b}, x) \mapsto \Vert W^{(L)}D^{(L-1)}(x)W^{(L-1)} \cdots D^{(0)}(x) W^{(0)}\Vert_2
\end{equation*}
is Borel measurable, whence the set 
\begin{equation*}
\left\{ (\arrow{W}, \arrow{b}, x) \in \RR^{D_1} \times \RR^{D_2} \times \RR^d: \ \Vert W^{(L)}D^{(L-1)}(x) \cdots D^{(0)}(x) W^{(0)}\Vert_2 > a\right\}
\end{equation*}
is Borel measurable. Since $\pi$ is trivially Borel measurable, it follows from \cite[Theorem 12.24]{aliprantis_infinite_2006} that $\xi^{-1}((a, \infty))$ is an analytic set and hence, $\xi^{-1}((a, \infty))$ is universally measurable by \cite[Theorem 12.41]{aliprantis_infinite_2006}.
\end{proof}
It remains unclear, whether $\xi$ is not merely universally measurable, but in fact Borel measurable. 

In order for the expectation in \Cref{sec:upper} to be well-defined we have to require that the map
\begin{equation*}
\mathscr{W}: \quad \Omega \to \RR^{D_1 + D_2}, \quad \omega \mapsto (\arrow{W}(\omega), \arrow{b}(\omega)),
\end{equation*}
is measurable with respect to $\mathscr{A}$ and the $\sigma$-algebra of universally measurable sets on $\RR^{D_1 + D_2}$. Here, $(\Omega, \mathscr{A}, \PP)$ denotes the underlying ``master probability space''. This can be achieved by defining the probability space to be $\Omega \defeq \RR^{D_1 + D_2}$ together with the $\sigma$-algebra of universally measurable subsets, $\mathscr{W} = \id$ and $\PP$ as the (completion of the) probability measure representing the distribution of $(\arrow{W}, \arrow{b})$ according to \Cref{assum:1}.

\section{A proposition for deriving high-probability results}
This appendix contains a folklore proposition which is needed for proving the high-probability results derived in the present paper. For the sake of completeness, we decided to include a proof of this proposition in the present paper. It states the following:
\begin{proposition} \label{prop:highprob}
Let $(\Omega, \mathscr{A}, \PP)$ be a probability space, $(\Omega_1, \mathscr{A}_1)$ and $(\Omega_2, \mathscr{A}_2)$ two measurable spaces and $X: \Omega \to \Omega_1$, $Y: \Omega \to \Omega_2$ two stochastically independent random variables. Moreover, let $A \subseteq \Omega_1 \times \Omega_2$ be measurable (with respect to the product-sigma-algebra $\mathscr{A}_1 \otimes \mathscr{A}_2$), for every $x \in \Omega_1$ let some $A_2(x) \in \mathscr{A}_2$ be given and let $A_1 \in \mathscr{A}_1$. Finally, assume that 
\begin{equation} \label{eq:implic}
\forall \ (x,y) \in \Omega_1 \times \Omega_2: \quad x \in A_1, \ y \in A_2(x) \quad \Longrightarrow \quad (x,y) \in A.
\end{equation}
Then it holds
\begin{equation*}
\PP((X,Y) \in A) \geq \PP^{X} \left(A_1\right)\cdot \underset{x \in A_1}{\inf} \PP^{Y} \left(A_2(x)\right).
\end{equation*}
\end{proposition}

\begin{proof}
The condition \eqref{eq:implic} can be translated to
\begin{equation*}
\mathbbm{1}_{A}(x,y) \geq \mathbbm{1}_{A_1}(x) \cdot \mathbbm{1}_{A_2(x)}(y)
\end{equation*}
for every $(x,y) \in \Omega_1 \times \Omega_2$. Note that due to the stochastical independence of $X$ and $Y$ it holds $\PP^{(X,Y)} = \PP^X \otimes \PP^Y$. Applying Tonelli's theorem, we get
\begin{align} \label{eq:fubini_appl}
\PP((X,Y) \in A) &= \int_{\Omega_1 \times \Omega_2} \mathbbm{1}_A (x,y) \ \dd \PP^{(X,Y)}(x,y) = \int_{\Omega_1} \int_{\Omega_2} \mathbbm{1}_A(x,y) \ \dd \PP^Y(y) \ \dd \PP^X(x).
\end{align}
For fixed $x \in \Omega_1$ we get
\begin{align*}
\int_{\Omega_2} \mathbbm{1}_A(x,y) \ \dd \PP^Y(y) &\geq \mathbbm{1}_{A_1}(x)\int_{\Omega_2} \mathbbm{1}_{A_2(x)}(y) \ \dd \PP^Y(y) = \begin{cases} \PP^Y (A_2(x)),& x \in A_1, \\0,& x \notin A_1,\end{cases} \\
&\geq \mathbbm{1}_{A_1}(x) \cdot \underset{x \in A_1}{\inf} \PP^Y (A_2(x)).
\end{align*}
Inserting this into \eqref{eq:fubini_appl} yields
\begin{align*}
\PP((X,Y) \in A) \geq \int_{\Omega_1}\mathbbm{1}_{A_1}(x) \cdot \underset{x \in A_1}{\inf} \PP^Y (A_2(x)) \ \dd \PP^X(x) = \PP^{X} \left(A_1\right) \cdot \underset{x \in A_1}{\inf} \PP^{Y} \left(A_2(x)\right),
\end{align*}
as was to be shown.
\end{proof}

\section{On Dudley's inequality} \label{app:dudley}

In this section we add important notes on Dudley's inequality,
which is crucial for deriving the upper bounds in \Cref{sec:upper}. 

We first define the notion of separability of a random process. 

\begin{definition}[{cf. \cite[Definition 5.22]{van2014probability}}]
A random process $(X_t)_{t \in T}$ on a probability space $(\Omega, \mathscr{A}, \PP)$ indexed by a metric space $(T, \varrho)$
is called \emph{separable} if there exists a countable set $T_0 \subseteq T$ and a $\PP$-null set $N \subseteq \Omega$
such that for all $\omega \in \Omega \setminus N$
and $t \in T$ there exists a sequence $(s_n)_{n \in \NN} \in T_0^\NN$
with $s_n \to t$ and 
\begin{equation*}
\lim_{n  \to \infty} X_{s_n}(\omega) = X_t(\omega).
\end{equation*} 
\end{definition}

\begin{proposition} \label{prop:equal}
Let $(\Omega, \mathscr{A}, \PP)$ be a probability space, $(T, \varrho)$ a separable metric space and $(X_t)_{t \in T}$ a random process with $X_t: \Omega \to \RR$ for every $t \in T$. Moreover, we assume that for every fixed $\omega \in \Omega$ the map
\begin{equation*}
T \to \RR, \quad t \mapsto X_t(\omega)
\end{equation*}
is continuous. Then $(X_t)_{t \in T}$ is a separable random process. 
\end{proposition}

\newcommand{\T}{\widetilde{T}}
\begin{proof}
Let $T_0 \subseteq T$ be a countable dense subset of $T$ and $\omega \in \Omega, \ t \in T$.
Since $T_0$ is dense there is a sequence $(s_n)_{n \in \NN} \in T_0^\NN$ with $s_n \to t$.
Then, using the continuity assumption, it holds
\begin{equation*}
\lim_{n \to \infty} \ X_{s_n}(\omega) = X_t (\omega),
\end{equation*} 
which yields the claim.
\end{proof}

We can now formulate and prove a slight extension of Dudley's inequality from \cite{van2014probability}.

\begin{theorem}[{cf.~\cite[Theorems~5.25~and~5.29]{van2014probability}}] \label{thm:dudleyyy}
There is a constant $C>0$ with the following property: If $(X_t)_{t \in T}$ is a mean-zero separable random process indexed by a metric space $(T,\varrho)$ satisfying
\begin{equation*}
\Vert X_t - X_s \Vert_{\psi_2} \leq K \, \varrho(t,s)
\end{equation*}
for every $t,s \in T$ with a constant $K>0$ which does not depend on $s$ and $t$, the following hold:
\begin{enumerate}
\item{For all $t_0 \in T$ and $u \geq 0$ it holds
\begin{equation*}
\underset{t \in T}{\sup} \ (X_t - X_{t_0}) \leq CK \left(\int_0^\infty \sqrt{\ln \left(\mathcal{N}(T, \varrho, \eps)\right)} \ \dd\eps + u\diam(T)\right)
\end{equation*}
with probability at least $1- 2\exp(-u^2)$.}
\item{
$ \displaystyle
\EE \left[ \underset{t \in T}{\sup}\ X_t\right] \leq CK \int_0^\infty \sqrt{\ln \left(\mathcal{N}(T, \varrho, \eps)\right)} \ \dd\eps.
$}
\end{enumerate}
\end{theorem}

\begin{proof}
Firstly, note that the separability assumption implies that $\underset{t \in T}{\sup} \ X_t = \underset{t \in T_0}{\sup} \ X_t$ almost surely for some countable set $T_0 \subseteq T$ so that no measurability issues arise (at least if we are willing to pass to the completion of the underlying probability space). 

Let $C_1 > 0$ be the absolute constant appearing in \cite[Equation (2.16)]{vershynin_high-dimensional_2018} and define $Y_t \defeq \frac{X_t}{K \sqrt{2C_1}}$ for $t \in T$. Then it holds
\begin{equation*}
\Vert Y_t - Y_s \Vert_{\psi_2} = \frac{\Vert X_t - X_s \Vert_{\psi_2}}{K \sqrt{2C_1}} \leq \frac{\varrho(t,s)}{\sqrt{2C_1}}
\end{equation*}
for every $t,s \in T$ and hence, using \cite[Equation (2.16)]{vershynin_high-dimensional_2018}, we get
\begin{align*}
\EE \left[ \exp(\lambda(Y_t - Y_s))\right] \leq \exp(C_1 \lambda^2 \Vert Y_t - Y_s\Vert_{\psi_2}^2) \leq \exp(\lambda^2 \varrho(t,s)^2/2)
\end{align*}
for all $\lambda \in \RR$. Hence, $(Y_t)_{t \in T}$ is a \emph{sub-gaussian} process (see \cite[Definition 5.20]{van2014probability}) and trivially inherits the property of separability from $(X_t)_{t \in T}$. Hence, \cite[Theorems~5.25~and~5.29]{van2014probability} yield the existence of a constant $C_2>0$ such that the following hold:
\begin{enumerate}
\item{For all $u \geq 0$ and $t_0 \in T$ it holds 
\begin{equation*}
\underset{t \in T}{\sup} \ (Y_t - Y_{t_0}) \leq C_2 \left(\int_0^\infty \sqrt{\ln(\mathcal{N}(T, \varrho, \eps))} \ \dd \eps + u \diam(T) \right)
\end{equation*}
with probability at least $1- C_2\exp(-u^2)$.}
\item{$ \displaystyle \EE \left[\underset{t \in T}{\sup} \  Y_t \right] \leq C_2 \int_0^\infty \sqrt{\ln(\mathcal{N}(T, \varrho, \eps))} \ \dd \eps $}.
\end{enumerate}
Note that in (1) we can bound the probability even by $1-2\exp(-u^2)$ which is due to the first inequality on p.137 in \cite{van2014probability}. Specifically, we have
\begin{equation*}
\sum_{k=1}^\infty \ee ^{-k /2} \leq \sum_{k=1}^\infty\left( \frac{9}{4} \right)^{-k/2} = \sum_{k=1}^\infty \left(\frac{3}{2}\right)^{-k} = \frac{2/3}{1- 2/3} = 2.
\end{equation*}
Using this fact and the exact definition of $(Y_t)_{t \in T}$ yields the claim by letting $C \defeq \sqrt{2C_1}C_2$.
\end{proof}

Now we apply Dudley's inequality to certain finite-dimensional Gaussian processes.

\begin{proposition}[{special version of Dudley's inequality}]\label{prop:dudley}
There exists an absolute constant $C>0$ with the following property: For any $T \subseteq \RR^N$ with $0 \in T$ and any random vector $X \in \RR^N$ with 
\begin{equation*}
X_i \iid \mathcal{N}(0,1) \quad \text{for} \quad 1 \leq i \leq N,
\end{equation*}
 the following hold:
 \begin{enumerate}
 \item{ For any $u \geq 0$, we have
 \begin{equation*}
 \underset{t \in T}{\sup} \ \langle X, t \rangle \leq  C \left(\int_0^\infty \sqrt{\ln \left(\mathcal{N}(T, \Vert \cdot \Vert_2, \eps)\right)} \ \dd \eps + u \cdot \diam(T)\right)
 \end{equation*}
 with probability at least $1 - 2 \exp(-u^2)$.
 }
 \item{
 $ \displaystyle
 \EE \left[ \underset{t \in T}{\sup} \ \langle X, t \rangle\right] \leq C\int_0^\infty \sqrt{\ln \left(\mathcal{N}(T, \Vert \cdot \Vert_2, \eps)\right)} \ \dd \eps.
 $
 }
 \end{enumerate}
\end{proposition}
\begin{proof}
For $t,s \in T$ we have
\begin{equation*}
\Vert \langle X,t \rangle - \langle X,s \rangle \Vert_{\psi_2} = \Vert \langle X, t-s \rangle \Vert_{\psi_2} \leq C_1 \Vert t -s \Vert_2
\end{equation*}
with an absolute constant $C_1>0$. Here we used that$\langle X, t-s \rangle \sim \mathcal{N}(0, \Vert t -s \Vert_2^2)$ (see for instance \cite[Exercise 3.3.3]{vershynin_high-dimensional_2018}) and \cite[Example 2.5.8 (i)]{vershynin_high-dimensional_2018}. Furthermore, it holds
\begin{equation*}
\EE \left[ \langle X,t \rangle\right] = 0
\end{equation*}
for fixed $t \in T$ since $\langle X,t \rangle \sim \mathcal{N}(0, \Vert  t \Vert_2^2)$. For a fixed realization of $X$ it is trivial that $t \mapsto \langle X,t \rangle$ is continuous and hence, the process $(\langle X,t \rangle)_{t \in T}$ is separable by \Cref{prop:equal}. Dudley's inequality and its high-probability version (cf. \Cref{thm:dudleyyy}) then show that:
\begin{enumerate}
\item{ For any $u \geq 0$, we have
\begin{equation*}
\underset{t \in T}{\sup} \ \langle X, t \rangle = \underset{t \in T}{\sup}  \big(\langle X,t \rangle - \langle X, 0 \rangle \big)\overset{0 \in T}{\leq}  C_1C_2 \left(\int_0^\infty \sqrt{\ln \left(\mathcal{N}(T, \Vert \cdot \Vert_2, \eps)\right)} \ \dd \eps + u \cdot \diam(T)\right)
\end{equation*}
with probability at least $1 - 2\exp(-u^2)$. 
}
\item{
$ \displaystyle
\EE \left[ \underset{t \in T}{\sup} \ \langle X, t \rangle\right] \leq C_1C_2\int_0^\infty \sqrt{\ln \left(\mathcal{N}(T, \Vert \cdot \Vert_2, \eps)\right)} \ \dd \eps
$}
\end{enumerate}
with an absolute constant $C_2 > 0$. Hence, the claim follows letting $C \defeq C_1C_2$. 
\end{proof}

\section{Differentiability of a random ReLU network at a fixed point} \label{app:diff}

In this appendix, we prove that, if we fix a point $x_0 \in \RR^d \setminus \{0\}$, a random $\relu$ network that follows \Cref{assum:1} is differentiable at $x_0$ with
\begin{equation*}
\left(\nabla \Phi (x_0)\right)^T = W^{(L)} \cdot D^{(L-1)}(x_0) \cdot W^{(L-1)}\cdots D^{(0)}(x_0) \cdot W^{(0)}.
\end{equation*}
Since the vector $x_0$ is fixed, we often omit the argument and write $D^{(\ell)}$ instead of $D^{(\ell)}(x_0)$ for each $\ell \in \{0,...,L-1\}$. 
Let us now formulate and prove the statement.
\begin{theorem}
Let $x_0 \in \RR^d \setminus \{0\}$ be fixed and consider a random $\relu$ network $\Phi: \RR^d \to \RR$ of width $N$ and depth $L$, satisfying the random initialization as described in \Cref{assum:1}. 
Then $\Phi$ is almost surely differentiable at $x_0$ with
\begin{equation} \label{eq:des}
\left(\nabla \Phi(x_0)\right)^T = W^{(L)} \cdot D^{(L-1)}\cdot W^{(L-1)} \cdots D^{(0)}\cdot W^{(0)}.
\end{equation}
\end{theorem}
\begin{proof}
We denote by $\mathcal{A}$ the event that $\Phi$ is differentiable at $x_0$ and \eqref{eq:des} holds. 
Let $x^{(\ell)}$, $\ell \in \{0,...,L-1\}$ be defined as in \eqref{eq:d-matrices}.
We denote
\begin{align*}
\eta &= \eta (W^{(0)},..., W^{(L-1)}, b^{(0)},..., b^{(L-1)}) \\
&\defeq \min \left\{\ell \in \{0,...,L-1\}: \ \exists i \in \{1,...,N\} \text{ with } \left(W^{(\ell)}x^{(\ell)} + b^{(\ell)}\right)_i = 0\right\}
\end{align*}
and $\eta \defeq L$ if the minimum above is not well-defined. 
Since the events $\{\eta = \ell\}$ for $\ell \in \{0,...,L\}$ are pairwise disjoint and their union is the entire probability space, we get
\begin{equation*}
\PP(\mathcal{A}) = \sum_{\ell = 0}^L \PP(\mathcal{A} \cap \{\eta = \ell\}).
\end{equation*}
The goal is now to show that
\begin{equation}\label{eq:Ptoshow}
\PP (\mathcal{A} \cap \{\eta = \ell\}) = \PP(\{\eta = \ell\}) \quad \text{for every } \ell \in \{0,...,L\}, 
\end{equation}
which will directly give us $\PP(\mathcal{A}) = 1$.

To this end, we start with the case $\ell = L$. Note that 
\begin{equation*}
\eta = L \quad \Longleftrightarrow \quad \left(W^{(\ell)}x^{(\ell)} + b^{(\ell)}\right)_i \neq 0 \quad \text{ for all } \ell \in \{0,...,L-1\}, \ i \in \{1,...,N\}.
\end{equation*}
But since the $\relu$ is differentiable on $\RR \setminus \{0\}$, this directly implies the differentiability of $\Phi$ at $x_0$ and that \eqref{eq:des} holds. 
Hence, it even holds
\begin{equation*}
\mathcal{A} \cap \{\eta = L\} = \{\eta = L\}.
\end{equation*}
We proceed with the case $\ell \in \{1,...,L-1\}$. 
Note that it holds
\begin{equation*}
\PP (\mathcal{A} \cap \{\eta = \ell\}) = \PP (\mathcal{A} \cap \{\eta = \ell\}\cap \{x^{(\ell)} \neq 0\}) + \PP (\mathcal{A} \cap \{\eta = \ell\}\cap \{x^{(\ell)} = 0\}).
\end{equation*}
We get
\begin{align}
\PP (\mathcal{A} \cap \{\eta = \ell\}\cap \{x^{(\ell)} \neq 0\}) &\leq \PP (\{\eta = \ell\}\cap \{x^{(\ell)} \neq 0\}) \nonumber\\
&= \underset{b^{(0)},...,b^{(\ell-1)}}{\underset{W^{(0)},..., W^{(\ell - 1)}}{\EE}} \left[\mathbbm{1}_{x^{(\ell)} \neq 0} \cdot \underset{W^{(\ell)}, b^{(\ell)}}{\EE} \mathbbm{1}_{\eta = \ell} \right]\nonumber\\
\label{eq:zeroset} 
&= 0.
\end{align}
Here, the last equality stems from the fact that, if $x^{(\ell)} \neq 0$, it follows that, conditioning on the weights $W^{(0)}, ..., W^{(\ell -1)}$ and the biases $b^{(0)},..., b^{(\ell - 1)}$, it holds
\begin{equation*}
\left(W^{(\ell)} x^{(\ell)} + b^{(\ell)} \right)_i \sim \mathcal{N}\left(0, \frac{2}{N} \Vert x^{(\ell)} \Vert_2^2\right) \ast \mathcal{D}^{(\ell)}_i \quad \text{for each }i \in \{1,...,N\},
\end{equation*}
the fact that the distribution on the right-hand side is an absolutely continuous distributions (as follows from \cite[Proposition~9.16]{dudley2002real}) and the fact that all the random variables $\left(W^{(\ell)} x^{(\ell)} + b^{(\ell)} \right)_i$, $i \in \{1,...,N\}$, are jointly independent. Here, $\ast$ denotes the convolution of two (probability) measures.

We now show that it holds
\begin{equation} \label{eq:capequal}
\mathcal{A} \cap \{\eta = \ell\}\cap \{x^{(\ell)} = 0\} = \{\eta = \ell\}\cap \{x^{(\ell)} = 0\}.
\end{equation}
To this end, assume that $\eta = \ell$ and $x^{(\ell) } = 0$ hold. This implies 
\begin{equation*}
\left(W^{(\ell - 1)}x^{(\ell - 1)} + b^{(\ell - 1)}\right) _i \neq 0 \quad \text{and} \quad \left(W^{(\ell - 1)}x^{(\ell - 1)} + b^{(\ell - 1)}\right) _i \leq 0 \quad \text{for all } i \in \{1,...,N\},
\end{equation*}
or equivalently
\begin{equation*}
\left(W^{(\ell - 1)}x^{(\ell - 1)} + b^{(\ell - 1)}\right) _i < 0 \quad \text{for all }i \in \{1,...,N\}.
\end{equation*}
But by continuity, these inequalities also hold in an open neighborhood of $x_0$. 
This implies that $\Phi$ is constant on a neighborhood of $x_0$ and hence, $\Phi$ is differentiable at $x_0$ with  
\begin{equation*}
\left(\nabla \Phi (x_0)\right)^T = 0 = W^{(L)} \cdot D^{(L-1)} \cdot W^{(L-1)}\cdots D^{(0)} \cdot W^{(0)}.
\end{equation*}
Here, the last equality stems from $D^{(\ell - 1)} = \Delta (W^{(\ell - 1)} x^{(\ell - 1)} + b^{(\ell - 1)})= 0$. This proves \eqref{eq:capequal}. Overall, we get
\begin{align*}
\PP (\mathcal{A} \cap \{\eta = \ell\}) \overset{\eqref{eq:zeroset}}&{=} \PP (\mathcal{A} \cap \{\eta = \ell\}\cap \{x^{(\ell)} = 0\}) \overset{\eqref{eq:capequal}}{=} \PP(\{\eta = \ell\}\cap \{x^{(\ell)} = 0\}) \\
\overset{\eqref{eq:zeroset}}&{=}\PP(\{\eta = \ell\}\cap \{x^{(\ell)} = 0\}) + \PP (\{\eta = \ell\}\cap \{x^{(\ell)} \neq 0\})\\
&= \PP(\{\eta = \ell\}) .
\end{align*}
It remains to deal with the case $\ell = 0$. Since $x^{(0)} = x_0 \neq 0$, we get with the same reasoning as above (\Cref{eq:zeroset}) that 
\begin{equation*}
0\leq \PP(\mathcal{A} \cap \{\eta = 0\}) \leq \PP (\eta = 0) = 0,
\end{equation*}
such that equality has to hold everywhere. 

We thus have shown \eqref{eq:Ptoshow}, which implies the claim. 
\end{proof}

\textbf{Acknowledgments.} 
The authors thank Felix Krahmer, Afonso Bandeira, and Sam Buchanan for fruitful discussions and helpful comments.
FV and PG acknowledge support by
the German Science Foundation (DFG) in the context of the Emmy Noether junior research
group VO 2594/1-1.
FV acknowledges support by the Hightech Agenda Bavaria. 

\medskip

\printbibliography
\end{document}